\documentclass[11pt]{article} 

\pdfoutput=1

\usepackage{etoolbox}
\newtoggle{arxiv}
\toggletrue{arxiv}

\usepackage[utf8]{inputenc}
\usepackage{mathrsfs}
\usepackage{microtype}
\usepackage{subfigure}
\usepackage{booktabs} 
\usepackage{longtable,makecell}
\usepackage{amsmath, amsfonts, amsthm, amssymb, graphicx}
\usepackage{mathtools,verbatim}
\usepackage{enumitem}
\usepackage{bm}
\usepackage{thm-restate}
\usepackage{natbib}
\usepackage{xspace}
\usepackage[dvipsnames]{xcolor}

\usepackage{color-edits}

\addauthor{df}{ForestGreen}
\addauthor{ms}{orange}

\usepackage[noend]{algpseudocode}
\usepackage{algorithm}

\iftoggle{arxiv}
{
	\usepackage[scaled=.90]{helvet}
	
	\usepackage[vmargin=1.00in,hmargin=1.00in,centering,letterpaper]{geometry}
	\usepackage{fancyhdr, lastpage}

	\setlength{\headsep}{.10in}
	\setlength{\headheight}{15pt}
	\cfoot{}
	\lfoot{}
	\rfoot{\sc Page\ \thepage\ of\ \protect\pageref{LastPage}}

}

\usepackage{hyperref}
\usepackage{cleveref}
\hypersetup{colorlinks,citecolor=blue}
\newtoggle{upperbound}
\togglefalse{upperbound}

\usepackage[utf8]{inputenc} 
\usepackage[T1]{fontenc}    
\usepackage{hyperref}       
\usepackage{url}            
\usepackage{booktabs}       
\usepackage{amsfonts}       
\usepackage{nicefrac}       
\usepackage{microtype}      
\usepackage{xcolor}         
\usepackage{graphicx}

\usepackage{amssymb} 
\usepackage{amsmath} 
\usepackage{amsthm}
\usepackage{quoting}
\usepackage{dsfont}
\usepackage{bbm}
\usepackage{longtable}
\usepackage{hyperref}
\usepackage{algorithm}
\usepackage{lineno}
\usepackage{parskip}
\usepackage[normalem]{ulem}
\usepackage{enumitem}

\usepackage{cleveref}
\usepackage{xspace}
\usepackage[noend]{algpseudocode}
\usepackage{algorithm}
\usepackage{wrapfig}
 
\newcommand{\poly}{\mathrm{poly}}

\numberwithin{equation}{section}

{
 \theoremstyle{plain}
      
}
\creflabelformat{asm}{#2#1#3}

\theoremstyle{plain}
\newtheorem{thm}{Theorem}
\newtheorem{lem}{Lemma}[section]

\newtheorem{claim}[lem]{Claim}

\theoremstyle{definition}
\newtheorem{defn}{Definition}[section]

\theoremstyle{plain}
\newtheorem{theorem}{Theorem}
\newtheorem{lemma}{Lemma}[section]
\newtheorem{corollary}{Corollary}
\newtheorem{proposition}[thm]{Proposition}
\theoremstyle{definition}
\newtheorem{definition}{Definition}[section]

\newtheorem{example}{Example}[section]

\newtheorem{remark}{Remark}[section]

\renewcommand{\Pr}{\mathbb{P}}
\newcommand{\Exp}{\mathbb{E}}

\newcommand{\KL}{\mathrm{KL}}

\newcommand{\Z}{\mathbb{Z}}

\newcommand{\R}{\mathbb{R}}
\newcommand{\I}{\mathbb{I}}

\DeclareMathOperator*{\argmax}{arg\,max}
\DeclareMathOperator*{\argmin}{arg\,min}

\renewcommand{\P}{\mathbf{P}}

\newcommand{\simplex}{\triangle}
\newcommand{\cOtil}{\widetilde{\cO}}

\newcommand{\yst}{y_\star}

\newcommand{\ragest}{\textsc{Rage}^\epsilon\xspace}
\newcommand{\safebai}{\textsc{Beside}\xspace}
\newcommand{\Delsafe}{\Delta_{\mathrm{safe}}}
\newcommand{\Delhatsafe}{\widehat{\Delta}_{\mathrm{safe}}}

\newcommand{\Delhat}{\widehat{\Delta}}
\newcommand{\thetahat}{\widehat{\theta}}

\newcommand{\cEsafe}{\cE_{\mathrm{safe}}}
\newcommand{\cEragest}{\cE_{\ragest}}
\newcommand{\muhat}{\widehat{\mu}}
\newcommand{\iotaeps}{\iota_\epsilon}
\newcommand{\epstil}{\widetilde{\epsilon}}

\newcommand{\yhat}{\widehat{y}}
\newcommand{\pos}{\mathfrak{p}}
\newcommand{\Deltil}{\widetilde{\Delta}}
\newcommand{\cYend}{\cY_{\mathrm{end}}}
\newcommand{\eend}{\mathrm{end}}
\newcommand{\Delend}{\Delta^{\eend}}
\newcommand{\Rtilgam}{\widetilde{R}^\alpha}

\newcommand{\hh}{\widehat{h}}
\newcommand{\tih}{\widetilde{h}}
\newcommand{\ub}{\tfrac{4m|\cZ|\ell^2}{\delta}}

\newcommand{\lamtil}{\widetilde{\lambda}}

\newcommand{\zhat}{\widehat{z}}

\newcommand{\ubb}{\tfrac{4 |\cZ|^2 \ell^2}{\delta}}
\newcommand{\XYsafe}{$\cX\cY_{\mathrm{safe}}$\xspace}
\newcommand{\XYdiff}{$\cX\cY_{\mathrm{diff}}$\xspace}

\newcommand{\wtil}{\widetilde{w}}
\newcommand{\alphatil}{\widetilde{\alpha}}
\newcommand{\cAtil}{\widetilde{\cA}}

\newcommand{\ztil}{\widetilde{z}}
\newcommand{\z}{\mathfrak{z}}

\newcommand{\cat}{\mathsf{cat}}
\newcommand{\rips}{\textsf{RIPS}\xspace}

\newcommand{\ca}{c_a}
\newcommand{\cb}{c_b}
\newcommand{\cc}{c_c}
\newcommand{\csafe}{c_{\Delta}}
\newcommand{\cf}{c_f}
\newcommand{\cd}{c_d}
\newcommand{\ce}{c_e}
\newcommand{\cg}{c_g}
\newcommand{\cabs}{c_0}

\newcommand{\mc}[1]{\mathcal{#1}}

\def\ddefloop#1{\ifx\ddefloop#1\else\ddef{#1}\expandafter\ddefloop\fi}
\def\ddef#1{\expandafter\def\csname bb#1\endcsname{\ensuremath{\mathbb{#1}}}}
\ddefloop ABCDEFGHIJKLMNOPQRSTUVWXYZ\ddefloop

\def\ddefloop#1{\ifx\ddefloop#1\else\ddef{#1}\expandafter\ddefloop\fi}
\def\ddef#1{\expandafter\def\csname fr#1\endcsname{\ensuremath{\mathfrak{#1}}}}
\ddefloop ABCDEFGHIJKLMNOPQRSTUVWXYZ\ddefloop
\def\ddefloop#1{\ifx\ddefloop#1\else\ddef{#1}\expandafter\ddefloop\fi}
\def\ddef#1{\expandafter\def\csname scr#1\endcsname{\ensuremath{\mathscr{#1}}}}
\ddefloop ABCDEFGHIJKLMNOPQRSTUVWXYZ\ddefloop
\def\ddefloop#1{\ifx\ddefloop#1\else\ddef{#1}\expandafter\ddefloop\fi}
\def\ddef#1{\expandafter\def\csname b#1\endcsname{\ensuremath{\mathbf{#1}}}}
\ddefloop ABCDEFGHIJKLMNOPQRSTUVWXYZ\ddefloop
\def\ddef#1{\expandafter\def\csname c#1\endcsname{\ensuremath{\mathcal{#1}}}}
\ddefloop ABCDEFGHIJKLMNOPQRSTUVWXYZ\ddefloop

\theoremstyle{plain}

\def\ddefloop#1{\ifx\ddefloop#1\else\ddef{#1}\expandafter\ddefloop\fi}
\def\ddef#1{\expandafter\def\csname mat#1\endcsname{\ensuremath{\mathbf{#1}}}}
\ddefloop abcdefgijklmnopqrstuvwxyzABCDEFGHIJKLMNOPQRSTUVWXYZ\ddefloop

\allowdisplaybreaks

\def \E{\mathbb E}
\def \P{\mathbb P}
\def \R{\mathbb R}

\def \st2{\texttt{\texttt{$($ST$)^2$}}}

\def \cE{{\cal E}}

\def \X{{\cal X}}
\def \Y{{\cal Y}}

\def \Y{{\cal Y}}

\def \L{{\cal L}}

\def \H{{\cal H}}

\def \Y{{\cal Y}}
\def \Z{{\cal Z}}

\def\L{\mathcal{L}}

\newcommand{\must}{\mu_*}
\newcommand{\thetast}{\theta_*}

\def \E{\mathbb E}
\def \P{\mathbb P}
\def \R{\mathbb R}

\def \st2{\texttt{\texttt{$($ST$)^2$}}}

\def \cE{{\cal E}}

\def \X{{\cal X}}
\def \Y{{\cal Y}}

\def \Y{{\cal Y}}

\def \L{{\cal L}}

\def \H{{\cal H}}

\def \Y{{\cal Y}}
\def \Z{{\cal Z}}

\def\L{\mathcal{L}}

\def\Ber{\textbf{Ber}}

\newcommand{\zst}{z_*}

\def \1{\mathbbm{1}}
\def \0{\textbf{0}}

\theoremstyle{plain}
\makeatletter
\def\thm@space@setup{%
  \thm@preskip=0.5em
  \thm@postskip=\thm@preskip 
}
\makeatother

\def\ddefloop#1{\ifx\ddefloop#1\else\ddef{#1}\expandafter\ddefloop\fi}
\def\ddef#1{\expandafter\def\csname bb#1\endcsname{\ensuremath{\mathbb{#1}}}}
\ddefloop ABCDEFGHIJKLMNOPQRSTUVWXYZ\ddefloop

\def\ddefloop#1{\ifx\ddefloop#1\else\ddef{#1}\expandafter\ddefloop\fi}
\def\ddef#1{\expandafter\def\csname fr#1\endcsname{\ensuremath{\mathfrak{#1}}}}
\ddefloop ABCDEFGHIJKLMNOPQRSTUVWXYZ\ddefloop
\def\ddefloop#1{\ifx\ddefloop#1\else\ddef{#1}\expandafter\ddefloop\fi}
\def\ddef#1{\expandafter\def\csname scr#1\endcsname{\ensuremath{\mathscr{#1}}}}
\ddefloop ABCDEFGHIJKLMNOPQRSTUVWXYZ\ddefloop
\def\ddefloop#1{\ifx\ddefloop#1\else\ddef{#1}\expandafter\ddefloop\fi}
\def\ddef#1{\expandafter\def\csname b#1\endcsname{\ensuremath{\mathbf{#1}}}}
\ddefloop ABCDEFGHIJKLMNOPQRSTUVWXYZ\ddefloop
\def\ddef#1{\expandafter\def\csname c#1\endcsname{\ensuremath{\mathcal{#1}}}}
\ddefloop ABCDEFGHIJKLMNOPQRSTUVWXYZ\ddefloop
\def\ddef#1{\expandafter\def\csname h#1\endcsname{\ensuremath{\widehat{#1}}}}
\ddefloop ABCDEFGHIJKLMNOPQRSTUVWXYZ\ddefloop
\def\ddef#1{\expandafter\def\csname t#1\endcsname{\ensuremath{\widetilde{#1}}}}
\ddefloop ABCDEFGHIJKLMNOPQRSTUVWXYZ\ddefloop

\def\ddefloop#1{\ifx\ddefloop#1\else\ddef{#1}\expandafter\ddefloop\fi}
\def\ddef#1{\expandafter\def\csname mat#1\endcsname{\ensuremath{\mathbf{#1}}}}
\ddefloop abcdefgijklmnopqrstuvwxyzABCDEFGHIJKLMNOPQRSTUVWXYZ\ddefloop

\newlength\tindent
\title{Active Learning with Safety Constraints}
\author{Romain Camilleri,
  Andrew Wagenmaker,
  Jamie Morgenstern,
  Lalit Jain,
  Kevin Jamieson\\
  University of Washington, Seattle, WA\\
  \texttt{\{camilr,ajwagen,jamiemmt,jamieson\}@cs.washington.edu,lalitj@uw.edu}}
\date{\today}

\begin{document}

\maketitle

\begin{abstract}
Active learning methods have shown great promise in reducing the number of samples necessary for learning. As automated learning systems are adopted into real-time, real-world decision-making pipelines, it is increasingly important that such algorithms are designed with \emph{safety} in mind. In this work we investigate the complexity of learning the best \emph{safe} decision in interactive environments. We reduce this problem to a constrained linear bandits problem, where our goal is to find the best arm satisfying certain (unknown) safety constraints. We propose an adaptive experimental design-based algorithm, which we show efficiently trades off between the difficulty of showing an arm is unsafe vs suboptimal. To our knowledge, our results are the first on best-arm identification in linear bandits with safety constraints. In practice, we demonstrate that this approach performs well on synthetic and real world datasets.

\end{abstract}

\section{Introduction}

In many problems in online decision-making, the goal of the learner is to take measurements in such a way as to learn a near-optimal policy. Oftentimes, though the space of policies may be large, the set of feasible, or safe policies could be much smaller, effectively constraining the search space of the learner. Furthermore, these constraints may themselves depend on unknown problem parameters.

For example, consider the problem of bidding sequentially in a series of auctions where
the bidder bids a price $s_t$, the value of winning an item $t$ is denoted $v_t$, and the utility of winning that item and paying price $p_t$ is $v_t - p_t$. The goal of the bidder is to choose an optimal strategy amongst bidding strategies $s \in S, s: \mathbb{R} \to \mathbb{R}$. When a bidder is deciding how to choose these strategies, they often face constraints: they may have a budget $B$ they must abide to; they may wish to have those auctions they win be well-distributed across time (e.g. in the case of advertising campaigns); they may want to ensure the set of items they win satisfy some other property (e.g. for advertisements, they might want to ensure they are not over-targeting any demographic group).

As another example, inventory management systems may face similar issues of deciding
amongst strategies, where there is some objective function (such as
revenue) and a variety of constraints at play in this choice (e.g.
capacity of a set of warehouses, employee scheduling constraints, or
limits on the duration of delivery lag). They also operate in markets
with changing demand and other uncertainties, leading to uncertainty
about which strategies are feasible or safe (satisfy constraints) and
uncertainty about the revenue they generate.

Both of these scenarios motivate understanding the sample
complexity of selecting an action or strategy which
approximately maximizes an objective while also satisfying some
constraints, where samples are needed to both learn the objective
value of actions and whether or not they satisfy said constraints. In
this work, we study the \emph{active} sample complexity of this
task---if the learner can choose which examples to observe and have
labeled, how many fewer samples might they need compared to the number
needed in a passive setting? We pose this as a best-arm identification problem in the setting of linear bandits with safety constraints, where the goal is to estimate the best arm, subject to it meeting certain (initially unknown) safety constraints. We propose an experiment design-based algorithm which efficiently learns the best safe decision, and show the efficacy of this approach in practice through several experimental examples. To the best of our knowledge, ours is the first approach to handle best-arm identification in linear bandits with safety constraints.



\newcommand{\cZeps}{\cZ_{\epsilon}}
\newcommand{\musti}{\mu_{*,i}}

\subsection{Linear Bandits with Safety Constraints}

Let $\delta\in (0,1)$ be a confidence parameter, $\cX, \cZ \subseteq \R^d$ be finite known sets of vectors, and assume there exists $\theta_{\ast}\in \mathbb{R}^d$, $\mu_{\ast}\in \R^{m\times d}$ unknown to the learner. For simplicity, we assume that $\| \thetast \|_2 \le 1$, and $\| \musti \|_2 \le 1, i\in [m]$ and $\| x \|_2 \le 1$, $\| z \|_2 \le 1$, $\forall x \in \cX,z \in \cZ$. The learner plays according to the following protocol: at each time step $t$ the learner chooses some action $x_t\in \cX$, observes $(r_t, \{s_{t,i}\}_{i=1}^m)$ where $r_t = \thetast^\top x_t + w_t^\theta$ and $s_{t,i} = \mu_{*,i}^\top x_t + w_{t,i}^{\mu}$ for all $i\in [m]$, where $w_t^\theta,w_{t,i}^{\mu}$ are i.i.d. mean zero 1-subGaussian noise. The choice of action $x_t$ is measurable with respect to the history $\mc{F}_t = \{(x_j, r_j, \{s_{j,i}\}_{i=1}^m)\}_{j=1}^{t-1}$. The learner stops at a stopping time $\tau_{\delta}$ which is measurable with respect to the filtration generated by $\mc{F}_{t\leq \tau}$, and returns $\zhat_{\tau} \in \cZ$. In general, when referring to any expectation $\E$ or probability $\P$, the underlying measure will be with respect to the actions, observed rewards, and internal randomness of the algorithm. 

 We are interested in the \emph{safe transductive best-arm identification problem} (\textbf{STBAI}), where the goal of the learner is to identify 
\begin{align*}
\zst := \max_{z \in \cZ} z^\top \thetast \quad \text{s.t.} \quad z^\top \mu_{*,i} \le \gamma, \forall i \in [m]
\end{align*}
for some (known) threshold $\gamma$. In words, our goal is to identify the best \emph{safe} arm in $\cZ$, $\zst$, where we say an arm $z$ is safe if it satisfies every linear constraint: $z^\top \mu_{*,i} \le \gamma, \forall i \in [m]$. We are interested in obtaining learners that take the fewest number of samples possible to accomplish this. 
In practice, we will consider a slightly easier objective. Fix some tolerance $\epsilon > 0$ and let
\begin{align*}
\cZeps := \{ z \in \cZ \ : \ z^\top \thetast \ge \zst^\top \thetast - \epsilon, \ z^\top \musti \le \gamma + \epsilon, \forall i \in [m] \}.
\end{align*}
Then our goal is to obtain an $(\epsilon,\delta)$-PAC learner defined as follows:

\begin{definition}[$(\epsilon,\delta)$-PAC Learner]
A learner is $(\epsilon,\delta)$-PAC if for any instance it returns $\zhat_\tau$ such that $\P[\zhat_\tau \in \cZeps] \ge 1 - \delta$. 
\end{definition}




We define the \emph{optimality gap} for any $z \in \cZ$ as $\Delta(z) := \thetast^\top ( \zst - z)$, and the \emph{safety gap} for constraint $i$ as $\Delsafe^i(z) := \gamma -  \musti^\top z$. Note that either $\Delta(z)$ or $\Delsafe^i(z)$ can be negative. If $\Delta(z) < 0$, it follows that $z$ has larger \emph{value}---$z^\top \thetast$---than the best safe arm $\zst$, which implies it must be unsafe. If $\Delsafe^i(z) < 0$ for some $i$, then arm $z$ is unsafe. We also define the $\epsilon$-\emph{safe optimality gap} as:
\begin{align}\label{eq:Deleps}
\Delta^\epsilon(z) = \max_{z' \in \cZ} (z' - z)^\top \thetast \quad \text{s.t.} \quad \min_{i \in [m]} \Delsafe^i(z) \ge \epsilon.
\end{align}
$\Delta^\epsilon(z)$ is then the gap in value between arm $z$ and the best arm with minimum safety gap at least $\epsilon$. 

\paragraph{Mathematical Notation.} Let $\| x \|_{A}^2 = x^\top A x$ and $\pos(x) := \max \{ x, 0 \}$. $\cOtil(\cdot)$ hides factors that are logarithmic in the arguments. $\lesssim$ denotes inequality up to constants. We denote the simplex as $\triangle_{\cX} := \{\lambda\in \R_{\geq 0}^{|\cX|}:\sum_{x\in \cX}\lambda_x =1\}$.



\section{Safe Best-Arm Identification in Linear Bandits}
\newcommand{\conf}{c_{\ell}}

\subsection{Algorithm Definition}\label{sec:alg}

\begin{algorithm}[h]
\caption{\textbf{Be}st \textbf{S}afe Arm \textbf{Ide}ntification (\safebai)}\label{alg:beside}
\begin{algorithmic}[1]
\State \textbf{input:} tolerance $\epsilon$, confidence $\delta$
\State $\iotaeps \leftarrow \lceil \log(\frac{20}{\epsilon}) \rceil$, $\Delhatsafe^{i,0}(z) \leftarrow 0, \Delhat^0(z) \leftarrow 0$ for all $z \in \cZ$
\For{$\ell = 1,2,\ldots,\iotaeps$}
	\State $\epsilon_\ell \leftarrow  20 \cdot 2^{-\ell}$
	\Statex {\color{blue} \texttt{// Phase 1: Solve design to reduce uncertainty in safety constraints}}
	\State Define 
	\begin{align*}
	    \conf(z) &= \min_j |\Delhatsafe^{j,\ell-1}(z)| + \max_{j} \pos(-\Delhatsafe^{j,\ell-1}(z)) + \pos(\Delhat^{\ell-1}(z))
	\end{align*}
	\State Let $\tau_\ell$ be the minimal value of $\tau\in \mathbb{R}_+$ which is greater than $ 4 \log \ub$ such that the objective to the following is no greater than $ \epsilon_\ell/100$, and $\lambda_\ell$ the corresponding optimal distribution  \label{line:safety_gap_est}
	\begin{align*}
	\inf_{\lambda \in \simplex_\cX} \max_{z \in \cZ}  - \frac{1}{100} \left ( \conf(z) + \epsilon_\ell \right ) + \sqrt{\tau^{-1} \cdot \| z \|_{A(\lambda)^{-1}}^2 \cdot \log(\ub)}
	\end{align*}
	\State Sample $x_t \sim \lambda_\ell$, collect $\tau_\ell$ observations $\{ (x_t,r_t,s_{t,1},\ldots,s_{t,m}) \}_{t=1}^{\tau_\ell}$ 
	\Statex {\color{blue} \texttt{// Phase 2: Estimate safety constraints}}

	\State $\{ \muhat^{i,\ell} \}_{i=1}^m \leftarrow \rips(\{ (x_t, s_{t,i} ) \}_{t=1}^{\tau_\ell},\cZ,\frac{\delta}{2m\ell^2})$ \hfill 
	\State $\Delhatsafe^{i,\ell}(z) \leftarrow  \gamma - z^\top \muhat^{i,\ell} + \| z \|_{A(\lambda_\ell)^{-1}}\sqrt{\tau_\ell^{-1}   \log(\ub)}$ 
 	\Statex {\color{blue} \texttt{// Phase 3: Refine estimates of optimality gaps}}
	\State $\{ \Delhat^{\ell}(z) \}_{z \in \cZ} \leftarrow \ragest \Big (\cZ,\cY_\ell,\epsilon_\ell,\tfrac{\delta}{4 \ell^2}, \{ \Delhatsafe(z) \leftarrow \max_j \pos(-\Delhatsafe^{j,\ell}(z)) \}_{z \in \cZ} \Big )$
\EndFor

\Statex {\color{blue}  \texttt{// Perform final round of exploration to ensure we find $\epsilon$-good arm}}
\State $\cYend \leftarrow \{ z \in \cZ : c_{\ell}(z) \lesssim \Delhatsafe^{i,\ell}(z) + \epsilon \}$
\State $\{ \Delhat^\eend(z) \}_{z \in \cYend} \leftarrow \ragest(\cYend,\cYend,\epsilon,\delta,\{ \Delhatsafe(z) \leftarrow \max_j \pos(-\Delhatsafe^{j,\ell}(z)) \}_{z \in \cZ})$ 
\State \textbf{return} $\zhat = \argmin_{z \in \cYend} \Delhat^{\eend}(z)$
\end{algorithmic}
\label{alg:safe_bai_comp}
\end{algorithm}

The main challenge in algorithm design for the safe best-arm identification problem is ensuring that we are efficiently balancing our exploration between refining our estimates of both the safety gaps, as well as the optimality gaps. 
Our approach is given in Algorithm~\ref{alg:safe_bai_comp}, \safebai.

\safebai relies on a round-based adaptive experimental design approach. In each round \safebai consists of three phases. 
In the first phase, it solves an experimental design over $\lambda_{\ell}\in \triangle_{\cX}$, with the goal of refining our estimates of the safety gaps. 
It then takes $\tau_{\ell}$ samples from $\lambda_{\ell}$. In the second phase these samples are used to estimate the safety constraints, $\muhat^{i,\ell}$, and the safety gaps of each arm, $\Delhatsafe^{i,\ell}(z)$. Finally, in Phase 3, an additional experimental design is solved which now aims to refine our estimates of the optimality gaps, and the estimates of the optimality gaps $\Delhat^{\ell}(z)$ for each $z\in \mc{Z}$ are then computed. We encapsulate Phase 3 in a subroutine, $\ragest$, which we outline in the following.
We now carefully describe each phase---we begin with Phase 2 to explain how our estimator works.

\paragraph{Phase 2:} In Phase 2 the algorithm would like to use the $\tau_{\ell}$ samples drawn from the design $\lambda_{\ell}$ to estimate the constraints for each $z \in \cZ$: $z^\top \mu_{*,i}$ for each $i \in [m]$. Past works using adaptive experimental design in the linear bandits literature have utilized the least-squares estimator along with complicated rounding schemes \cite{fiez2019sequential} which may require an additional $\text{poly}(d)$ samples each round (this $\poly(d)$ factor could be prohibitively large---for example, in active classification problems, $d$ is the total number of data points).  We instead utilize the \textsf{RIPS} estimator of \cite{camilleri2021highdimensional} which gives us a guarantee of the form: with probability greater than $1-\delta$, for all $z\in\cZ$,
\begin{align}\label{eq:rips_alg}
    |z^\top (\muhat^{i,\ell} - \musti)| \lesssim \| z \|_{A(\lambda_\ell)^{-1}} \cdot  \sqrt{\tau_{\ell}^{-1}  \log(\ub) }. 
\end{align}
We describe the \rips estimator in more detail in \Cref{sec:rips}.

\paragraph{Phase 1:}  By our definition of the experimental design on Line \ref{line:safety_gap_est}, our safety gap estimation error bound in \eqref{eq:rips_alg} satisfies, for each $z\in \cZ$:
\begin{align}\label{eq:ci}
| z^\top (\muhat^{i,\ell} - \musti)| \lesssim \| z \|_{A(\lambda_\ell)^{-1}} \cdot  \sqrt{\tau_{\ell}^{-1}  \log(\ub) } \lesssim c_{\ell}(z) + \epsilon_\ell .
\end{align}
Note that our design chooses an allocation that minimizes the variance in our estimate of each safety constraint (up to some tolerance), which scales as $\| z \|_{A(\lambda)^{-1}}^2$. This can be thought of as a form of \emph{$\cX\cY$-design}---a design of the form $\inf_{\lambda \in \simplex_\cX} \max_{y \in \cY} \| y \|_{A(\lambda)^{-1}}^2$---where here $\cY \leftarrow \cZ$ is chosen to reduce our uncertainty in estimating the safety value for each $z \in \cZ$. We refer to such a design objective henceforth as \XYsafe.
Assume that at round $\ell-1$, we can guarantee
\begin{align}
c_\ell(z) &= \min_j  |\Delhatsafe^{j,\ell-1}(z)|  + \max_{j} \pos(-\Delhatsafe^{j,\ell-1}(z)) + \pos(\Delhat^{\ell-1}(z)) + \epsilon_\ell \nonumber \\
& \lesssim  \min_j |\Delsafe^{j}(z)|  + \max_{j} \pos(-\Delsafe^{j}(z)) + \pos(\Delta^{\epsilon_{\ell-1}}(z)) + \epsilon_\ell \label{eq:safety_exp_design_result2}.
\end{align}
Then combining the above inequalities, we see that the experiment design on Line \ref{line:safety_gap_est} aims to minimize the uncertainty in our estimate of $z^\top \musti$ up to a tolerance that scales as the maximum of the four terms in \eqref{eq:safety_exp_design_result2}. It follows that if any of these terms is large, we will only allocate a small number of samples to refining our estimate of arm $z$. Each one of these terms can be intuitively motivated by thinking through what is needed to prove that an arm $z\neq \zst$. 
\begin{itemize}[leftmargin=*]
\item \textbf{$z$ has small safety gap $\min_j |\Delsafe^{j}(z)|$:} if this term is large, it implies that minimum safety gap for $z$ is large. To show an arm is safe or unsafe, it suffices to learn each safety gap up to a tolerance a constant factor from its value---regularizing by this term ensures we do just that. 
\item \textbf{$z$ fails some safety constraint $\max_{j} \pos(-\Delsafe^{j}(z))$:} if this term is large, it implies that arm $z$ is very unsafe for some constraint. In this case, we can easily determine $z$ is unsafe, and therefore do not need to reduce our uncertainty in the safety gap any more.
\item \textbf{$z$ is sub-optimal $\pos(\Delta^{\epsilon_{\ell-1}}(z))$:} if this term is large, it implies that $z$ is very suboptimal compared to some safe arm with safety gap at least $\epsilon_{\ell-1}$. In this case, we do not need to estimate $z$'s safety gap, as we will have already eliminated it. 
\end{itemize}
It remains to ensure that~\eqref{eq:safety_exp_design_result2} holds.
As we show in \Cref{sec:safe_bai_proofs} through a careful inductive argument, combining \eqref{eq:ci} with our guarantee on the estimates of the optimality gaps obtained in Phase 3, $\Delhat^\ell(z)$, is sufficient to guarantee \eqref{eq:safety_exp_design_result2} holds.
In particular, if any gap is greater than $\epsilon_{\ell}$ it is estimated up to a constant factor, and otherwise it is estimated up to $\cO(\epsilon_{\ell})$. This ensures that our gaps are estimated at the correct rate while guaranteeing we do not collect too many samples in each round.


\paragraph{Phase 3:} In this phase we estimate the suboptimality gaps using $\ragest$.
$\ragest$ is inspired by the \textsc{Rage} algorithm of \cite{fiez2019sequential} for best-arm identification. In the interest of space, we defer the full definition of $\ragest$ to \Cref{sec:ragest} but provide some intuition here. After Phase 2, by \eqref{eq:ci} the set of arms $\mc{Y}_{\ell}:=\{z\in \cZ: c_{s}(z)\lesssim \Delhat^{i,s}(z), \forall i \in [m]\}$ for $s\leq \ell$ are precisely the ones that we can certify are safe (note that we do not need to ever explicitly construct such a set---we can instead maintain an implicit definition through the constraints). $\ragest$ uses an adaptive experimental design procedure to sample in such a way as to optimally estimate the gaps $(z - \widehat{y})^{\top}\theta_{\ast}, \forall z\in \cZ$ and some $\widehat{y} \in \cY_{\ell}$ up to some (sufficient) tolerance. In particular, it also solves an $\cX\cY$-design, but now on the set $\cY \leftarrow \{ z - \widehat{y} : z \in \cZ\}$. Thus, rather than minimizing $\| z \|_{A(\lambda)^{-1}}^2$, we minimize $\| z - \widehat{y} \|_{A(\lambda)^{-1}}^2$. This design reduces uncertainty on the \emph{differences} between arms, which allows us to refine our estimates of their optimality gaps. Henceforth we refer to such a design as \XYdiff.
We describe the importance of the choice of design in more detail in \Cref{sec:role_of_exp_design}.
Ultimately, if an arm $z$ has value within a factor of $\epsilon_{\ell}$ of the best safe arm in $\cY_\ell$, and if we have not yet shown arm $z$ is unsafe, then we will estimate its optimality gap up to a constant factor of $\epsilon_\ell$.
If we were maintaining arm sets explicitly (similar to the original \textsc{RAGE} algorithm of \cite{fiez2019sequential}) we would eliminate arms at this point.

\begin{remark}[Computational Complexity]
The main computational challenge in \safebai and $\ragest$ is the calculation of the experimental designs (i.e. Line \ref{line:safety_gap_est} and the corresponding design in $\ragest$). 
In general, the presence of the square root implies that the resulting optimization problem may not be convex in $\lambda$. To handle this issue we note that $2\sqrt{x y} = \min_{\alpha > 0 } \alpha x + \frac{y}{\alpha}$---thus we can replace the existing design with 
$\inf_{\lambda \in \simplex_\cX} \max_{z \in \cZ}\min_{\alpha > 0}  - \frac{1}{100} \left ( \conf(z) + \epsilon_\ell \right ) +  \alpha \| z \|_{A(\lambda)^{-1}}^2 + \log(\ub)/(\alpha\tau)$. By appropriately discretizing the space we search over for $\tau$ and $\alpha$ we can then apply the Frank-Wolfe algorithm to minimize over $\lambda$. 
While computationally efficient in theory, this procedure is quite complicated and impractical for large problems. 
In the experiments section we provide a practical heuristic that is motivated by the above algorithm and is  computationally efficient for larger problems.
\end{remark}


\subsection{Main Result}
\safebai achieves the following complexity.

\begin{theorem}\label{thm:main_complexity_nice}
\safebai is $(\epsilon, \delta)$-PAC. In other words, with probability at least $1-\delta$, \safebai returns an arm $\zhat \in \cZ$ such that 
\begin{align*}
\zhat^\top \thetast \ge \zst^\top \thetast -  \epsilon, \quad \min_{i \in [m]} \Delsafe^i(\zhat) \ge -  \epsilon
\end{align*}
and terminates after collecting at most
\begin{align*}
& C \cdot \sup_{\epstil \ge \epsilon} \inf_{\lambda \in \simplex_\cX} \max_{z \in \cZ} \frac{ \| z \|_{A(\lambda)^{-1}}^2 \cdot \log(\frac{m | \cZ |}{\delta})}{\big ( \min_j |\Delsafe^j(z)| + \max_j \pos(-\Delsafe^j(z))  + \pos(\Delta^{\epstil}(z)) + \epstil \big )^2} \tag{safety} \label{eq:complexity_safety} \\
& \qquad + C \cdot \sup_{\epstil \ge \epsilon} \inf_{\lambda \in \simplex_\cX} \max_{z \in \cZ} \frac{\| z - \zst \|_{A(\lambda)^{-1}}^2 \cdot \log(\tfrac{ |\cZ|}{\delta})}{\big ( \max_j \pos( - \Delsafe^{j}(z)) + \pos ( \Delta^{\epstil}(z)) + \epstil \big ) ^2} + C_0 \tag{optimality} \label{eq:complexity_optimality}
\end{align*}
samples for some $C = \poly\log(\frac{1}{\epsilon})$ and $C_0 = \poly\log(\frac{1}{\epsilon}, | \cZ|) \cdot \log \frac{1}{\delta}$.
\end{theorem}

The complexity bound given in \Cref{thm:main_complexity_nice} may, at first glance, appear rather opaque, yet it in fact yields a very intuitive interpretation. The first term in the complexity, the safety term, is the complexity needed to show each arm is safe or unsafe, \emph{if they have not otherwise been eliminated}. As described in the previous section, if $\pos(\Delta^{\epstil}(z))$ is large, this implies we have found an arm better than $z$, so learning its safety value is irrelevant.

The second term in the complexity, the optimality term, corresponds to the difficulty of showing an arm is worse \emph{than the best arm we can guarantee is safe}. Note that we can only guarantee an arm is suboptimal if we can find a safe arm with higher value.
Recall the definition of $\Delta^{\epstil}(z)$ given in \eqref{eq:Deleps}. Intuitively, $\Delta^{\epstil}(z)$ denotes the gap in value between arm $z$ and the best arm with safety gap at least $\epstil$. As we make $\epstil$ smaller, we can show additional arms are safe, which increases $\Delta^{\epstil}(z)$. While this makes it easier to show $z$ is suboptimal, it comes at a cost---the extra samples necessary to decrease our safety tolerance, given by the first term in the complexity. \safebai trades off between optimizing for each of these terms---gradually decreasing its tolerance on both the safety and optimality terms to more easily eliminate suboptimal arms, while not allocating too many samples to guarantee safety.  

To help illustrate this complexity, we consider a simple example with orthogonal arms, i.e. a multi-armed bandit example. 

\begin{example}[\safebai on Multi-Armed Bandits]\label{ex:hard_BAI_instance}
In the multi-armed bandit setting, we have $\cX = \cZ = \{ e_1,\ldots,e_d \}$. Let $m = 1$, $d = 3$, and consider the settings of $\thetast$ and $\must$ given in \Cref{fig:mab}. Here we see that arm $e_1$ is safe and has value much higher than any other arm, so $\zst = e_1$, and can be shown to be safe relatively easily; arm $e_2$ has near-optimal value but is very unsafe; and arm $e_3$ is unsafe with very small safety gap, but has the smallest value.

\begin{figure}[h!]
    \centering
    \begin{minipage}{0.5\textwidth}
        \centering
        \includegraphics[width=0.9\textwidth]{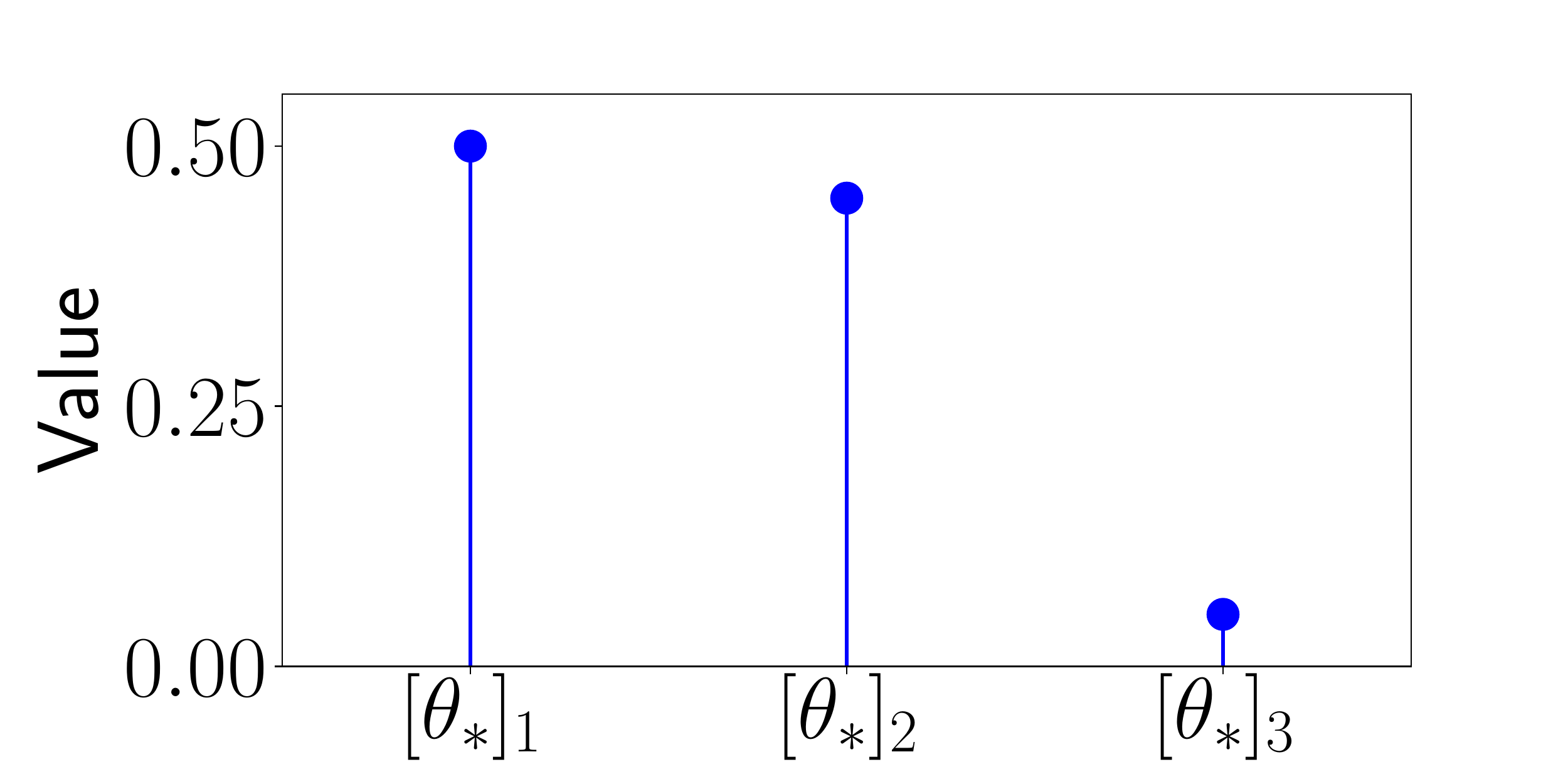} 
    \end{minipage}\hfill
    \begin{minipage}{0.5\textwidth}
        \centering
        \includegraphics[width=0.9\textwidth]{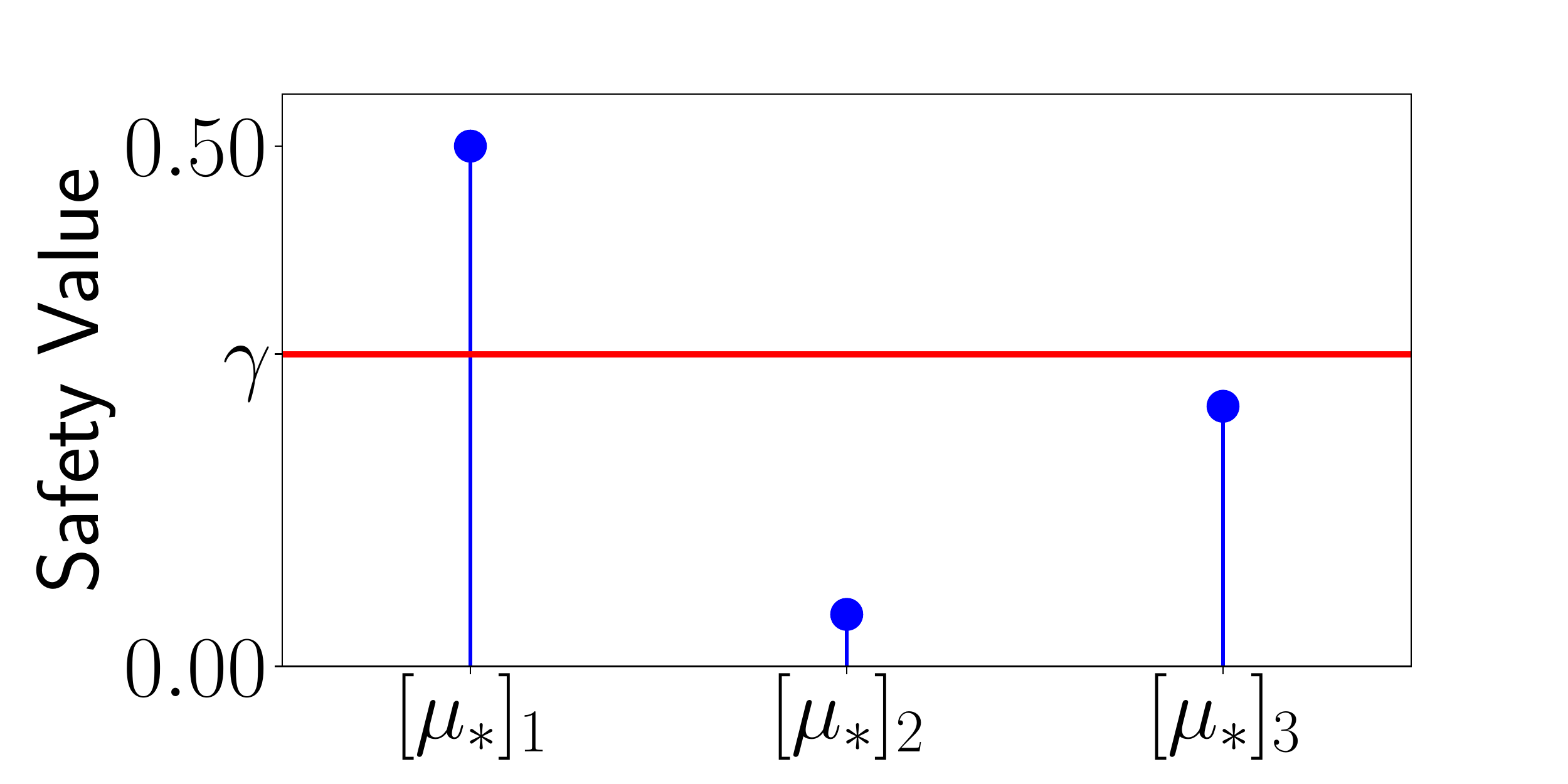}
    \end{minipage}
    \caption{Multi-Armed Bandit Instance}
    \label{fig:instancei=1}
    \label{fig:mab}
\end{figure}

\textbf{Showing $e_2$ is Suboptimal.} As $e_2$ has near-optimal value, $\Delta(e_2)$ is very small and it is very difficult to show $e_2$ is suboptimal. However, $-\Delsafe(e_2) = \cO(1)$, so it is very easy to show $e_2$ is unsafe. It follows that $\pos(-\Delsafe(e_2)) = \cO(1)$ so both denominators in our complexity will always be $\cO(1)$ for $z = e_2$---\safebai does not attempt to show $e_2$ is suboptimal, but instead shows it is unsafe, and therefore does not pay for the small optimality gap of $\Delta(e_2)$ in the complexity.

\textbf{Showing $e_3$ is Suboptimal.} Recall the definition of $\Delta^{\epsilon}(z) = \max_{z' : \Delsafe(z') \ge \epsilon} \thetast^\top (z' - z)$. In this case, for $\epsilon = \cO(1)$, we will have $\Delsafe(e_1) \ge \epsilon$, which implies that $\Delta^{\epsilon}(e_3) = \thetast^\top(e_1 - e_2) = \Delta(e_3) = \cO(1)$. To show $e_3$ is suboptimal, we could either show it is unsafe (which is very difficult) or suboptimal (which is very easy). Observing the sample complexity of \Cref{thm:main_complexity_nice}, we see that the denominator of both terms will always be $\cO(1)$ for $z = e_3$ since $\Delta^{\epsilon}(e_2)= \cO(1)$---\safebai never pays for the small safety gap of $e_3$, it instead takes advantage of the fact that $e_3$ can easily be shown to be suboptimal, and uses this to eliminate it.

In both of these cases we see that \safebai does the ``right'' thing, always using the easier of the two criteria---either showing an arm is unsafe or suboptimal---to show that $z \neq \zst$. Combining the above observations, for $\epsilon \approx \min \{ \Delta(e_3), - \Delsafe(e_2), \Delsafe(e_1) \}$, it follows that on this example the total sample complexity of \safebai given by \Cref{thm:main_complexity_nice} scales as:
\begin{align*}
    \cOtil \left ( \Big ( \frac{1}{\Delsafe(e_1)^2} + \frac{1}{\Delsafe(e_2)^2} + \frac{1}{\Delta(e_3)^2} \Big ) \cdot \log \frac{1}{\delta} \right )
\end{align*}
where the $1/\Delsafe(e_1)^2$ arises because we must also show $e_1$ is safe.



\end{example}

\subsection{Optimality of \safebai}

\paragraph{Optimality in Best-Arm Identification.}
Consider applying \safebai to a problem instance where $m = 1$, $\mu_{*,1} = 0$, and $\gamma = 1$. In this case, every arm is safe, and the safety constraints are essentially vacuous---every arm can easily be shown safe. We can therefore think of this as simply an instance of the best-arm identification problem. In this setting, we obtain the following corollary.

\begin{corollary}\label{cor:bai}
Consider running \safebai on a problem instance where $m = 1$, $\mu_{*,1} = 0$, and $\gamma = 1$, and set $\epsilon = \frac{1}{2} \max_{z \neq \zst} \thetast^\top (\zst - z)$. Then with probability at least $1-\delta$, \safebai returns $\zst$ 
and has sample complexity bounded by:
\begin{align*}
\cOtil \bigg ( \inf_{\lambda \in \simplex_\cX} \max_{z \in \cZ} \frac{\| z - \zst \|_{A(\lambda)^{-1}}^2 }{\Delta(z)^2} \cdot \log \frac{ | \cZ |}{\delta}  + \inf_{\lambda \in \simplex_\cX} \max_{z \in \cZ} \| z \|_{A(\lambda)^{-1}}^2 \cdot \log \frac{ | \cZ |}{\delta}  \bigg ) .
\end{align*}
\end{corollary}
Up to lower-order terms, this exactly matches the lower bound on best-arm identification given in \cite{fiez2019sequential}. Thus, in settings where the safety constraint is vacuous, \safebai hits the optimal rate.

\paragraph{Worst-Case Performance of \safebai.}
We next consider the worst-case performance of \safebai in settings when $\cX = \cZ$. We have the following result. 

\begin{corollary}\label{cor:worst_case_upper}
Assume that $\cX = \cZ$. Then for any $\thetast$ and $(\musti)_{i=1}^m$, the sample complexity of \safebai necessary to return an $\epsilon$-good and $\epsilon$-safe arm is bounded as $\cOtil(\frac{d}{\epsilon^2} \cdot (\log (m| \cX |) + \log \frac{1}{\delta}))$. 
\end{corollary}

Theorem 2 of \cite{wagenmaker2022reward} shows a worst-case lower bound of $\Omega(d^2/\epsilon^2)$ on the sample complexity of identifying an $\epsilon$-optimal arm in the standard linear bandit setting. Safe best-arm identification problems in which the safety constraint is vacuous are at least as hard as the standard best-arm identification problem, since at minimum we need to find the best arm out of every safe arm. Thus, $\Omega(d^2/\epsilon^2)$ is also a worst-case lower bound for the safe best-arm identification problem. The hard instance of \cite{wagenmaker2022reward} has $|\cX| = \cO(2^d)$, so it follows that on this instance, \safebai achieves a complexity of $\cOtil(\frac{d}{\epsilon^2} \cdot (d + \log \frac{1}{\delta}))$, and therefore \safebai has optimal dimensionality dependence. In addition, this also implies that safe best-arm identification, in the worst-case, is no harder than the standard best-arm identification problem---it is no harder to find the best \emph{safe} arm, regardless of the number of safety constraints, than to find the best arm, ignoring safety constraints.

\subsection{The Role of Experiment Design}\label{sec:role_of_exp_design}


We can think of the safe best-arm identification problem, in some sense, as an interpolation of the standard best-arm identification problem, as well as the level-set estimation problem, where the goal is to identify $z \in \cZ$ satisfying $z^\top \must \le \gamma$ \citep{mason2021nearly}. In the former problem, \cite{fiez2019sequential} shows that the instance-optimal rate can be attained by running a round-based algorithm and at every round solving an instance of the \XYdiff experiment design, as defined in \Cref{sec:alg}. In the latter problem, \citep{mason2021nearly} also show that a round-based algorithm can hit the instance-optimal rate, but instead solving the \XYsafe problem at each round.
It is natural to ask whether either of these strategies could be applied to the safe best-arm identification problem directly, or if it is necessary to alternate between them. The following results show that, on their own, each of these designs is unable to hit the optimal rate.

\begin{proposition}\label{prop:xy_fails}
Fix some small enough $\epsilon > 0$. Then there exist instances of the safe best-arm identification problem, $\mc{I}_{i} = (\thetast^i,\must^i,\cX^i,\cZ^i)$, $i=1,2$,  with $d = |\cX^i| = |\cZ^i| = 2$, $m = 1$, such that:
\begin{itemize}[leftmargin=*]
\item On $\mc{I}^1$, any $(\epsilon,\delta)$-PAC algorithm which plays only allocations minimizing \XYdiff must have $\Exp[\tau_{\delta}] \ge \Omega \left ( \frac{1}{\epsilon^3} \cdot \log \frac{1}{ \delta} \right )$, while \safebai identifies an $\epsilon$-optimal arm after $\cOtil(\frac{1}{\epsilon^2} \cdot \log 1/\delta)$ samples. 
\item On $\mc{I}^2$, any $(\epsilon,\delta)$-PAC algorithm which plays only allocations minimizing \XYsafe must have $\Exp[\tau_{\delta}] \ge \Omega \left ( \frac{1}{\epsilon^{3/2}} \cdot \log \frac{1}{ \delta} \right )$, while \safebai identifies an $\epsilon$-optimal arm after $\cOtil(\frac{1}{\epsilon} \cdot \log 1/\delta)$ samples. 
\end{itemize}
\end{proposition}

\Cref{prop:xy_fails} implies that, to solve the safe best-arm identification problem optimally, more care must be taken in exploring than either standard experiment design induces---we must trade off between \XYdiff and \XYsafe as \safebai does. We remark briefly on the instance $\mc{I}^1$. On this instance we have $\cX = \{ e_1, e_2 \}$ and $\cZ = \{ z_1, z_2 \}$ with $z_1 = [1/4,1/2]$ and $z_2 = [3/4,1/2+\alpha]$. We set $\thetast^1 = [1,0]$, $\must^1=[0,1]$, and $\gamma = 1/2 + \alpha/2$. Here $z_2$ is unsafe while $z_1$ is safe, so it follows that $\zst = z_1$. As $z_2^\top \thetast^1 > z_1^\top \thetast^1$, to show $z_2 \neq \zst$, we must show it is unsafe. However, if we solve the design \XYdiff, we see that it places nearly all of the mass on the first coordinate. While this would be optimal if both $z_1$ and $z_2$ were safe and we simply wished to determine which has a higher value, to show $z_2$ is unsafe, the optimal strategy places (roughly) the same mass on each coordinate, since each coordinate could contribute to the safety value. This is precisely the allocation \safebai will play, so it is able to show that $z_2$ is unsafe much more efficiently than a naive \XYdiff approach.

\section{Experiments for Safe Best Arm Identification in Linear Bandits}
We next present experimental results on \safebai to demonstrate the advantage of experimental design---especially combining \XYdiff and \XYsafe designs.
As there are no existing algorithms that consider safe best-arm identification, as a benchmark we consider the naive adaptive approach \textsc{Baseline} that first solves the problem of estimating the safety gap of each arm up to a desired tolerance, and then solves the problem of finding the best (safe) arm among the arms that were found to be safe. 
We first describe instances on which we test \safebai. Our experimental details and precise implementation of $\safebai$ using elimination are described in Section~\ref{sec:exp_details}. 

\paragraph{Multi-Armed Bandit.}
We consider a best-arm identification problem in which every arm is safe, but the arm with highest value is very difficult to identify as safe, while the second-best arm can easily be shown safe. We vary the total number of arms and run \safebai and \textsc{Baseline} with $\epsilon=0.5$ and $\delta = 0.1$. 
From Figure~\ref{fig:exp_hard_instance}, we observe that the sample complexity of \safebai is smaller (up to about two times for $100$ arms) than the sample complexity of its baseline.

\paragraph{Linear Response Model.}
\textit{Random Instance}: We also consider the more general setup where $\X, \Z\subset \R^d$, $\theta\in\R^d$ and $\mu\in\R^d$ are randomly generated from independent Gaussian random variables with mean $0$ and variance $1$. We set $|\X| = 50$ and vary the size of $|\Z|$. In Figure~\ref{fig:exp_lb}, we see again that \safebai significantly outperforms the baseline. 

\textit{Hard Instance}: We last consider the instance of Proposition~\ref{prop:xy_fails} and benchmark against the strategy playing only allocations minimizing \XYdiff. In Figure~\ref{fig:exp_hardlb}, we see again that \safebai significantly outperforms this baseline, corroborating the theoretical result of Proposition~\ref{prop:xy_fails}.

\begin{figure}[tbh]
\centering
\begin{minipage}{.3\textwidth}
  \centering
  \includegraphics[width=1\linewidth]{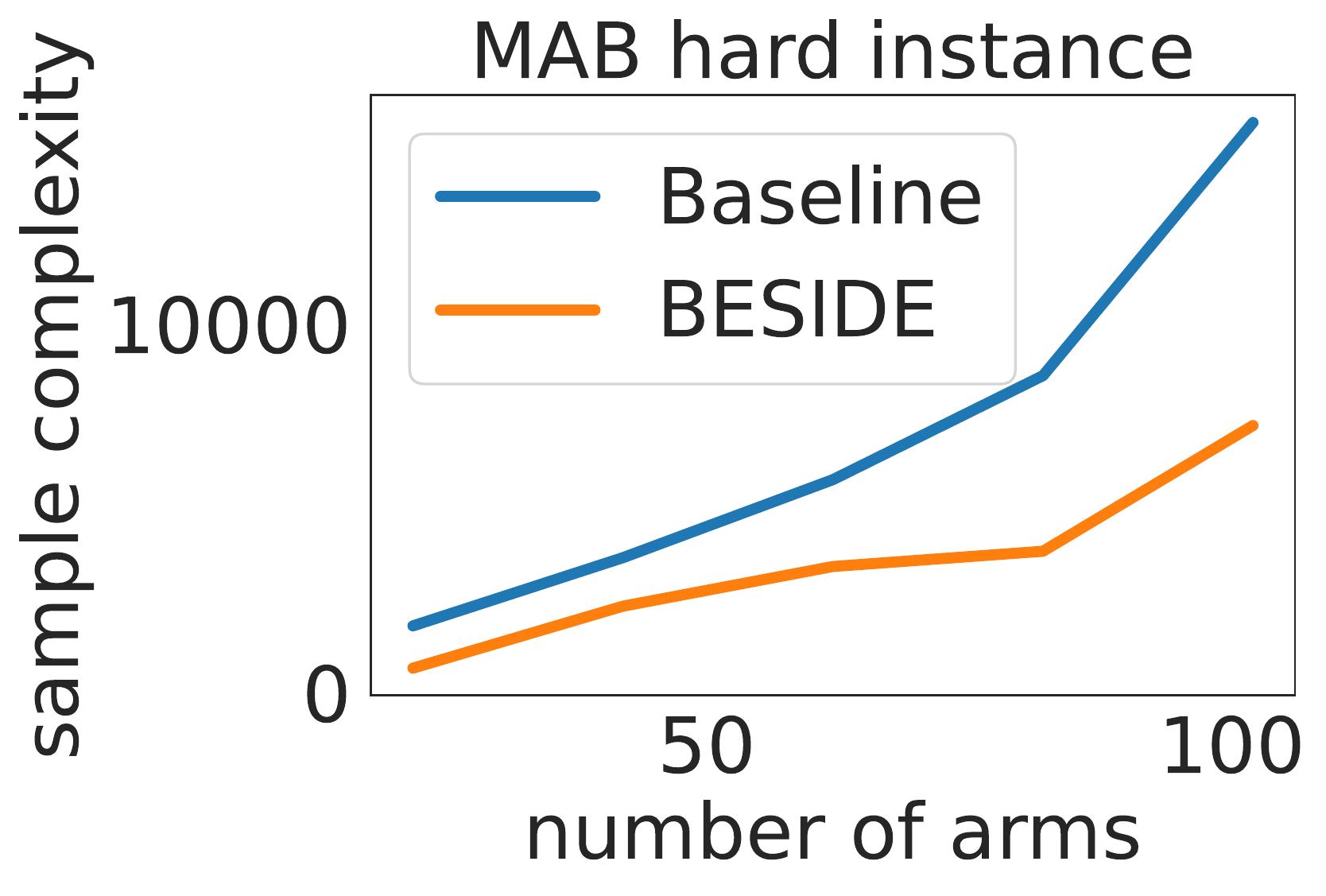}
  \vspace*{-8mm}
  \caption{Total arm pulls to termination vs. number of arms}
  \label{fig:exp_hard_instance}
\end{minipage}%
\hspace{5.00mm}
\begin{minipage}{.3\textwidth}
  \centering
  \includegraphics[width=1\linewidth]{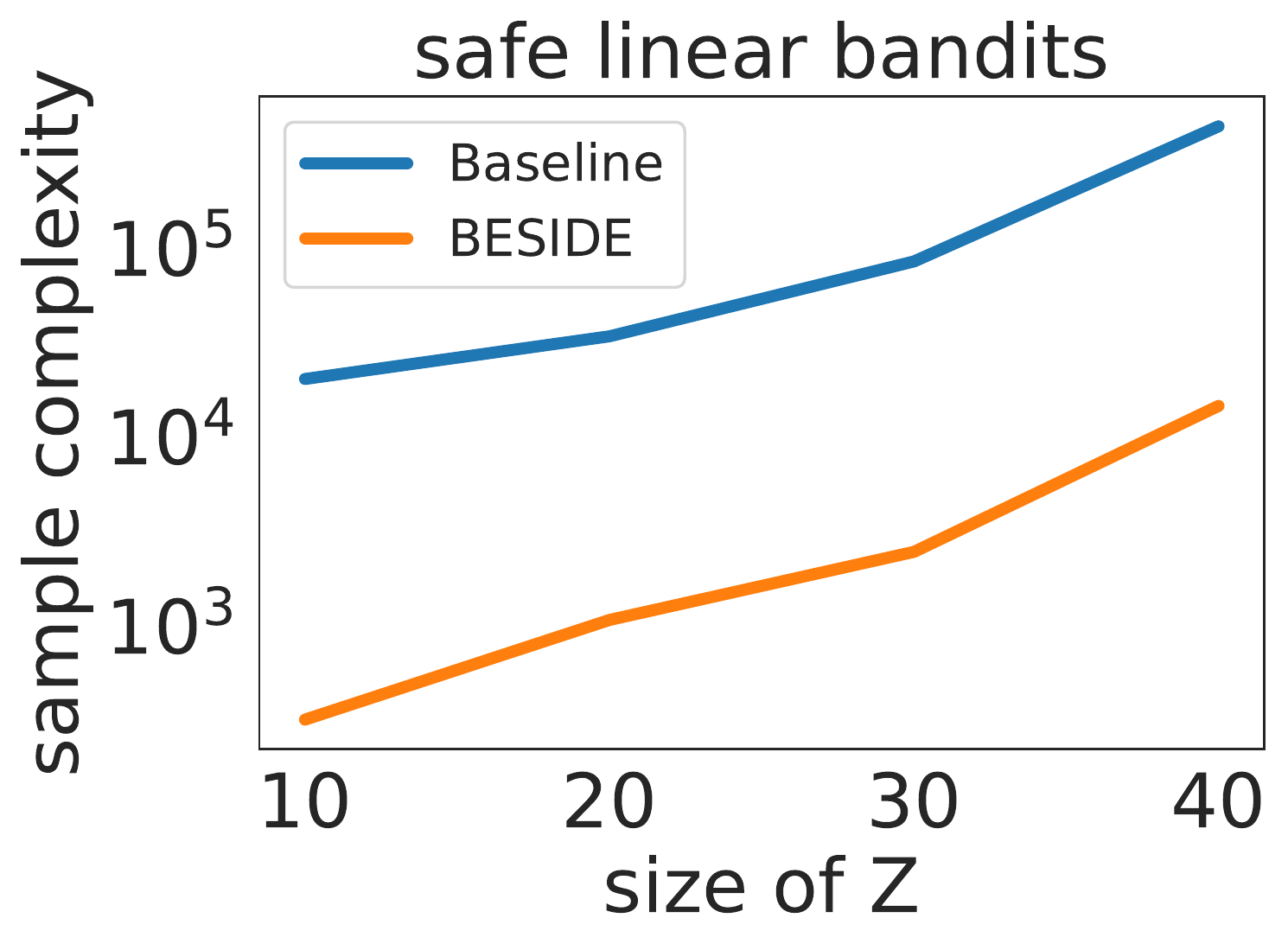}
  \vspace*{-8mm}
  \caption{Total arm pulls to termination vs. $|\Z|$}
  \label{fig:exp_lb}
\end{minipage}
\hspace{5.00mm}
\begin{minipage}{.3\textwidth}
  \centering
  \includegraphics[width=1\linewidth]{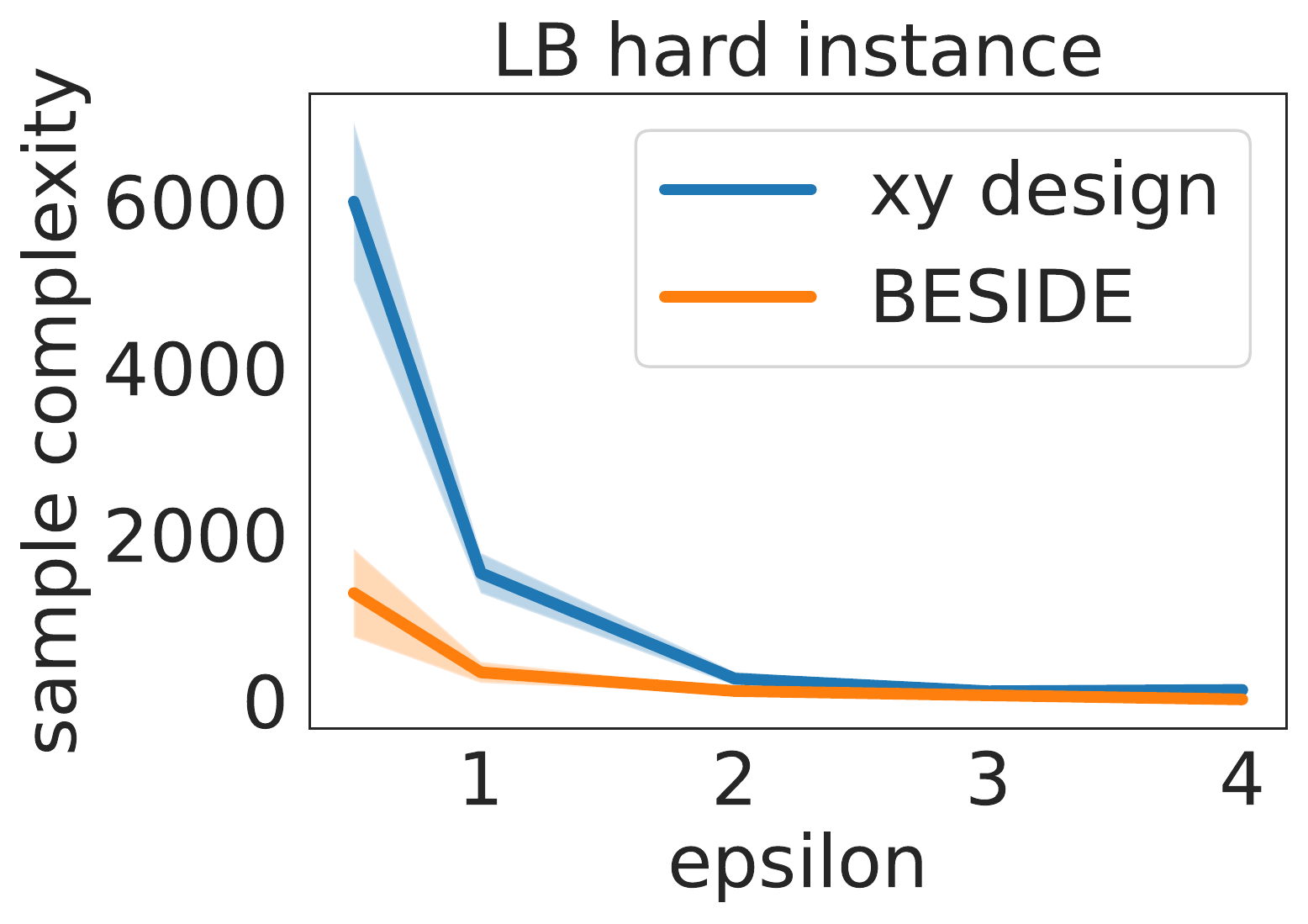}
  \vspace*{-8mm}
  \caption{Total arm pulls to termination vs. $\epsilon$}
  \label{fig:exp_hardlb}
\end{minipage}
\end{figure}
\vspace*{1mm}

\newcommand{\FDR}{\mathrm{FDR}}
\vspace{5pt}
\subsection{Practical Algorithms for Active Classification Under Constraints }\label{sec:practical_exp}
Next, we provide an application of the above ideas to pool-based active classification with constraints---namely, adaptive sampling to learn the highest accuracy classifier with a constraint on the false discovery rate (FDR). We first explain how this problem maps to the linear bandit setting.  Precisely, let $\X$ be the example space and $\mc{Y} = \{0,1\}$ the label space. Fix a hypothesis class $\mc{H}$ such that each $h \in \mc{H}$ is a classifier $h: \mc{X} \rightarrow \mc{Y}$. We represent each $h$ with an associated indicator vector $z_{h}\in \{0,1\}^{|\X|}$ where $z_h(x) =1\iff h(x) = 1$. 
Similarly, let $\eta\in [0,1]^{|\mc{X}|}$ represent the label distribution, i.e. 
$\eta(x) = \P(Y = 1 | X=x)$. 
Then the risk of a classifier $R(h) := \E_{x\sim\text{Unif}(\X) ,Y\sim \Ber(\eta(x))}[\1[h(x)\neq Y]] = z_h^{\top}(2\eta-1)$ 
and the FDR is defined as $\FDR(h) := (\mathbf{1}-\eta)^{\top}z/\mathbf{1}^{\top}z$. 
In the case when $\eta\in \{0,1\}^{|\X|}$, $\FDR(h)$ is the proportion of examples that $h$ incorrectly labels as $1$ out of all examples $h$ labels as $1$. Our goal is to solve the following constrained best arm identification problem:
\begin{align}\label{eqn:FDRconstraint}
   \widehat{h} = \min_{h\in\H}R (h)\quad \text{s.t.} \quad \FDR(h) \leq q \iff \min_{h\in \H} z_h^{\top}\eta \quad \text{s.t.} \quad ((\mathbf{1}-\eta)^{\top} - q\mathbf{1}^{\top})^{\top}z \leq 0.
\end{align}

The main challenge in running \safebai on this problem directly is a potentially high computational cost from computing a design over an extremely large hypothesis class $\cH$ (e.g. neural networks of a bounded width). In this section we provide an alternative approach motivated by \safebai. Algorithm~\ref{alg:active_cc} follows a similar design as \safebai and relies on an oracle, \texttt{CERM}, that can solve \eqref{eqn:FDRconstraint}, i.e. given a dataset it returns the highest accuracy classifier under an FDR constraint. Such oracles are available in, for example in \cite{agarwal2018reductions, cotter2018training}. In each round of \Cref{alg:active_cc} we perform \textit{randomized exploration} by perturbing the labels on our existing dataset with mean zero Gaussian noise, and then training $k$ classifiers $\widehat{h}_i, i\in [k]$, on the resulting datasets.  Implicitly, we are making the assumption that the loss function in the training of $\texttt{ERM}$ can handle continuous labels, such as the MLE of logistic regression. As described in \cite{kveton2019randomized}, randomized exploration emulates sampling from a posterior distribution on our possible set of classifiers. We then use the labels generated from these classifiers to compute safe classifiers $h_i, i\in [k]$. Finally, mimicking the strategy of \safebai, we compute \XYsafe and \XYdiff designs on these $k$ safe classifiers and repeat (note that the designs computed on Line \ref{line:eff_Designs} are equivalent to \XYsafe and \XYdiff in the classification setting).

\setlength{\textfloatsep}{0pt}
\begin{algorithm}[h!]
\caption{\texttt{Active constrained classification with randomized exploration}}
\label{alg:active_cc}
\begin{algorithmic}[1]
\Require{Batch size $n$, initial (labeled) data $x^{(0)}_1, \ldots, x^{(0)}_n$, number of rounds $L$, number of classifiers per round $k$, perturbation variance $\sigma$}
\For{$\ell = 1, \ldots, L$}
\For{$i = 1, \ldots, k$}

\State{$\hh_i = \texttt{ERM}(\{(x^{(\ell)}_t, y^{(\ell)}_t+\epsilon^{(i)}_t)\}_{t=1}^n)$, where $\{\epsilon^{(i)}_t\}_{1\leq t \leq n} \overset{i.i.d.}{\sim} {\cal{N}}(0, \sigma^2)$}
\State{$h_i = \texttt{CERM}(\{(x, \hh_i(x))\}_{x\in\X})$}
\EndFor
\State{Compute designs: $\lambda_{\text{safe}} = \arg\min_{\lambda\in\triangle_\X}\max_{1\leq i \leq k}\sum_{x\in \mc{X}} \frac{\1\{h_i(x)\neq 0\}}{\lambda_x}$, $\lambda_{\text{diff}} = \arg\min_{\lambda\in\triangle_\X}\max_{1\leq i\neq j\leq k}\sum_{x\in \mc{X}}\frac{\1\{h_i(x)\neq h_j(x)\}}{\lambda_x}$

}\label{line:eff_Designs}
\State{Sample $x^{(\ell)}_1, \ldots, x^{(\ell)}_n$ from a uniform mixture of $\lambda_{\text{safe}}, \lambda_{\text{diff}}$}
\State{Observe corresponding labels $y^{(\ell)}_1, \ldots, y^{(\ell)}_n$}
\EndFor
\Return{$\tih = \texttt{CERM}(\{(x^{(\ell)}_t, y^{(\ell)}_t)\}_{1\leq t \leq n, 0\leq \ell \leq L})$}
\end{algorithmic} 
\end{algorithm} 

To validate Algorithm~\ref{alg:active_cc}, we experiment against a passive baseline that selects points uniformly at randoms from the pool of examples $\X$, retrains the model using the same Constrained Empirical Risk Minimization oracle (\texttt{CERM}) as Algorithm~\ref{alg:active_cc} on its current samples, and report the accuracy and FDR. We evaluate on  two real world datasets next and provide an additional experiment on a synthetic dataset in Section~\ref{sec:exp_details}.

\paragraph{Adult dataset.}
\begin{wrapfigure}[16]{r}{.4\textwidth}
    \includegraphics[width=.4\textwidth]{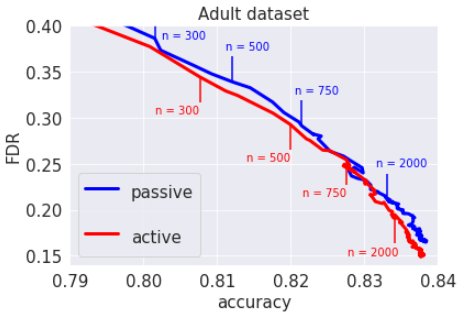}
    \vspace{-0.5em}
    \caption{\footnotesize FDR vs accuracy for active (Algorithm~\ref{alg:active_cc}) and passive sampling, ticks report number of samples. FDR and accuracy are averaged over $5$ trials}
    \label{fig:exp_adult}
\end{wrapfigure}
We evaluate on the adult income data set \cite{lichman2013UCI} (48,842 examples) where the goal is to predict whether someone's income is above \$50k per year. We report in Figure~\ref{fig:exp_adult} the accuracy and the FDR obtained when varying the number of labels given to each method. We observe that for any desired accuracy Algorithm~\ref{alg:active_cc} allows us to provide a classifier with lower FDR. Also, for any chosen number of total labels---such as $500, 750, 2000$ as reported in Figure~\ref{fig:exp_adult}---the Algorithm~\ref{alg:active_cc} gives a classifier with higher accuracy and lower FDR. In general we found that the active method needed half the number of samples as the passive sampling to achieve a given FDR. This demonstrates the effectiveness of Algorithm~\ref{alg:active_cc} to learn simultaneously the objective (risk) and the constraint (FDR), in a similar favorable way as characterized by our theoretical findings. 

\paragraph{German Credit dataset.}
\begin{wrapfigure}[14]{r}{.4\textwidth}
    \includegraphics[width=.4\textwidth]{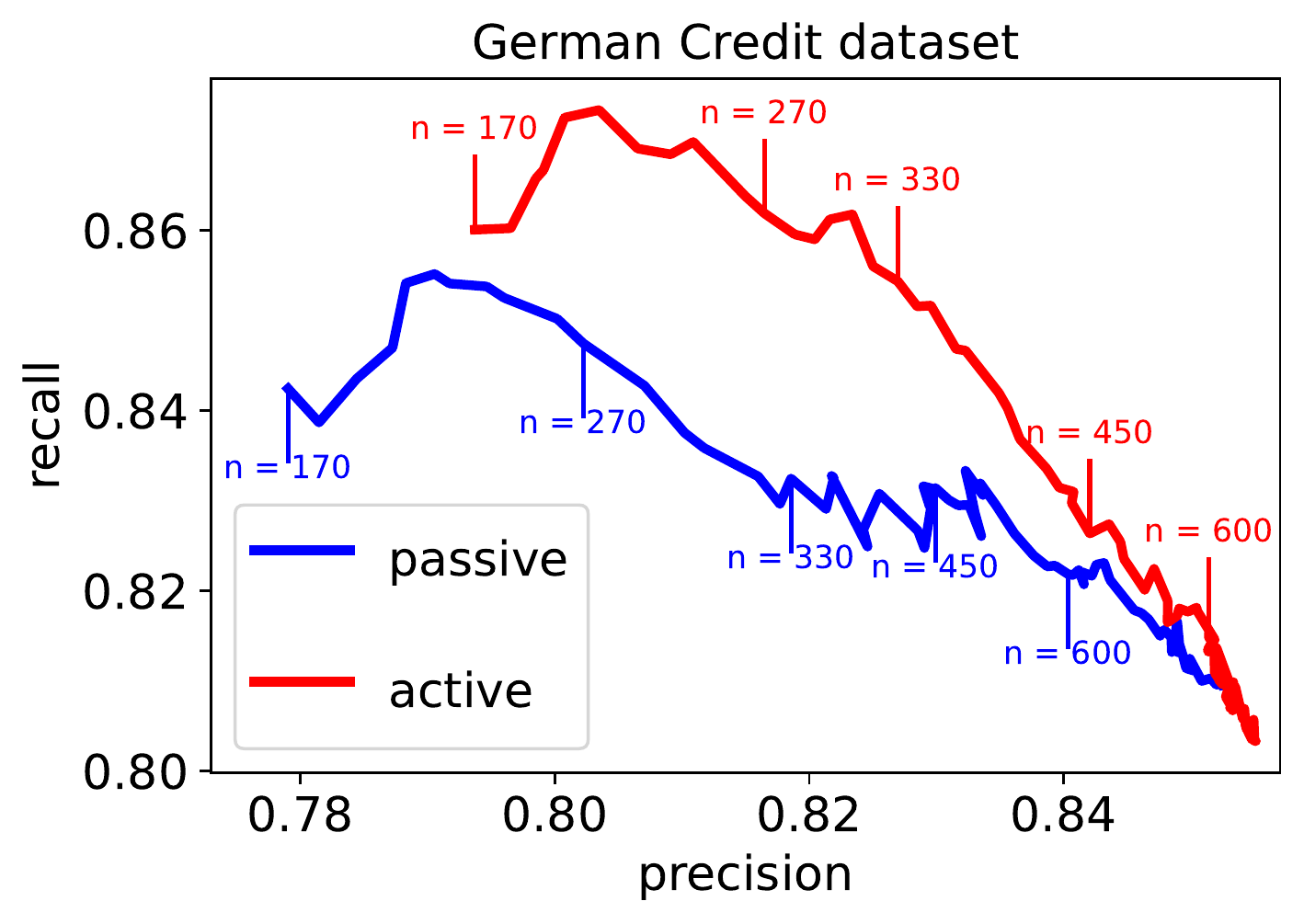}
    \vspace{-0.5em}
    \caption{\footnotesize TPR vs FDR for active active (Algorithm~\ref{alg:active_cc}) and passive sampling, ticks report number of samples. Precision is $1-\text{FDR}$, recall is TPR. Precision and recall are averaged over $25$ trials}
    \label{fig:exp_germancredit}
\end{wrapfigure}
We consider the German Credit Dataset originally from the Staflog Project Databases \cite{keogh1998UCI}. The goal is to predict whether someone's credit is 'bad' or 'good'. We report in Figure~\ref{fig:exp_germancredit} the recall (TPR) and the precision ($1-\text{FDR}$) obtained when varying the number of labels given to each method. We observe that for any desired precision Algorithm~\ref{alg:active_cc} allows us to provide a classifier with higher recall. 
Also, for any chosen number of total labels---such as $170, 270, 330, 450, 600$ as reported in Figure~\ref{fig:exp_germancredit}---the Algorithm~\ref{alg:active_cc} gives a classifier with higher precision and higher recall. As for the Adult dataset we found that the active method needed half the number of samples as the passive sampling to achieve a given precision. 

\section{Related works}


\paragraph{Constrained Bandits.}
A growing body of work seeks to address the question of safe learning in interactive environments. In particular, the majority of such works have considered the problem of regret minimization in linear bandits with linear safety constraints. Here, the goal is to maximize online reward, $x_t^\top \thetast$, by choosing actions $x_t \in \cX \subseteq \R^d$, while ensuring a safety constraint of the form $x_t^\top \must \le \gamma$ is met at all times (either in expectation or with high probability). A variety of algorithms have been proposed, including UCB-style \citep{kazerouni2017conservative,amani2019safety,pacchiano2020constraints}, and Thompson Sampling \citep{ahmadreza2019safe, ahmadreza2020stage}. While these works show that $\sqrt{T}$ regret is attainable, they only provide worst-case bounds (while we obtain instance-dependent bounds) and do not study the pure-exploration best-arm identification problem. 

To our knowledge, only several existing works consider the question of best-arm identification with safety constraints \citep{sui2015safe,sui2018stagewise,wang21safety}.  \citep{sui2015safe,sui2018stagewise} consider a general constrained optimization setting where the goal of the learner is to minimize some function $f(x)$ over a domain $x \in \cD$, while only having access to noisy samples of $f(x)$, $f(x_t) + w_t$, and guaranteeing that a safety constraint $g(x_t) \ge h$ is met for every query point $x_t$. While they do provide a sample complexity upper bound, they give no lower bound, and, as shown in \citep{wang21safety}, their approach can be very suboptimal. \citep{wang21safety} considers the setting of best-arm identification in multi-armed bandits. In their setting, at every step $t$ they query a value $a_t \in \cA$ for a particular coordinate $i_t$, and their goal is to identify the coordinate $i^*$ such that $a^*_{i^*} \theta_{i^*} \ge \max_i a^*_i \theta_i$, where $a^*_i$ is the largest value respecting the safety constraint: $a^*_i = \argmax_{a \in \cA} a \theta_i \text{ s.t. } a \mu_i \le \gamma$. Similar to \citep{sui2015safe,sui2018stagewise}, they require that the safety constraint $a_t \mu_{i_t} \le \gamma$ must be met while learning. Though they do show matching upper and lower bounds, and in addition consider a slightly more general setting that allows for nonlinear (but monotonic) response functions, they treat every coordinate as independent, and do not allow for information-sharing between coordinates---the key generalization the linear bandit setting targets. We remark as well that in our setting, unlike these works, we allow the learner to query unsafe points during exploration, and only require that they output a safe decision at termination.

\paragraph{Best-Arm Identification in Linear Bandits.}
The best-arm identification problem in multi-armed bandits (without safety constraints) is a classical and well-studied problem \citep{bechhofer1958sequential,paulson1964sequential,even2002pac,bubeck2009pure}, and near-optimal algorithms exist \citep{jamieson2014lil,kaufmann2016complexity}. More recently, there has been a growing interest in understanding the sample complexity of best-arm identification in linear bandits \citep{soare2014best,karnin2016verification,xu2018fully,fiez2019sequential,katz2020empirical,degenne2020gamification}. We highlight in particular the work of \cite{fiez2019sequential} which proposes an experiment-design based algorithm, \textsc{Rage}, that our approach takes inspiration from. While much progress has been made in understanding best-arm identification in linear bandits, to our knowledge, no existing works consider the setting of best-arm identification in linear bandits with safety constraints, the setting of this work.

\paragraph{Active Classification under FDR constraints} We finally mention one other related body of work---the problem of actively sampling to find a classifier with high accuracy or recall under precision constraints. Motivated by the experimental design approach of our main algorithm, \safebai, we provide a heuristic algorithm for this problem with good empirical performance in Section~\ref{sec:practical_exp}. There is an extensive body of work on active learning (see the survey~\cite{hanneke2014theory}) but only recently have works made the connection between best-arm identification for linear bandits and classification~\cite{katzsamuels2021improved, jain2020new, camilleri2021selective}. Precision constraints has been less studied in the adaptive context, we only know of \cite{jain2020new, bennett2017algorithms}. 



\section{Conclusion}
In this work we have shown that it is possible to efficiently find the best \emph{safe} arm in linear bandits with a carefully designed adaptive experiment design-based approach. Our results open up several interesting directions for future work.

\paragraph{Instance Optimality.} 
While \safebai is worst-case optimal, in \Cref{sec:lower_bounds} we show an instance-dependent lower bound which \safebai does not, in general, seem to hit. We conjecture that this lower bound may be loose---addressing this discrepancy and showing matching instance-dependent upper and lower bounds is an exciting direction for future work.

\paragraph{Safety During Exploration.}
Though there are many interesting applications where we may not require safety during exploration (i.e. only querying safe arms), in other cases we may need to ensure safety is met during exploration. Extending our work to this setting is an interesting open problem.

\paragraph{Potential Impacts.} As with any algorithm making stochastic assumptions, if assumptions are not met we can not guarantee the performance. In this case, one limitation is that if the underlying environment is changing (i.e. the constraints vary over time) the algorithm could have unexpected behavior with unintended consequences. Such a situation could lead to harmful results in examples such as the online advertising bidding example from the introduction. To mitigate this limitation of our setting, practitioners are encouraged to monitor many metrics, both short and long-term.

\subsection*{Acknowledgements}
The work of AW was supported by an NSF GFRP Fellowship DGE-1762114. The work of JM was supported by an NSF Career award, and NSF AI institute (IFML) and the Simons collaborative grant on the foundations of fairness. The work of KJ was funded in part by the AFRL and NSF TRIPODS 2023166.

\newpage
\bibliography{main}
\newpage
\appendix

\tableofcontents
\newpage

\newcommand{\Calt}{\mathcal{C}_{\mathrm{alt}}}

\section{Lower Bounds}\label{sec:lower_bounds}

\subsection{Oracle Lower Bound}

\begin{theorem}[Oracle Lower Bound]\label{thm:lb}
Let $\tau$ denote the stopping time for any $(0, \delta)$-PAC algorithm for pure exploration in safe linear bandits. Then
\begin{align*}
\frac{\E_{\thetast,\must}[\tau]}{\log \frac{1}{2.4 \delta}} &\ge \min_{\lambda \in \simplex_\cX} \max\left\{\max_{z \in \Z \setminus \zst}  \min \left \{ \frac{ \|z\|^2_{A(\lambda)^{-1}}}{ \pos(-\Delsafe(z))^2  }, \frac{\|z-\zst\|^2_{A(\lambda)^{-1}}}{ \pos(\Delta(z))^2} \right \}, \frac{\|\zst\|_{A(\lambda)^{-1}}^2}{(\zst^\top\must - \alpha)^2}\right\} .
\end{align*}
\end{theorem}

\paragraph{Comparing Complexity with \Cref{thm:main_complexity_nice}.}

In the single-constraint setting, the complexity of \safebai reduces to 
\begin{align*}
& C \cdot \sup_{\epstil \ge \epsilon} \inf_{\lambda \in \simplex_\cX} \max_{z \in \cZ} \frac{ \| z \|_{A(\lambda)^{-1}}^2 \cdot \log(\frac{m | \cZ |}{\delta})}{\big (  |\Delsafe(z)|   + \pos(\Delta^{\epstil}(z)) + \epstil \big )^2} \\
& \qquad + C \cdot \sup_{\epstil \ge \epsilon} \inf_{\lambda \in \simplex_\cX} \max_{z \in \cZ} \frac{\| z - \zst \|_{A(\lambda)^{-1}}^2 \cdot \log(\tfrac{ |\cZ|}{\delta})}{\big (  \pos( - \Delsafe(z)) + \pos ( \Delta^{\epstil}(z)) + \epstil \big ) ^2} + C_0 
\end{align*}
Consider the case when $\Delta^{\epstil}(z)$ is ``smooth'' in $\epstil$, in the sense that $\Delta^{\epstil}(z) \ge \Delta(z) - \epstil$. This condition corresponds to the case, for example, where $\zst$ has a large safety gap (in which case we simply have $\Delta^{\epstil}(z) = \Delta(z)$ for moderate values of $\epstil$), or where $\zst$ might have a small safety gap, but where there are arms placed at even intervals so that, as we let the safety gap get smaller, we are always able to find better arms. Under this assumption, the complexity can be upper bounded as
\begin{align*}
& C \cdot \inf_{\lambda \in \simplex_\cX} \max_{z \in \cZ} \frac{ \| z \|_{A(\lambda)^{-1}}^2 \cdot \log(\frac{m | \cZ |}{\delta})}{\big (  |\Delsafe(z)|   + \pos(\Delta(z)) \big )^2}  + C \cdot \inf_{\lambda \in \simplex_\cX} \max_{z \in \cZ \backslash \zst} \frac{\| z - \zst \|_{A(\lambda)^{-1}}^2 \cdot \log(\tfrac{ |\cZ|}{\delta})}{\big (  \pos( - \Delsafe(z)) + \pos ( \Delta(z))  \big ) ^2} + C_0 \\
& \le C \cdot \inf_{\lambda \in \simplex_\cX} \max_{z \in \cZ}  \frac{ \| z \|_{A(\lambda)^{-1}}^2 \cdot \log(\frac{m | \cZ |}{\delta})}{\big (  |\Delsafe(z)|   + \pos(\Delta(z)) \big )^2}  + C \cdot \inf_{\lambda \in \simplex_\cX} \max_{z \in \cZ \backslash \zst} \frac{\| z - \zst \|_{A(\lambda)^{-1}}^2 \cdot \log(\tfrac{ |\cZ|}{\delta})}{\big (  \pos( - \Delsafe(z)) + \pos ( \Delta(z))  \big ) ^2} + C_0 
\end{align*}
which can be upper bounded as
\begin{align*}
C \log(\tfrac{m | \cZ |}{\delta}) \cdot \left (  \inf_{\lambda \in \simplex_\cX} \max_{z \in \cZ}  \frac{ \| z \|_{A(\lambda)^{-1}}^2}{\max \{ \Delsafe(z)^2, \pos(\Delta(z))^2 \}} + \inf_{\lambda \in \simplex_\cX} \max_{z \in \cZ \backslash \zst}  \frac{\| z - \zst \|_{A(\lambda)^{-1}}^2}{\max \{ \pos( - \Delsafe(z))^2, \pos ( \Delta(z))^2 \}}  \right ) + C_0 .
\end{align*}
While this does not match the lower bound of \Cref{thm:lb} exactly, it scales in a similar manner. As in \Cref{thm:lb}, we pay only for the larger of the optimality gap, $\pos ( \Delta(z))$, and safety gap $\pos( - \Delsafe(z))$ (if the arm is unsafe). The primary difference between \Cref{thm:lb} and this complexity are the terms in the numerator---in \Cref{thm:lb}, the numerator scales as $\| z - \zst \|_{A(\lambda)^{-1}}^2$ only if an arm is easier to eliminate by showing it is suboptimal, while in our complexity it could scale this way in either case.

The primary difficulty in hitting the lower bound exactly is that \Cref{thm:lb} is a \emph{verification} lower bound. It assumes knowledge of the best arm, and is told whether every other arm has smaller safety gap (if the arm is unsafe) or optimality gap. It can therefore simply use this knowledge to focus all samples on verifying an arm is either unsafe, or suboptimal. 

In practice, we do not have access to such information. Without knowing whether it is easier to eliminate an arm by showing it is unsafe or suboptimal, the best we can hope to do is to seek to estimate both the safety value and reward value of every arm, until we have estimated one well enough to show the arm is suboptimal or unsafe.

We conjecture that the lower bound of \Cref{thm:lb} is loose, and that \Cref{thm:main_complexity_nice} is nearly optimal. 
We believe the gap arises because, as noted, lower bound proof techniques, such as those proposed in \cite{kaufmann2016complexity}, which is what we rely on to prove \Cref{thm:lb}, are lower bounding only the complexity of verifying the optimal solution. In problem settings such as ours where the \emph{order} matters---where we will obtain a very different rate if we focus our attention on one arm versus another, to show it is safe or unsafe---such techniques appear insufficient to obtain a tight lower bound. Indeed, we conjecture that a ``moderate-confidence'' lower bound can be shown using techniques from \cite{simchowitz2017simulator}, and that such a lower bound may have a complexity nearly matching that of \Cref{thm:main_complexity_nice}. We leave proving this for future work.

\begin{proof}[Proof of \Cref{thm:lb}]
Following the proof of Theorem 1 of \cite{fiez2019sequential} and applying the Transportation Lemma of \cite{kaufmann2016complexity}, we have that any $\delta$-PAC algorithm must satisfy
\begin{align*}
\sum_{x \in \cX} \Exp[T_x] \ge \log \frac{1}{2.4\delta} \cdot \inf_{\lambda \in \simplex_\cX} \frac{1}{\min_{(\theta,\mu) \in \Calt} \sum_{x \in \cX} \lambda_x \KL(\nu_{(\thetast,\must),i} || \nu_{(\theta,\mu),i})}
\end{align*}
there $T_x$ denotes the number of pulls to arm $x$, and $\Calt$ is the set of alternate instances defined in \Cref{lem:lb_alg_set_proj}. As we assume that the noise is $\cN(0,1)$, and since the noise is independent for the safety observations and reward observations, we have
\begin{align*}
\KL(\nu_{(\thetast,\must),i} || \nu_{(\theta,\mu),i}) = \frac{1}{2} (x_i^\top (\thetast - \theta))^2 +\frac{1}{2} (x_i^\top (\must - \mu))^2.
\end{align*}
Some algebra shows that
\begin{align*}
 \sum_{x \in \cX} \lambda_x \KL(\nu_{(\thetast,\must),i} || \nu_{(\theta,\mu),i}) = \frac{1}{2} \| \thetast - \theta \|_{A(\lambda)}^2 + \frac{1}{2} \| \must - \mu \|_{A(\lambda)}^2.
\end{align*}
The result then follows by applying \Cref{lem:lb_alg_set_proj} to compute
\begin{align*}
\min_{(\theta,\mu) \in \Calt} \frac{1}{2} \| \thetast - \theta \|_{A(\lambda)}^2 + \frac{1}{2} \| \must - \mu \|_{A(\lambda)}^2.
\end{align*}
\end{proof}

\begin{lemma}\label{lem:lb_alg_set_proj}
Define the alternate set:
\begin{align*}
    \Calt = \{(\theta, \mu) \;\text{s.t.}\; \mu^\top z^* > \alpha\}\cup \{(\theta, \mu) \;\text{s.t.}\; \exists z'\neq z^*, \mu^\top z' \leq \alpha, \theta^\top(z^*-z')\leq 0\},
\end{align*}
Then the projection to the alternate is
\begin{align*}
&\min_{(\theta, \mu)\in \Calt}\|\theta - \thetast\|_{A(\lambda)}^2 + \|\mu - \must\|_{A(\lambda)}^2 = \min\left\{\min_{z\neq z^*}\frac{\pos(z^\top \must - \alpha)^2 }{ \|z\|^2_{A(\lambda)^{-1}}} + \frac{\pos((\zst-z)^\top \thetast)^2}{\|z-\zst\|^2_{A(\lambda)^{-1}}},\frac{(\zst^\top \must - \alpha)^2}{\|\zst\|^2_{A(\lambda)^{-1}}}\right\}.
\end{align*}
\end{lemma}
\begin{proof}
For each arm $x$ the associated and we want to solve
$$
\min_{(\theta, \mu)\in \Calt}\|\theta - \thetast\|^2_{\sum_{x\in\X}\lambda_x xx^\top} + \|\mu - \must\|^2_{\sum_{x\in\X}\lambda_x xx^\top}.
$$

To do so, we use that $\min_{x\in A\cup B}f(x) = \min_{S \in \{A, B\}}\min_{x\in S}f(x)$ on the quadratic objective by defining the sets
$$
A := \{(\theta, \mu) \;\text{s.t.}\; \mu^\top z^* > \alpha\} \;, \; B =  \{(\theta, \mu) \;\text{s.t.}\; \exists z'\neq z^*, \mu^\top z' \leq \alpha, \theta^\top(z^*-z')\leq 0\},
$$
such that their union is $A\cup B = \Calt$.

Note that we know from \cite{mason2021nearly} that
$$
\min_{(\theta, \mu)\in A}\|\theta - \thetast\|^2_{\sum_{x\in\X}\lambda_x xx^\top} + \|\mu - \must\|^2_{\sum_{x\in\X}\lambda_x xx^\top} = \frac{(\zst^\top \must - \alpha)^2}{\|\zst\|^2_{A(\lambda)^{-1}}}.
$$

We now lift $B$ to a set $\text{lift}(B)$ that is defined as
$$\text{lift}(B) = \{[\theta, \mu] \;\text{s.t.}\; \exists z'\neq z^*, [\theta, \mu]^\top [(z^*-z'), 0; 0, z'] \leq [0,\alpha]\}. $$
Thus we can focus on $D_z = \{ \kappa \in \R^{2n}\;\text{s.t.}\; A_z \kappa  \leq b \}$ where $A_z = [(z^*-z'), 0; 0, z'] \in R^{2\times 2n}$. Now we want to solve
$$
\min_{z\in\Z\setminus\{\zst\}}\min_{\kappa \in \R^{2n}\;\text{s.t.}\; A_z \kappa \leq b}\|\kappa - \kappa_*\|_{\Gamma},
$$
where $\Gamma = I_2 \bigotimes
\left(\sum_{x\in\X}\lambda_x xx^\top\right)$.

\begin{lemma}\label{lem:lb_opt_sln}
The optimal solution of
\begin{align*}
    \min_{\kappa \in \R^{2n}\;\text{s.t.}\; A \kappa \leq b}\frac{\|\kappa - \kappa_*\|_{\Gamma}}{2}
\end{align*}
is $\kappa_0 = \kappa_* - \Gamma^{-1}A^\top(A \Gamma^{-1}A^\top)^{-1}\{ A \kappa_* - b \}_+$ and the optimal value is 
\begin{align*}
    \frac{1}{2} \pos( A \kappa_* - b )^\top(A \Gamma^{-1}A^\top)^{-1} \pos(A \kappa_* - b ),
\end{align*}
where $\pos(\cdot)$ is applied element-wise to $A \kappa_* - b$. 
\end{lemma}

This translate to 
\begin{align*}
\min_{(\theta, \mu)\in B}\|\theta - \thetast\|^2_{\sum_{x\in\X}\lambda_x xx^\top} + \|\mu - \must\|^2_{\sum_{x\in\X}\lambda_x xx^\top} = \min_{z\neq z^*}\frac{\pos(z^\top \must - \alpha)^2 }{ \|z\|^2_{A(\lambda)^{-1}}} + \frac{\pos((\zst-z)^\top \thetast)^2}{\|z-\zst\|^2_{A(\lambda)^{-1}}},
\end{align*}
and we get the desired result.

\end{proof}
\begin{proof}[Proof of \Cref{lem:lb_opt_sln}]
Consider the Lagrangian
\begin{align*}
    \L(\kappa, \mu) &= \frac{1}{2}(\kappa - \kappa_*)^\top \Gamma (\kappa - \kappa_*) + \mu^\top( A \kappa - b) \\
    \L(\kappa', \mu) &= \frac{1}{2}\kappa'^\top \Gamma \kappa' + \mu^\top( A \kappa' - b + A \kappa_*) 
\end{align*}
minimized at $\kappa'_0 = - \Gamma^{-1}A^\top \mu$. We have 
\begin{align*}
    \max_{\mu\geq 0}\min_{\kappa'}\L(\kappa', \mu) &= \max_{\mu\geq 0}\frac{1}{2}(\Gamma^{-1}A^\top \mu)^\top \Gamma (\Gamma^{-1}A^\top \mu) + \mu^\top( - A \Gamma^{-1}A^\top \mu - b + A \kappa_*) \\
    &= \max_{\mu\geq 0}-\frac{1}{2}\mu^\top A\Gamma^{-1}A^\top \mu + \mu^\top( A \kappa_* - b )
\end{align*}
maximized at $\mu_0 = (A\Gamma^{-1} A\top)^{-1}\{ A \kappa_* - b \}_+$ where $\{[b_1, b_2]\}_+ = [\max\{b_1, 0\}, \max\{b_2, 0\}]$. Plugging $\mu_0$ back in the solution $\kappa'_0$, we get the solution $\kappa_0$
\begin{align*}
    \kappa_0 = \kappa_* - \Gamma^{-1}A^\top(A \Gamma^{-1}A^\top)^{-1}\{ A \kappa_* - b \}_+
\end{align*}
and the optimal value follows.
\end{proof}

\subsection{Proof of \Cref{prop:xy_fails}}

\begin{proof}[Proof of \Cref{prop:xy_fails}] \
\paragraph{Proof for $\mc{I}^1$.}
Fix $\alpha  \in (0,0.1)$ and consider the following instance with $m = 1$:
\begin{align*}
& \cX = \{ e_1, e_2 \}, \quad \cZ = \{ z_1, z_2 \}, \quad z_1 = [1/4,1/2], \quad z_2 = [3/4,1/2+\alpha] \\
& \thetast = e_1, \quad \must = [0,1], \quad \gamma = 1/2 + \alpha/2.
\end{align*}
On this example, $z_1$ is safe and $z_2$ is unsafe with $\Delsafe(z_2) = -\alpha/2$. 

Let $A(\lambda) = \lambda_1 e_1 e_1^\top + \lambda_2 e_2 e_2^\top$ denote the design matrix. Then the allocation that minimizes \XYdiff:
\begin{align*}
\max_{z,z' \in \cZ} \| z - z' \|_{A(\lambda)^{-1}}^2 = \frac{1}{4 \lambda_1} + \frac{\alpha^2}{ \lambda_2} 
\end{align*}
is 
\begin{align*}
\lambda_1 = \frac{1}{1 + 2 \alpha}, \quad \lambda_2 = \frac{2\alpha}{1 + 2 \alpha}.
\end{align*}
Denote this allocation as $\lamtil$. 

Applying the Transportation Lemma of \cite{kaufmann2016complexity}, this implies that any $\delta$-PAC strategy must have
\begin{align*}
\Exp[T_1] \KL(\nu_{(\thetast,\must),1} ||\nu_{(\theta,\mu),1}) + \Exp[T_2] \KL(\nu_{(\thetast,\must),2} ||\nu_{(\theta,\mu),2}) \ge \log \frac{1}{2.4\delta}
\end{align*}
for all $(\theta,\mu) \in \Calt$, where $\Calt$ is defined as in \Cref{lem:lb_alg_set_proj}.
If a learner plays $\lamtil$ for $T$ steps, they will have $\Exp[T_1] = \lambda_1 \Exp[T], \Exp[T_2] = \lambda_2 T$. In this case, the above can be rewritten as
\begin{align*}
\Exp[T] & \ge \log \frac{1}{2.4\delta} \cdot \frac{1}{\lambda_1 \KL(\nu_{(\thetast,\must),1} ||\nu_{(\theta,\mu),1}) + \lambda_2 \KL(\nu_{(\thetast,\must),2}||\nu_{(\theta,\mu),2})} \\
& =  \log \frac{1}{2.4\delta} \cdot \frac{2}{\| \thetast - \theta \|_{A(\lamtil)}^2 + \| \must - \mu \|_{A(\lamtil)}^2}.
\end{align*}
where the equality follows by the same calculation as in the proof of \Cref{thm:lb}. Take $(\theta,\mu)$ to be $\theta = \thetast$, $\mu = [0,1 - \frac{\alpha}{1+2\alpha}]$ and note that $(\theta,\mu) \in \Calt$ since with this choice of $\mu$, arm $z_2$ is now safe. Now,
\begin{align*}
\| \must - \mu \|_{A(\lamtil)}^2 = (\frac{\alpha}{1+2\alpha})^2 \cdot \frac{2 \alpha}{1 + 2 \alpha} = \frac{2 \alpha^3}{(1+2\alpha)^3}.
\end{align*}
This gives a lower bound of
\begin{align*}
\Exp[T] \ge \log \frac{1}{2.4\delta} \cdot \frac{2 (1+2\alpha)^3 }{2 \alpha^3} \ge \log \frac{1}{2.4\delta} \cdot \frac{1}{\alpha^4}.
\end{align*}
This lower bound is for best-arm identification ($\epsilon = 0$), but setting $\alpha \leftarrow 2 \epsilon - a$ for $a$ arbitrarily small, identifying an $\epsilon$-optimal, $\epsilon$-safe arm is equivalent to identifying the best arm, so this therefore holds as a lower bound on $(\epsilon,\delta)$-PAC algorithms.

The upper bound on the performance of \safebai follows trivially since by setting $\lambda_1 = \lambda_2$, we can make the numerator in both terms of the complexity $\cO(1)$, and the denominator of each term will be at least $\epsilon^2$.

\paragraph{Proof for $\mc{I}^2$.}
Fix $\alpha  \in (0,0.1)$ and consider the following instance with $m = 1$:
\begin{align*}
& \cX = \{ e_1, e_2 \}, \quad \cZ = \{ z_1, z_2 \}, \quad z_1 = [1/2+\alpha^2/2,0], \quad z_2 = [1/2,\alpha/2] \\
& \thetast = [1/2,0], \quad \must = [0,0], \quad \gamma = 1.
\end{align*}
On this instance, both $z_1$ and $z_2$ are safe, and $z_1$ is optimal. 

The \XYsafe design minimizes:
\begin{align*}
\max_{z \in \cZ} \| z \|_{A(\lambda)^{-1}}^2 = \max \left \{ \frac{1+2\alpha+\alpha^2}{4\lambda_1}, \frac{1}{4\lambda_1} + \frac{\alpha^2}{4\lambda_2} \right \} .
\end{align*}
Some computation shows that, for $\alpha$ small, the optimal settings are $\lambda_1 = \cO(1)$ and $\lambda_2 = \cO(\alpha)$ (where here $\cO(\cdot)$ hides terms that are $o(\alpha)$).
Denote this allocation as $\lamtil$. Following the same argument as above, we have
\begin{align*}
\Exp[T] & \ge   \log \frac{1}{2.4\delta} \cdot \frac{2}{\| \thetast - \theta \|_{A(\lamtil)}^2 + \| \must - \mu \|_{A(\lamtil)}^2}
\end{align*}
for any $(\mu,\theta) \in \Calt$. Let $\theta = [1/2,2\alpha]$ and note that $(\must,\theta) \in \Calt$ since $z_2$ is now the optimal arm with this $\theta$. We then have
\begin{align*}
\| \thetast - \theta \|_{A(\lamtil)}^2 + \| \must - \mu \|_{A(\lamtil)}^2 = \cO(\alpha^3)
\end{align*}
which gives a lower bound of
\begin{align*}
\Exp[T] \ge \Omega \left ( \frac{1}{\alpha^3} \cdot \log \frac{1}{\delta} \right ).
\end{align*}
This lower bound holds for the best-arm identification problem, but setting $\alpha  \leftarrow \sqrt{2\epsilon} - a$ for $a$ arbitrarily small, finding an $\epsilon$-optimal arm is equivalent to finding the best arm, so the lower bound applies in that setting as well. 

To compute the sample complexity of \safebai, we note that $\Delsafe(z_1) = \Delsafe(z_2) = 1$, so the first term in the complexity is negligible. We also have that $\Delta^{\epstil}(z_2) = \alpha^2/2 = \cO(\epsilon)$ for $\epstil \le 1$. Thus, the second term in the complexity scales as
\begin{align*}
\cOtil \left ( \inf_{\lambda \in \simplex_\cX} \max_{z,z' \in \cZ} \frac{\| z - z' \|_{A(\lambda)^{-1}}^2 \cdot \log 1/\delta}{\epsilon^2} \right ) & = \cOtil \left ( \inf_{\lambda \in \simplex_\cX}  \frac{(\alpha^4/\lambda_1 + \alpha^2/\lambda_2) \cdot \log 1/\delta}{\epsilon^2} \right ) \\
& = \cOtil \left ( \frac{\alpha^2 \cdot \log 1/\delta}{\epsilon^2} \right ) \\
& = \cOtil \left ( \frac{\log 1/\delta}{\epsilon} \right ). 
\end{align*}
\end{proof}

\section{Robust Mean Estimation}\label{sec:rips}
In order to form estimates of $z^\top \thetast$ and $z^\top \musti$, we will rely on the RIPS procedure proposed in \cite{camilleri2021highdimensional}, instantiated with the robust Catoni estimator \cite{catoni2012challenging}. 

\paragraph{Catoni Estimation.}
The robust Catoni mean estimator proposed in \cite{catoni2012challenging} is defined as follows.

\begin{definition}[Catoni Estimator]
Consider real values $X_1,\ldots,X_T$. Then the \emph{robust Catoni mean estimator}, $\cat_\alpha[\{X_t\}_{t=1}^T]$, with parameter $\alpha > 0$ is the unique root $z$ of the function
\begin{align*}
f_{\cat}(z; \{ X_i \}_{i=1}^T, \alpha) := \sum_{t=1}^T \psi_{\cat}(\alpha(X_t - z)) \quad \text{for} \quad \psi_{\cat}(y) := \begin{cases} \log(1 + y + y^2) & y \ge 0 \\
\log(1-y+y^2) & y < 0 \end{cases}.
\end{align*}
The Catoni estimator satisfies the following guarantee.

\begin{proposition}\label{prop:catoni}
Let $X_1,\ldots,X_T$ be independent, identically distributed random variables with mean $\zeta$ and variance $\sigma^2 < \infty$. Fix $\delta \in (0,1)$ and assume $T \ge 4 \log (1/\delta)$. Then the Catoni estimator $\cat_{\alpha}[\{X_t\}_{t=1}^T]$ with parameter
\begin{align}\label{cat:alpha}
\alpha = \sqrt{\frac{2 \log 1/\delta}{T \sigma^2}}
\end{align}
satisfies, with probability at least $1-2\delta$,
\begin{align*}
| \cat_{\alpha}[\{X_t\}_{t=1}^T] - \zeta | \le \sqrt{\frac{8 \sigma^2 \log 1/\delta}{T }}.
\end{align*}
\end{proposition}
\end{definition}

Notably, the estimation error given by \Cref{prop:catoni} scales only with the variance of the random variables, and not with their magnitude.

\paragraph{Robust Inverse Propensity Score (RIPS) Estimator.}
We apply the Catoni estimator with the RIPS estimator of \cite{camilleri2021highdimensional}. In particular, consider running the following procedure.


\begin{algorithm}[h]
\caption{Robust Inverse Propensity Score Estimation (\rips)}
\begin{algorithmic}[1]
\State \textbf{input:} samples $\{ (x_t, r_t) \}_{t=1}^T$ for $x_t \sim \lambda$ and $r_t = \theta^\top x_t + w_t$, active set $\cY$, confidence $\delta$
\State For each $y \in \cY$, set $W^y \leftarrow \cat_{\alpha}[\{ y^\top A(\lambda)^{-1} x_t r_t \}_{t=1}^T]$,
for $\alpha$ chosen as in \eqref{cat:alpha} with $\delta \leftarrow \frac{\delta}{2 |\cY|}$, and $A(\lambda) = \sum_{x \in \cX} \lambda_x x x^\top $. 
\State Set 
\begin{align*}
\thetahat = \argmin_{\theta} \max_{y \in \cY} \frac{| \theta^\top y - W^y |}{\| y \|_{A(\lambda)^{-1}}}.
\end{align*}
\State \textbf{return} $\thetahat$
\end{algorithmic}
\label{alg:rips}
\end{algorithm}

We have the following guarantee on this procedure.


\begin{proposition}[Theorem 1 of \cite{camilleri2021highdimensional}]\label{prop:rips}
If $T \ge 4 \log \frac{2|\cZ|}{\delta}$, then with probability at least $1-\delta$, for all $z \in \cZ$, the \rips estimator of \Cref{alg:rips} returns an estimate $\thetahat$ which satisfies:
\begin{align*}
& | y^\top(\thetahat -  \thetast) | \le \| y \|_{A(\lambda)^{-1}}  \cdot   \sqrt{\frac{8\log(2 |\cZ|/\delta)}{T}} .
\end{align*}
\end{proposition}

The use of the RIPS estimator allows us to avoid sophisticated rounding     procedures often found in the linear bandit literature. Note that the RIPS estimator can be computed in time scaling polynomially in $|\cY|$, $d$, and $T$.


\newcommand{\fpevala}[1]{#1}

%
%

\section{$\ragest$}\label{sec:ragest}

\paragraph{A note on constants.}
Throughout our algorithm definitions, in both this section and the following, we use generic constants rather than precise numerical settings, and carry these generic constants through our proofs. At various points in the proofs, we require that these constants satisfy certain constraints. The following result shows that there exist suitable settings for all constants such that these constraints are satisfied.

\begin{lemma}\label{lem:constants}
There exist settings of $\ca,\cb,\cd,\ce,\cf,\ce,\cg,c_1,c_2,c_3,c_4,\csafe$ and $\cabs$ such that Equations \eqref{eq:constants_condition6}, \eqref{eq:constants_condition4}, \eqref{eq:constants_condition2}, \eqref{eq:constants_condition3}, \eqref{eq:constants_condition1},  \eqref{eq:constants_condition7}, \eqref{eq:constants_condition5}, \eqref{eq:ragest_constants_condition1}, \eqref{eq:ragest_constants_condition2}, and \eqref{eq:ragest_constants_condition3} are satisfied, and
\begin{align*}
\frac{c_3(1+\cg)}{1 - c_3} \le 0.2, \quad \cg \le 0.2, \quad \cabs \ge 0.0001.
\end{align*}
\end{lemma}
\begin{proof}
First, note that in addition to the conditions listed above, we must also have
\begin{align*}
c_1 \le \cf, \quad 3(\cd + \ce) \le c_2.
\end{align*}
Furthermore, by \Cref{lem:Yend_properties}, it suffices to always take $\csafe = 3 \cd + 3 \ce - \cg$.
Direct computation then shows that the following settings suffice, up to machine precision:

\begin{align*}
c_1 &= 0.05978841810030329 \\
c_2 &= 0.0600087370242953 \\
c_3 &= 0.1 \\
c_4 &= 0.1 \\
\ca &= 0.0013004532984432395 \\
\cb &= 0.41043329378840077 \\
\cd & = 0.01 \\
\ce & = 0.01 \\
\cc &= 0.0014065949472697806 \\
\cg &= 0.178 \\
\cf & = c_1 
\end{align*}

Given these settings, we can bound
\begin{align*}
\frac{c_3(1+\cg)}{1 - c_3} \le 0.5.
\end{align*}
\end{proof}


\subsection{Preliminaries} 

\paragraph{Assumptions and Definitions.}
For all $y \in \cY$, $\Delhatsafe(y) \ge -\csafe \epsilon$. We will also assume that $\cY \subseteq \cX$. We define
\begin{align*}
\yst = \argmin_{y \in \cY} y^\top \thetast
\end{align*}
and
\begin{align*}
\Delta(z) = \thetast^\top (z - \yst) . 
\end{align*}

We will take $\gamma = 0$, so we set $A(\lambda) =  \sum_{x \in \cX} \lambda_x x x^\top$.

\subsection{Algorithm and Main Results}

At a high-level, $\ragest$ attempts to estimate the difference between the performance of each $z \in \cZ$ and the \emph{best} $y \in \cY$. The safety gap estimate, $\Delhatsafe(z)$, acts as a regularizer: if $\Delhatsafe(z) < 0$, then we do not seek to estimate the gap of $z$ with as high accuracy, since we can already eliminate it by showing it is unsafe. The proof in this section follow closely the proof given in Section 6.4.4 of \cite{jamieson2022interactive}.

\begin{algorithm}[h]
\caption{$\ragest$}\begin{algorithmic}[1]
\State \textbf{input:} active set $\cZ$, optimal set $\cY$, tolerance $\epsilon$, confidence $\delta$, safety gap estimate $\{ \Delhatsafe(z) \}_{z \in \cZ}$

\State Choose $\yhat_{0}$ arbitrarily from $\cY$, set $\Delhat^{0}(z) \leftarrow 0$ for all $z \in \cZ$
\For{$\ell = 1,2,\ldots,\lceil \log(2/\cf \epsilon) \rceil$}
	\State $\epsilon_\ell \leftarrow \frac{2}{\cf} \cdot 2^{-\ell}$
	\State Let $\tau_\ell$ be the minimal value of $\tau = 2^j \ge 4 \log \frac{4 | \cZ |^2 \ell^2}{\delta}$ such that the objective to the following is no greater than $\cc \epsilon_\ell$, and $\lambda_\ell$ the corresponding optimal distribution \label{line:ragest_allocation}
	\begin{align*}
\inf_{\lambda \in \simplex_\cX} \max_{z \in \cZ} -\ca ( \pos( - \Delhatsafe(z)) + \pos(\Delhat^{\ell-1}(z)) + \epsilon_\ell) + \sqrt{ \frac{\| z - \yhat_{\ell-1} \|_{A(\lambda)^{-1}}^2 \cdot \log(\ubb)}{\tau}}   .
\end{align*}
	\State Sample $x_t \sim \lambda_\ell$, collect observations $\{ (x_t,r_t,s_{t,1},\ldots,s_{t,m}) \}_{t=1}^{\tau_\ell}$
	\State $\cW \leftarrow \{ z - z' \ : \ z, z' \in \cZ\}$
	\State $\thetahat^\ell \leftarrow \rips(\{ (x_t,r_t ) \}_{t=1}^{\tau_\ell},\cW,\frac{\delta}{2\ell^2})$
	\State Set 
	\begin{align*}
	\yhat_\ell & \leftarrow \argmin_{y \in \cY} y^\top \thetahat^\ell + 8\sqrt{\frac{ \| y - \yhat_{\ell-1} \|_{A(\lambda_\ell)^{-1}}^2 \cdot \log(\ubb)}{\tau_\ell}} \\
	\Delhat^\ell(y) & \leftarrow (y - \yhat_\ell)^\top \thetahat^\ell + \sqrt{\frac{\| y - \yhat_{\ell} \|_{A(\lambda_\ell)^{-1}}^2 \cdot \log(\ubb)}{\tau_\ell} }
	\end{align*}
\EndFor
\State \textbf{return} $\{ \Delhat^{\ell}(z) \}_{z \in \cZ}$
\end{algorithmic}
\label{alg:ragest}
\end{algorithm}

\begin{thm}\label{thm:ragest}
With probability at least $1-\delta$, $\ragest$ will terminate after collecting at most
\begin{align*}
C \cdot \sum_{\ell=1}^{\lceil \log 2/\cf \epsilon \rceil} \inf_{\lambda \in \simplex_\cX} \max_{z \in \cZ} \frac{ \| z - \yst\|_{A(\lambda)^{-1}}^2}{(\pos(-\Delhatsafe(z)) + \pos( \Delta(z)) + \epsilon_\ell )^2} \cdot \log (\ubb) + 4 \lceil \log \tfrac{2}{\cf \epsilon} \rceil \log (\tfrac{4|\cZ|^2 \lceil \log \tfrac{2}{\cf \epsilon} \rceil}{\delta})
\end{align*}
samples, for a universal constant $C$, and will output estimates of the gaps $\Delhat(z)$ such that, for all $z \in \cZ$,  
\begin{align*}
| \Delhat(z) - \Delta(z)| \le \cf \left ( \epsilon + \pos(\Delta(z)) + \pos( - \Delhatsafe(z)) \right ). 
\end{align*}
\end{thm}

\newcommand{\alphast}{\alpha^\star}
\newcommand{\alphamin}{\alpha_{\min}}
\newcommand{\alphamax}{\alpha_{\max}}

\subsection{Estimating the Gaps}

\begin{lemma}\label{lem:cEragest_high_prob}
Let $\cEragest$ denote the event that for all $\ell$ and all $z,z' \in \cX$, we have: 
\begin{align*}
|(\thetahat^\ell - \thetast)^\top (z - z') | \le & \sqrt{8 \| z - z' \|_{A(\lambda_\ell)^{-1}}^2 \cdot \frac{\log \ubb}{\tau_\ell}} . 
\end{align*}
Then $\Pr[\cEragest] \ge 1-\delta$. 
\end{lemma}
\begin{proof}
Since $\tau_\ell \ge 4 \log \ubb$, we can apply \Cref{prop:rips} to get that, with probability at least $1-\delta/2\ell^2$, for all $w \in \cW$,
\begin{align*}
| (\thetahat^\ell - \thetast)^\top w | \le  \sqrt{8 \| w \|_{A(\lambda_\ell)^{-1}}^2 \cdot \frac{\log \ubb}{\tau_\ell}}.
\end{align*}
The result then follows by a union bound since
\begin{align*}
\sum_{\ell=1}^{\infty} \frac{\delta}{2 \ell^2} = \frac{\pi^2}{12} \delta \le \delta. 
\end{align*}
\end{proof}

\begin{lemma}\label{lem:ragest_gap_bound}
On $\cEragest$, for all $z \in \cZ$ and all $\ell$, 
\begin{align*}
| \Delhat^\ell(z) - \thetast^\top(z - \yhat_\ell)| \le 8 \ca \left (  \pos ( \Delhat^{\ell-1}(z)) + \pos( -\Delhatsafe(z)) + \pos( \Delhat^{\ell-1}(\yhat_\ell)) \right ) + 8(\cc + \ca +  2\ca \csafe) \epsilon_\ell.
\end{align*}
\end{lemma}
\begin{proof}
By construction, we have that
\begin{align*}
\max_{z \in \cZ} -\ca ( \pos( - \Delhatsafe(z)) + \pos(\Delhat^{\ell-1}(z)) + \epsilon_\ell) + \sqrt{ \frac{\| z - \yhat_{\ell-1} \|_{A(\lambda_\ell)^{-1}}^2 \cdot \log(\ubb)}{\tau_\ell}}   \le \cc \epsilon_\ell .
\end{align*}
This implies that, for all $z \in \cZ$:
\begin{align*}
\sqrt{ \frac{\| z - \yhat_{\ell-1} \|_{A(\lambda_\ell)^{-1}}^2 \cdot \log(\ubb)}{\tau_\ell}}  \le \ca ( \pos( - \Delhatsafe(z)) + \pos(\Delhat^{\ell-1}(z)) ) + (\cc + \ca) \epsilon_\ell.
\end{align*}
On $\cEragest$, we have
\begin{align*}
| \Delhat_\ell(z)  - \thetast^\top (z - \yhat_{\ell})| & \le \sqrt{8 \| z - \yhat_\ell \|_{A(\lambda_\ell)^{-1}}^2 \cdot \frac{ \log(\ubb)}{\tau_\ell}} \\
& \le  \sqrt{ 16\| \yhat_{\ell-1} - \yhat_\ell \|_{A(\lambda_\ell)^{-1}}^2 \cdot \frac{ \log(\ubb)}{\tau_\ell}} +  \sqrt{ 16\| \yhat_{\ell-1} - \yhat_\ell \|_{A(\lambda_\ell)^{-1}}^2 \cdot \frac{ \log(\ubb)}{\tau_\ell}} \\
& \le 8 \ca \left (\pos( - \Delhatsafe(z)) + \pos( \Delhat^{\ell-1}(z))  + \pos(- \Delhatsafe(\yhat_\ell)) +   \pos( \Delhat^{\ell-1}(\yhat_\ell)) \right ) + 8( \cc + \ca) \epsilon_\ell .
\end{align*}
By construction we have that $\Delhatsafe(\yhat_\ell) \ge -\csafe \epsilon \ge -2\csafe \epsilon_\ell$, so $\pos( - \Delhatsafe(\yhat_\ell)) \le 2\csafe \epsilon_\ell$, which proves the result. 
\end{proof}

\begin{lemma}\label{lem:yhat_gap}
On $\cEragest$ and the event that $\Delhat^{\ell-1}(\yst) \le \cb \epsilon_\ell$, we have 
$$\Delta(\yhat_\ell) \le 6 (\cc + \ca(1+\cb + 2 \csafe))\epsilon_\ell.$$
\end{lemma}
\begin{proof}
By the definition of $\cEragest$ and $\yhat_\ell$, we can bound
\begin{align*}
\thetast^\top(\yhat_{\ell} - \yhat_{\ell-1}) & \le (\thetahat^\ell)^\top (\yhat_{\ell} - \yhat_{\ell-1})  + \sqrt{8\| \yhat_\ell - \yhat_{\ell-1} \|_{A(\lambda_\ell)^{-1}}^2 \cdot \frac{ \log(\ubb)}{\tau_\ell}} \\
& = \min_{y \in \cY} (\thetahat^\ell)^\top (y- \yhat_{\ell-1})  + \sqrt{8\| y - \yhat_{\ell-1} \|_{A(\lambda_\ell)^{-1}}^2 \cdot \frac{ \log(\ubb)}{\tau_\ell}} \\
& \le (\thetahat^\ell)^\top (\yst - \yhat_{\ell-1})  + \sqrt{8\| \yst - \yhat_{\ell-1} \|_{A(\lambda_\ell)^{-1}}^2 \cdot \frac{ \log(\ubb)}{\tau_\ell}} \\
& \le \thetast^\top (\yst - \yhat_{\ell-1})  +  2\sqrt{8\| \yst - \yhat_{\ell-1} \|_{A(\lambda_\ell)^{-1}}^2 \cdot \frac{ \log(\ubb)}{\tau_\ell}} .
\end{align*}
By the definition of $\tau_\ell$ and $\lambda_\ell$, we have
\begin{align*}
\cc \epsilon_\ell & \ge  \max_{z \in \cZ}  -\ca ( \pos( - \Delhatsafe(z)) + \pos( \Delhat^{\ell-1}(z)) + \epsilon_\ell) + \sqrt{\frac{ \| z - \yhat_{\ell-1} \|_{A(\lambda_\ell)^{-1}}^2 \cdot \log(\ubb)}{ \tau_\ell}}   \\
& \ge  -\ca ( \pos( - \Delhatsafe(\yst)) + \pos( \Delhat^{\ell-1}(\yst)) + \epsilon_\ell) + \sqrt{\frac{ \| \yst - \yhat_{\ell-1} \|_{A(\lambda_\ell)^{-1}}^2 \cdot \log(\ubb)}{ \tau_\ell}}   \\
& \overset{(a)}{\ge}  - \ca ( \pos( \Delhat^{\ell-1}(\yst)) + (1+2\csafe)\epsilon_\ell) + \sqrt{\frac{ \| \yst - \yhat_{\ell-1} \|_{A(\lambda_\ell)^{-1}}^2 \cdot \log(\ubb)}{ \tau_\ell}} \\
& \overset{(b)}{\ge} -\ca(1+\cb + 2 \csafe)\epsilon_\ell + \sqrt{\| \yst - \yhat_{\ell-1} \|_{A(\lambda_\ell)^{-1}}^2 \cdot \frac{\log(\ubb)}{ \tau_\ell}}
\end{align*}
where $(a)$ uses that $\Delhatsafe(\yst) \ge -\csafe \epsilon \ge -2\csafe \epsilon_\ell $, by definition, and $(b)$ follows by our assumption on $\Delhat^{\ell-1}(\yst)$. This implies that
\begin{align*}
\sqrt{\| \yst - \yhat_{\ell-1} \|_{A(\lambda_\ell)^{-1}}^2 \cdot \frac{\log(\ubb)}{ \tau_\ell}} \le (\cc + \ca(1+\cb + 2 \csafe))\epsilon_\ell.
\end{align*}
Combining this with the above we have that
\begin{align*}
\thetast^\top(\yhat_\ell - \yhat_{\ell-1}) \le \thetast^\top(\yst - \yhat_{\ell-1}) + 6 (\cc + \ca(1+\cb + 2 \csafe))\epsilon_\ell .
\end{align*}
Rearranging this proves the result. 
\end{proof}

\begin{lemma}\label{lem:ragest_gap_est}
For all $z \in \cZ$ and all $\ell$, on the event $\cEragest$,
\begin{align*}
| \Delhat^\ell(z) - \Delta(z)| \le \cf \left ( \epsilon_\ell + \pos(\Delta(z)) + \pos( - \Delhatsafe(z)) \right ). 
\end{align*}
\end{lemma}
\begin{proof}
We prove this by induction. Assume that at $\ell-1$, for all $z \in \cZ$,
\begin{align*}
| \Delhat^{\ell-1}(z) - \Delta(z)| \le \cf \left ( \epsilon_{\ell-1} + \pos( \Delta(z)) + \pos( - \Delhatsafe(z)) \right ). 
\end{align*}
On $\cEragest$ and by \Cref{lem:ragest_gap_bound} we can bound
\begin{align*}
| \Delhat^\ell(z) - \Delta(z)| & = |\Delhat^\ell(z) - (R(z) - R(\yhat_\ell) + R(\yhat_\ell) - R(\yst))|\\
& \le | \Delhat^\ell(z) - (R(z) - R(\yhat_{\ell}))| + \Delta(\yhat_\ell) \\
& \le 8 \ca \left (  \pos( \Delhat^{\ell-1}(z)) + \pos( -\Delhatsafe(z)) + \pos(\Delhat^{\ell-1}(\yhat_\ell)) \right ) + 8(\cc + \ca + 2\ca \csafe) \epsilon_\ell + \Delta(\yhat_\ell) .
\end{align*}
By the inductive hypothesis, we can bound
\begin{align*}
& \pos( \Delhat^{\ell-1}(z)) \le (1+\cf) \pos( \Delta(z)) + \cf \pos(- \Delhatsafe(z)) + \cf \epsilon_{\ell-1} \\
& \pos(\Delhat^{\ell-1}(\yhat_\ell)) \le (1+\cf) \pos( \Delta(\yhat_\ell)) + \cf \pos( - \Delhatsafe(\yhat_\ell)) + \cf \epsilon_{\ell-1}. 
\end{align*}
By construction $ \pos( - \Delhatsafe(\yhat_\ell)) \ge -\csafe \epsilon \ge - 2\csafe \epsilon_\ell$, so
\begin{align*}
| \Delhat^{\ell}(z) - \Delta(z)|  \le 8 \ca (1 + \cf) \pos(& \Delta(z))  + 8 \ca ( 1 + \cf) \pos(\Delhatsafe(z)) + (8 \ca (1 + \cf) + 1) \Delta(\yhat_\ell) \\
& + 8( \ca \cf (1 + \csafe) + \cc + \ca + 2\ca \csafe) \epsilon_\ell.
\end{align*}
It remains to bound $\Delta(\yhat_\ell) = R(\yhat_\ell) - R(\yst)$. On the inductive hypothesis, we have that
\begin{align*}
|\Delhat^{\ell-1}(\yst) - \Delta(\yst)| \le \cf \left ( \epsilon_{\ell-1} + \pos( \Delta(\yst)) + \pos( -\Delhatsafe(\yst)) \right ).
\end{align*}
By definition, $\Delta(\yst) = 0$ and $\Delhatsafe(\yst) \ge -\csafe \epsilon \ge - 2 \csafe \epsilon_\ell$, which implies that $\Delhat^{\ell-1}(\yst) \le 2\cf(1 + \csafe) \epsilon_\ell$. It follows that the conditions of \Cref{lem:yhat_gap} are met as long as
\begin{align}\label{eq:ragest_constants_condition1}
2 \cf (1 + \csafe) \le \cb,
\end{align}
so we can bound $\Delta(\yhat_\ell) \le  6 (\cc + \ca(1+\cb + 2 \csafe))\epsilon_\ell$. Thus,
\begin{align*}
| \Delhat^{\ell}(z) - \Delta(z)|  \le 8 \ca (1 + \cf) \pos(& \Delta(z))  + 8 \ca ( 1 + \cf) \pos(\Delhatsafe(z)) + 8( \ca \cf (1 + \csafe) + \cc + \ca + 2\ca \csafe) \epsilon_\ell \\
& +   (8 \ca (1 + \cf) + 1) (6 (\cc + \ca(1+\cb + 2 \csafe))) \epsilon_\ell.
\end{align*}
which proves the inductive hypothesis as long as
\begin{align}\label{eq:ragest_constants_condition2}
& 8( \ca \cf (1 + \csafe) + \cc + \ca + 2\ca \csafe) +  (8 \ca (1 + \cf) + 1) (6 (\cc + \ca(1+\cb + 2 \csafe)))  \le \cf \nonumber \\
& 8 \ca (1 + \cf) \le \cf 
\end{align}

For the base case, we need to show that 
\begin{align*}
| \Delhat^{0}(z) - \Delta(z)| \le \cf \left ( \epsilon_{0} + \pos( \Delta(z)) + \pos( - \Delhatsafe(z)) \right ). 
\end{align*}
By construction $\Delhat^{0}(z) = 0$ for all $z$, and $\pos( \Delta(z)) \ge 0$, $\pos( - \Delhatsafe(z)) \ge 0$. Thus, it suffices to show $|\Delta(z)| \le \cf \epsilon_0$. However, by construction $|\Delta(z)| \le 1$, and $\cf \epsilon_0 = 1$, which proves the base case. 
\end{proof}

\subsection{Bounding the Sample Complexity}

\begin{lemma}\label{lem:ragest_complexity}
On the event $\cEragest$, $\ragest$ will terminate after collecting at most
\begin{align*}
C \cdot \sum_{\ell=1}^{\lceil \log (2/\cf\epsilon) \rceil} \inf_{\lambda \in \simplex_\cX} \max_{z \in \cZ} \frac{ \| z - \yst\|_{A(\lambda)^{-1}}^2}{(\pos(-\Delhatsafe(z)) + \pos( \Delta(z)) + \epsilon_\ell )^2} \cdot \log (\ubb) + 8 \lceil \log \tfrac{2}{\cf\epsilon} \rceil \log (\tfrac{4 |\cZ|^2 \lceil \log \tfrac{2}{\cf\epsilon} \rceil^2}{\delta})
\end{align*}
samples, for a universal constant $C$.
\end{lemma}
\begin{proof}
If, for all $z \in \cZ$,
\begin{align*}
\tau \ge \frac{\| z - \yhat_{\ell-1} \|_{A(\lambda)^{-1}}^2}{( \ca ( \pos( - \Delhatsafe(z)) + \pos(\Delhat^{\ell-1}(z)) + \epsilon_\ell) + \cc \epsilon_\ell)^2} \cdot \log(\ubb)
\end{align*}
we will have that the objective on \Cref{line:ragest_allocation} of $\ragest$ is less than $\cc \epsilon_\ell$. Since we can take the best-case $\lambda \in \simplex_\cX$, and since we have that $\tau_\ell$ will be at most a factor of 2 from the optimal $\tau$, it follows that
\begin{align*}
\tau_\ell & \le \inf_{\lambda \in \simplex_\cX} \max_{z \in \cZ} \frac{2\| z - \yhat_{\ell-1} \|_{A(\lambda)^{-1}}^2}{( \ca ( \pos( - \Delhatsafe(z)) + \pos(\Delhat^{\ell-1}(z)) + \epsilon_\ell) + \cc \epsilon_\ell)^2} \cdot \log(\ubb) \vee 8 \log \ubb \\
& \le  \inf_{\lambda \in \simplex_\cX} \max_{z \in \cZ} \frac{2\| z - \yhat_{\ell-1} \|_{A(\lambda)^{-1}}^2}{( \ca ( \pos( - \Delhatsafe(z)) + \pos(\Delhat^{\ell-1}(z)) + \epsilon_\ell) + \cc \epsilon_\ell)^2} \cdot \log(\ubb) + 8 \log \ubb
\end{align*}
where the additional $8 \log \ubb$ factor arises since we always require $\tau_\ell \ge 4 \log \ubb$.

We can upper bound
\begin{align*}
 \| z - \yhat_{\ell-1} \|_{A(\lambda)^{-1}}^2 \le 2  \| z - \yst \|_{A(\lambda)^{-1}}^2 + 2  \| \yst - \yhat_{\ell-1} \|_{A(\lambda)^{-1}}^2 .
\end{align*}
By construction, $\pos(- \Delhatsafe(\yhat_{\ell-1})) \le 2 \csafe \epsilon_\ell$, so for any $z$, $\pos(- \Delhatsafe(\yhat_{\ell-1})) - 2 \csafe \epsilon_\ell \le \pos(-\Delhatsafe(z))$. Furthermore, by definition,
\begin{align*}
\Delhat^{\ell-1}(\yhat_{\ell-1}) & = 0
\end{align*}
so $\pos(\Delhat^{\ell-1}(z)) \ge \pos(\Delhat^{\ell-1}(\yhat_{\ell-1})) $. Thus, 
\begin{align*}
& \inf_{\lambda \in \simplex_\cX} \max_{z \in \cZ} \frac{\| z - \yhat_{\ell-1} \|_{A(\lambda)^{-1}}^2}{(\ca \pos(-\Delhatsafe(z)) + \ca \pos( \Delhat^{\ell-1}(z)) + (\ca + \cc ) \epsilon_\ell )^2} \\
& \le \inf_{\lambda \in \simplex_\cX} \max_{z \in \cZ} \frac{2 \| z - \yst\|_{A(\lambda)^{-1}}^2}{(\ca \pos(-\Delhatsafe(z)) + \ca \pos( \Delhat^{\ell-1}(z)) + (\ca + \cc ) \epsilon_\ell )^2} \\
& \qquad + \frac{2 \| \yhat_{\ell-1} - \yst\|_{A(\lambda)^{-1}}^2}{(\ca \pos(-\Delhatsafe(\yhat_{\ell-1})) + \ca \pos( \Delhat^{\ell-1}(\yhat_{\ell-1})) + (\ca + \cc - 2 \ca \csafe ) \epsilon_\ell )^2} \\
& \le \inf_{\lambda \in \simplex_\cX} \max_{z \in \cZ} \frac{4 \| z - \yst\|_{A(\lambda)^{-1}}^2}{(\ca \pos(-\Delhatsafe(z)) + \ca \pos( \Delhat^{\ell-1}(z)) + (\ca + \cc - 2 \ca \csafe )  \epsilon_\ell )^2}.
\end{align*} 
By \Cref{lem:ragest_gap_est}, we can lower bound
\begin{align*}
\Delhat^{\ell-1}(z) & \ge \Delta(z) - \cf ( \epsilon_\ell + \pos( \Delta(z))+ \pos(-\Delhatsafe(z))) 
\end{align*}
so 
\begin{align*}
& \ca \pos(-\Delhatsafe(z)) + \ca \pos( \Delhat^{\ell-1}(z)) + (\ca + \cc - 2 \ca \csafe )  \epsilon_\ell \\
& \qquad \ge  \ca ( 1 - \cf) \pos(- \Delhatsafe(z)) + \ca ( 1 - \cf) \pos(\Delta(z)) + (\ca + \cc - 2 \ca \csafe - \ca \cf )  \epsilon_\ell .
\end{align*}
The result follows by combining these inequalities and as long as 
\begin{align}\label{eq:ragest_constants_condition3}
\ca (1 - \cf) \ge \cabs, \qquad \ca + \cc - 2 \ca \csafe  - \ca \cf  \ge \cabs.
\end{align}
\end{proof}

\begin{proof}[Proof of \Cref{thm:ragest}]
\Cref{thm:ragest} follows directly from \Cref{lem:ragest_complexity} and \Cref{lem:ragest_gap_est} since, by \Cref{lem:cEragest_high_prob}, $\cEragest$ holds with probability at least $1-\delta$. 
\end{proof}


\section{Safe Best-Arm Identification}\label{sec:safe_bai_proofs}

\begin{algorithm}[h]
\caption{\textbf{Be}st \textbf{S}afe Arm \textbf{Ide}ntification (\safebai, defined with generic constants)}
\begin{algorithmic}[1]
\State \textbf{input:} tolerance $\epsilon$, confidence $\delta$
\State $\iotaeps \leftarrow \lceil \log(\frac{2}{\min \{ c_3, c_4 \} \cdot \epsilon}) \rceil$, $\Delhatsafe^0(z) \leftarrow 0, \Delhat^0(z) \leftarrow 0$ for all $z \in \cZ$
\For{$\ell = 1,2,\ldots,\iotaeps$}
	\State $\epsilon_\ell \leftarrow  \frac{2}{\min \{ c_3, c_4 \}} \cdot 2^{-\ell}$
	\Statex {\color{blue} \texttt{// Solve experiment to reduce uncertainty on safety constraints}}
	\State Let $\tau_\ell$ be the minimal value of $\tau = 2^{j} \ge 4 \log \ub$ such that the objective to the following is no greater than $\ce \epsilon_\ell$, and $\lambda_\ell$ the corresponding optimal distribution 
	\begin{align*}
	\inf_{\lambda \in \simplex_\cX} \max_{z \in \cZ}  - \cd \left ( \min_j |\Delhatsafe^{j,\ell-1}(z)| + \max_{j} \pos(-\Delhatsafe^{j,\ell-1}(z)) + \pos(\Delhat^{\ell-1}(z)) + \epsilon_\ell \right ) + \sqrt{\frac{ \| z \|_{A(\lambda)^{-1}}^2 \cdot \log(\ub)}{ \tau}}
	\end{align*}
	\State Sample $x_t \sim \lambda_\ell$, collect $\tau_\ell$ observations $\{ (x_t,r_t,s_{t,1},\ldots,s_{t,m}) \}_{t=1}^{\tau_\ell}$ 
	\State $\{ \muhat^{i,\ell} \}_{i=1}^m \leftarrow \rips(\{ (x_t, s_{t,i} ) \}_{t=1}^{\tau_\ell},\cZ,\frac{\delta}{2m\ell^2})$ \hfill {\color{blue} \texttt{// Estimate safety constraints}}
	\State $\Delhatsafe^{i,\ell}(z) \leftarrow  \gamma - z^\top \muhat^{i,\ell} + \| z \|_{A(\lambda_\ell)^{-1}}\sqrt{\tau_\ell^{-1}   \log(\ub)}$ {\color{blue} \texttt{// Safety gap estimates}} 
	\Statex {\color{blue} \texttt{// Form set of arms guaranteed to be safe}}
	\State \begin{align*}
	\cY_\ell \leftarrow \bigg \{ z \in \cZ \ : \  & 8 \cd \Big ( \min_j |\Delhatsafe^{j,\ell-1}(z)| + \max_{j} \pos(-\Delhatsafe^{j,\ell-1}(z)) + \pos(\Delhat^{\ell-1}(z)) \Big ) \\
	& + 8(\cd + \ce)\epsilon_\ell \le \Delhatsafe^{i,\ell}(z), \forall i \in [n] \bigg \} \cup \cY_{\ell-1}
	\end{align*}
	\Statex {\color{blue} \texttt{// Refine estimates of optimality gaps}}
	\State $\{ \Delhat^{\ell}(z) \}_{z \in \cZ} \leftarrow \ragest \Big (\cZ,\cY_\ell,\epsilon_\ell,\tfrac{\delta}{4 \ell^2}, \{ \Delhatsafe(z) \leftarrow \max_j \pos(-\Delhatsafe^{j,\ell}(z)) \}_{z \in \cZ} \Big )$
\EndFor
\Statex {\color{blue} \texttt{// Form set of arms guaranteed to be at most $\epsilon$-unsafe}}
\State \begin{align*}
\cYend \leftarrow \bigg \{ z \in \cZ \ : \  & 8 \cd \Big ( \min_j |\Delhatsafe^{j,\iotaeps}(z)| + \max_{j} \pos(-\Delhatsafe^{j,\iotaeps}(z)) + \pos(\Delhat^{\iotaeps}(z)) \Big ) \\
	& + 8(\cd + \ce)\epsilon- \cg \epsilon \le \Delhatsafe^{i,\iotaeps}(z), \forall i \in [n] \bigg \} 
\end{align*}
\Statex {\color{blue}  \texttt{// Find $\epsilon$-good arm out of $\epsilon$-safe arms}}
\State $\{ \Delhat^\eend(z) \}_{z \in \cYend} \leftarrow \ragest(\cYend,\cYend,\epsilon,\delta)$ 
\State \textbf{return} $\zhat = \argmin_{z \in \cYend} \Delhat^{\eend}(z)$
\end{algorithmic}
\label{alg:safe_bai_comp_const}
\end{algorithm}

\subsection{Preliminaries}
In general we want to consider multiple safety constraints, and let $m$ denote the number of constraints. In such settings, we will denote $\Delsafe^i(z)$ the safety gap for safety constraint $i$. 

Define
\begin{align*}
\Deltil^{\ell}(z) := \thetast^\top z -  \min_{y \in \cY_\ell} \thetast^\top y.
\end{align*}

\subsection{Algorithm and Main Result}

\begin{theorem}[Full version of \Cref{thm:main_complexity_nice}]\label{thm:main_complexity}
With probability at least $1-2\delta$, \Cref{alg:safe_bai_comp} returns an arm $\zhat$ such that 
\begin{align}\label{eq:main_thm_correctness}
\zhat^\top \thetast \ge (\zst)^\top \thetast -  \epsilon, \quad \Delsafe(\zhat) \ge -  \epsilon
\end{align}
and terminates after collecting at most
\begin{align*}
& C \cdot \sum_{\ell=1}^{\iotaeps}\inf_{\lambda \in \simplex_\cX} \max_{z \in \cZ} \frac{ \| z \|_{A(\lambda)^{-1}}^2 \cdot \log(\ub)}{\left ( \min_j |\Delsafe^j(z)| + \max_j \pos(-\Delsafe^j(z))  + \pos(\Delta^{\epsilon_{\ell-1}}(z)) + \epsilon_\ell \right )^2} \\
& \qquad + C \log \frac{1}{\epsilon} \cdot \sum_{\ell=1}^{\iotaeps}  \inf_{\lambda \in \simplex_\cX} \max_{z \in \cZ} \frac{\| z - \zst \|_{A(\lambda)^{-1}}^2 \cdot \log(\tfrac{8 |\cZ|^2 \log^4(1/\epsilon)}{\delta})}{( \max_j \pos( - \Delsafe^{j}(z)) + \pos ( \Delta^{\epsilon_\ell}(z)) + \epsilon_\ell)^2} + C_0
\end{align*}
samples for a universal constant $C$, $C_0 = \poly\log(\frac{1}{\epsilon}, | \cZ|) \cdot \log \frac{1}{\delta}$.
\end{theorem}

\subsection{Estimating the Safety Value}

\begin{lemma}\label{lem:cEsafe_high_prob}
Let $\cEsafe$ denote the event that, for all $\ell$, $z \in \cZ$, $i \in [m]$: 
\begin{align*}
| z^\top (\muhat^{i,\ell} - \must^i)|  \le \sqrt{8 \| z \|_{A(\lambda_\ell)^{-1}}^2 \cdot \frac{ \log(\ub)}{\tau_\ell}} .
\end{align*}
Then $\Pr[\cEsafe] \ge 1-\delta$. 
\end{lemma}
\begin{proof}
This follows directly from \Cref{prop:rips} and a union bound, as in \Cref{lem:cEragest_high_prob}. 
\end{proof}

\begin{lemma}\label{lem:safe_event_bound}
On $\cEsafe$, for all $z \in \cZ$, $i \in [m]$, and all $\ell$,
\begin{align*}
| \Delhatsafe^{i,\ell}(z) - \Delsafe^i(z)| \le 3 \cd \left ( \min_j |\Delhatsafe^{j,\ell-1}(z)| + \max_{j} \pos(-\Delhatsafe^{j,\ell-1}(z)) + \pos(\Delhat^{\ell-1}(z)) \right ) + 3(\cd + \ce)\epsilon_\ell .
\end{align*}
\end{lemma}
\begin{proof}
By construction, we have that
\begin{align*}
 \max_{z \in \cZ}  - \cd \left ( \min_j |\Delhatsafe^{j,\ell-1}(z)| + \max_{j} \pos(-\Delhatsafe^{j,\ell-1}(z)) + \pos(\Delhat^{\ell-1}(z)) + \epsilon_\ell \right ) + \sqrt{\frac{ \| z \|_{A(\lambda_\ell)^{-1}}^2 \cdot \log(\ub)}{ \tau_\ell}} \le \ce \epsilon_\ell.
	\end{align*}
This implies that, for all $z \in \cZ$,
\begin{align*}
\sqrt{\frac{ \| z \|_{A(\lambda_\ell)^{-1}}^2 \cdot \log(\ub)}{ \tau_\ell}}  \le \min_j \cd |\Delhatsafe^{j,\ell-1}(z)| + \max_{j} \cd \pos(-\Delhatsafe^{j,\ell-1}(z)) + \cd \pos(\Delhat^{\ell-1}(z)) +(\cd + \ce) \epsilon_\ell .
\end{align*}
On $\cEsafe$, we have
\begin{align*}
| \Delhatsafe^{i,\ell}(z)  - \Delsafe^i(z) | & \le \sqrt{8\frac{ \| z \|_{A(\lambda_\ell)^{-1}}^2 \cdot  \log(\ub)}{ \tau_\ell}} \\
& \le  \min_j 3\cd |\Delhatsafe^{j,\ell-1}(z)| + \max_{j} 3\cd \pos(-\Delhatsafe^{j,\ell-1}(z)) + 3 \cd \pos(\Delhat^{\ell-1}(z)) + 3(\cd + \ce)\epsilon_\ell
\end{align*}
which proves the result.
\end{proof}

\subsection{Tying Together Safety Estimation with Optimality Estimation}

\begin{defn}[Optimality Good Event]
Let $\cEragest^\ell$ denote the success event of $\ragest$ when called at the $\ell$th epoch, and $\cEragest := \cup_\ell \cEragest^\ell$. 
\end{defn}

\begin{lemma}\label{lem:safe_bai_gaps1}
On the event $\cEsafe \cap \cEragest$, we have that:
\begin{enumerate}
\item For all $\ell \le \iotaeps$, $y \in \cY_\ell$, and $i \in [m]$, $y^\top \musti \le \gamma$ .
\item For all $\ell$ and $z \in \cZ$, $\Deltil^{\ell-1}(z) \le \Deltil^\ell(z)$.
\end{enumerate}
\end{lemma}
\begin{proof}
By \Cref{lem:safe_event_bound}, we have that
\begin{align*}
\Delhatsafe^{i,\ell}(z) - 3 \cd \left ( \min_j |\Delhatsafe^{j,\ell-1}(z)| + \max_{j} \pos(-\Delhatsafe^{j,\ell-1}(z)) + \pos(\Delhat^{\ell-1}(z)) \right ) - 3(\cd + \ce)\epsilon_\ell \le \Delsafe^i(z).
\end{align*}
Thus, if the inclusion condition of $\cY_\ell$ is met, it must be the case that $\Delsafe^i(z) \ge 0$ for all $i$.

The second conclusion follows directly since $\cY_{\ell-1} \subseteq \cY_\ell$. 
\end{proof}

\begin{lemma}[Key Estimation Error Bound]\label{lem:safe_bai_gaps2}
On the event $\cEsafe \cap \cEragest$, for all $z \in \cZ$, $\ell$, and $i$, we have
\begin{align*}
|\Delhat^\ell(z) - \Deltil^\ell(z)| & \le c_3 \left (\epsilon_\ell +  \pos(\Deltil^{\ell}(z)) + \max_j \pos(-\Delsafe^j(z)) \right ) \\
| \Delhatsafe^{i,\ell}(z) - \Delsafe^i(z) | & \le c_4 \left (\epsilon_\ell +  \pos(\Deltil^{\ell}(z)) +  \min_j | \Delsafe^j(z) | + \max_j \pos(-\Delsafe^j(z)) \right ).
\end{align*}
\end{lemma}
\begin{proof}
We prove this by induction. Assume that the above inequalities hold at epoch $\ell-1$. On $\cEsafe \cap \cEragest$, by \Cref{lem:ragest_gap_est} and \Cref{lem:safe_event_bound}, we have
\begin{align*}
|\Delhat^\ell(z) - \Deltil^\ell(z)| & \le c_1(\epsilon_\ell +  \pos(\Deltil^{\ell}(z)) + \max_j \pos(-\Delhatsafe^{j,\ell-1}(z))) \\
| \Delhatsafe^{i,\ell}(z) - \Delsafe^i(z) | & \le c_2(\epsilon_\ell +  \pos(\Delhat^{\ell-1}(z)) +  \min_j | \Delhatsafe^{j,\ell-1}(z) | + \max_j \pos(-\Delhatsafe^{j,\ell-1}(z))) .
\end{align*}
By the inductive hypothesis, we can bound
\begin{align*}
\pos(\Delhat^{\ell-1}(z)) & \le \pos \left ( \Deltil^{\ell-1}(z) + c_3(\epsilon_{\ell-1} +  \pos(\Deltil^{\ell-1}(z)) +  \max_j \pos(-\Delsafe^j(z))) \right ) \\
& \le (1+c_3) \pos(\Deltil^{\ell-1}(z)) + c_3 \epsilon_{\ell-1} + \max_j c_3 \pos(-\Delsafe^j(z)) \\
& \le (1+c_3) \pos(\Deltil^{\ell}(z)) + 2 c_3 \epsilon_{\ell} + \max_j c_3 \pos(-\Delsafe^j(z))
\end{align*}
where the last inequality follows since, by \Cref{lem:safe_bai_gaps1}, $\Deltil^{\ell-1}(z) \le \Deltil^\ell(z)$.

Furthermore, again applying the inductive hypothesis,
\begin{align*}
\Delhatsafe^{i,\ell-1}(z) & \le \Delsafe^i(z) + c_4(\epsilon_{\ell-1} + \pos(\Deltil^{\ell-1}(z)) +   \min_j | \Delsafe^j(z) | + \max_j \pos(-\Delsafe^j(z))) \\
& \le \Delsafe^i(z) + 2c_4\epsilon_{\ell} + c_4 \pos(\Deltil^{\ell}(z)) +   \min_j c_4 | \Delsafe^j(z) | + \max_j c_4 \pos(-\Delsafe^j(z)) \\
& \le \Delsafe^i(z) + 2c_4\epsilon_{\ell} + c_4 \pos(\Deltil^{\ell}(z)) +    c_4 | \Delsafe^i(z) | + \max_j c_4 \pos(-\Delsafe^j(z)).
\end{align*}
Similarly,
\begin{align*}
\pos(-\Delhatsafe^{i,\ell-1}(z)) & \le \pos \left ( - \Delsafe^i(z) + 2c_4 \epsilon_{\ell} + c_4 \pos(\Deltil^{\ell}(z)) + \min_j c_4 | \Delsafe^j(z) | + \max_j c_4 \pos(-\Delsafe^j(z)) \right ) \\
& \le \pos \left ( -\Delsafe^i(z) + \min_j c_4 | \Delsafe^j(z) | \right ) +  2 c_4 \epsilon_{\ell} + c_4 \pos(\Deltil^{\ell}(z)  ) + \max_j c_4 \pos(-\Delsafe^j(z)) \\
& \le \pos \left ( -\Delsafe^i(z) +  c_4 | \Delsafe^i(z) | \right ) +  2 c_4 \epsilon_{\ell} + c_4 \pos(\Deltil^{\ell}(z)  ) + \max_j c_4 \pos(-\Delsafe^j(z)).
\end{align*}
Note that if $\Delsafe^i(z) \le 0$, then 
\begin{align*}
\pos ( -\Delsafe^i(z) + c_4 | \Delsafe^i(z)|) = \pos(-\Delsafe^i(z) - c_4 \Delsafe^i(z)) = (1+c_4) \pos(-\Delsafe^i(z))
\end{align*}
and if $\Delsafe^i(z) > 0$, then for $c_4 < 1$, $-\Delsafe^i(z) + c_4 | \Delsafe^i(z)| \le 0$, so
\begin{align*}
\pos ( -\Delsafe^i(z) + c_4 | \Delsafe^i(z)|) = 0 = (1+c_4) \pos(-\Delsafe^i(z)).
\end{align*}
Thus,
\begin{align*}
\pos(-\Delhatsafe^{i,\ell-1}(z)) \le  (1+c_4) \pos(-\Delsafe^i(z)) +  2 c_4 \epsilon_{\ell} + c_4 \pos(\Deltil^{\ell}(z)  ) + \max_j c_4 \pos(-\Delsafe^j(z)).
\end{align*}

Combining these inequalities, it follows that
\begin{align*}
| \Delhatsafe^{i,\ell}(z) - \Delsafe^i(z) | & \le c_2 \left (\epsilon_\ell +  \pos(\Delhat^{\ell-1}(z)) +  \min_j | \Delhatsafe^{j,\ell-1}(z) | + \max_j \pos(-\Delhatsafe^{j,\ell-1}(z)) \right ) \\
& \le c_2(1 + 2 c_3 + 4 c_4) \epsilon_\ell + c_2 (1 + c_3 + 2c_4) \pos(\Deltil^{\ell}(z)) \\
& \qquad +  c_2(1+c_3 + 3c_4) \max_j \pos(-\Delsafe^j(z)) + c_2 ( 1 + c_4) \min_j |\Delsafe^j(z)|
\end{align*}
and
\begin{align*}
|\Delhat^\ell(z) - \Deltil^\ell(z)| & \le c_1(\epsilon_\ell +  \pos(\Deltil^{\ell}(z)) + \max_j \pos(-\Delhatsafe^{j,\ell-1}(z)))  \\
& \le c_1(1 + 2 c_4) \epsilon_\ell + c_1 ( 1 + c_4) \pos(\Deltil^\ell(z)) + c_1 (1 + 2c_4)\max_j \pos(-\Delsafe^j(z)).
\end{align*}
This proves the inductive hypothesis, as long as
\begin{align}\label{eq:constants_condition6}
c_1 (1 + 2 c_4) \le c_3, \quad c_2 (1 + 2 c_3 + 4 c_4) \le c_4.
\end{align}

For the base case, we need to show that 
\begin{align*}
|\Delhat^0(z) - \Deltil^0(z)| & \le c_3 (\epsilon_0 +  \pos(\Deltil^{0}(z)) +  \max_j \pos(-\Delsafe^j(z))) \\
| \Delhatsafe^0(z) - \Delsafe(z) | & \le c_4 (\epsilon_0 +  \pos(\Deltil^{0}(z)) + \min_j | \Delsafe^j(z) | + \max_j \pos(-\Delsafe^j(z)) ).
\end{align*}
By construction, $\Delhat^0(z) = \Delhatsafe^0(z) = 0$. Thus, it suffices to show $|\Deltil^0(z)| \le c_3 \epsilon_0$ and $|\Delsafe(z)| \le c_4 \epsilon_0$. However, both of these are true by our choice of $\epsilon_0$.
\end{proof}

\begin{lemma}\label{lem:safebai_gap_lb}
On the event $\cEsafe \cap \cEragest$, for all $z \in \cZ$ and all $\ell$, we will have
\begin{align*}
\Deltil^\ell(z) \ge \Delta^{\epsilon_\ell}(z) \quad \text{where} \quad \Delta^{\epsilon_\ell}(z) = \max_{y \in \cZ \ : \ \epsilon_\ell \le \min_i \Delsafe^i(y)} y^\top \thetast - z^\top \thetast .
\end{align*}
\end{lemma}
\begin{proof}
By definition, we will have $z \in \cY_\ell$ if
\begin{align*}
8 \cd \left ( \min_j |\Delhatsafe^{j,\ell-1}(z)| + \max_{j} \pos(-\Delhatsafe^{j,\ell-1}(z)) + \pos(\Delhat^{\ell-1}(z)) \right )  + 8(\cd + \ce)\epsilon_\ell \le \Delhatsafe^{i,\ell}(z).
\end{align*}
The following claim allows us to obtain a sufficient condition to guarantee $z \in \cY_\ell$. 
\begin{claim}\label{claim:Yell_condition}
On the event $\cEsafe \cap \cEragest$,
\begin{align*}
& \min_j |\Delhatsafe^{j,\ell-1}(z)| + \max_{j} \pos(-\Delhatsafe^{j,\ell-1}(z)) + \pos(\Delhat^{\ell-1}(z)) \\
& \qquad \le 2(c_3 + 2 c_4) \epsilon_\ell + (1 + c_3 + 2 c_4) \pos(\Deltil^\ell(z)) + (1 + 2 c_4 ) \min_j | \Delsafe^j(z)| + (1 + c_3 + 2 c_4) \max_j \pos(-\Delsafe^j(z)).
\end{align*}
\end{claim}
\begin{proof}[Proof of \Cref{claim:Yell_condition}]
By \Cref{lem:safe_bai_gaps1} and \Cref{lem:safe_bai_gaps2}, we can bound
\begin{align*}
& \min_j |\Delhatsafe^{j,\ell-1}(z)| \le (1 + c_4) \min_j | \Delsafe^j(z)| +  2c_4 \epsilon_\ell + c_4 \pos(\Deltil^\ell(z)) + c_4 \max_j \pos(-\Delsafe^j(z)) \\
& \max_{j} \pos(-\Delhatsafe^{j,\ell-1}(z)) \le (1 + c_4) \max_j \pos(-\Delsafe^j(z)) + 2c_4 \epsilon_\ell + c_4 \pos(\Deltil^\ell(z))  + c_4 \min_j |\Delsafe^j(z)| \\
& \pos(\Delhat^{\ell-1}(z)) \le (1+c_3) \pos(\Deltil^{\ell}(z)) + 2 c_3 \epsilon_\ell + c_3 \max_j \pos(-\Delsafe^j(z)).
\end{align*}
The claim follows by summing these upper bounds.
\end{proof}
Thus, by \Cref{claim:Yell_condition}, we can bound
\begin{align*}
& 3 \cd \left ( \min_j |\Delhatsafe^{j,\ell-1}(z)| + \max_{j} \pos(-\Delhatsafe^{j,\ell-1}(z)) + \pos(\Delhat^{\ell-1}(z)) \right )  + 3(\cd + \ce)\epsilon_\ell \\
& \qquad \le 3( \cd +  \ce + 2 \cd c_3 + 4 \cd c_4) \epsilon_\ell + 3 \cd (1 + c_3 + 2 c_4) \pos(\Deltil^\ell(z)) \\
& \qquad + 3 \cd (1 + 2 c_4 ) \min_j | \Delsafe^j(z)| + 3 \cd (1 + c_3 + 2 c_4) \max_j \pos(-\Delsafe^j(z)).
\end{align*}
Furthermore, by \Cref{lem:safe_bai_gaps2},
\begin{align*}
\Delsafe^i(z) - c_4 \left ( \epsilon_\ell +  \pos(\Deltil^{\ell}(z)) + \min_j | \Delsafe^j(z)| + \max_j \pos(-\Delsafe^j(z)) \right ) \le \Delhatsafe^{i,\ell}(z)
\end{align*}
It follows that a sufficient condition for $z \in \cY_\ell$ is
\begin{align}\label{eq:suff_cond_Yell}
\begin{split}
& ( 3\cd +  3\ce + 6 \cd c_3 + 12 \cd c_4 + c_4) \left ( \epsilon_\ell +   \pos(\Deltil^\ell(z)) + \min_j | \Delsafe^j(z)| +   \max_j \pos(-\Delsafe^j(z)) \right ) \\
&\qquad \le \Delsafe^i(z), \quad \forall i \in [m].
\end{split}
\end{align}


If $y_\ell =  \argmax_{y \in \cZ \ : \ \epsilon_\ell \le \min_i \Delsafe^i(y)} y^\top \thetast$ is in $\cY_\ell$, then we are done. Assume then that $y_\ell \not\in \cY_\ell$. By construction, since $\Delsafe^i(y_\ell) > 0$ for all $i$, $\max_j \pos(-\Delsafe^j(z)) = 0$. Using that \eqref{eq:suff_cond_Yell} is a sufficient condition for inclusion in $\cY_\ell$, this implies that
\begin{align*}
\exists i \in [m] \quad \text{s.t.} \quad ( 3\cd +  3\ce + 6 \cd c_3 + 12 \cd c_4 + c_4) \left ( \epsilon_\ell +   \pos(\Deltil^\ell(y_\ell)) + \min_j | \Delsafe^j(y_\ell)|  \right ) > \Delsafe^i(y_\ell).
\end{align*}
which implies 
\begin{align}\label{eq:Deltil_Suff}
\exists i \in [m] \quad \text{s.t.} \quad ( 3\cd + 3\ce + 6 \cd c_3 + 12 \cd c_4 + c_4) \left ( \epsilon_\ell +   \pos(\Deltil^\ell(y_\ell)) + | \Delsafe^i(y_\ell)|  \right ) > \Delsafe^i(y_\ell).
\end{align}
By construction, though, $\Delsafe^i(y_\ell) \ge \epsilon_\ell$. If we assume that 
\begin{align}\label{eq:constants_condition4}
3\cd +  3\ce + 6 \cd c_3 + 12 \cd c_4 + c_4 \le 1/4,
\end{align}
then \eqref{eq:Deltil_Suff} can only hold if $ \pos(\Deltil^{\ell}(y_\ell)) > 0$. This implies that $\max_{y \in \cY_\ell} y^\top \thetast > y_\ell^\top \thetast$. Thus, in this case,
\begin{align*}
\Deltil^\ell(z) = \max_{y \in \cY_\ell} y^\top \thetast - z^\top \thetast > y_\ell^\top \thetast - z^\top \thetast = \Delta^{\epsilon_\ell}(z)
\end{align*}
which proves the result.
\end{proof}

\begin{lemma}\label{lem:Yend_properties}
On $\cEsafe \cap \cEragest$, for all $z \in \cYend$ we have
\begin{align*}
& \Delsafe^i(z) \ge - \cg \epsilon, \quad \forall i \in [m], \\
& \Delhatsafe^{i,\iotaeps}(z) \ge (3 \cd + 3 \ce - \cg) \epsilon, \quad \forall i \in [m].
\end{align*}
Furthermore, $\zst \in \cYend$. 
\end{lemma}
\begin{proof}
Recall that 
\begin{align*}
\cYend = \{ z \in \cZ \ : \  & 3 \cd \left ( \min_j |\Delhatsafe^{j,\iotaeps}(z)| + \max_{j} \pos(-\Delhatsafe^{j,\iotaeps}(z)) + \pos(\Delhat^{\iotaeps}(z)) \right ) \\
	& + 3(\cd + \ce)\epsilon- \cg \epsilon \le \Delhatsafe^{i,\iotaeps}(z), \forall i \in [m] \} 
\end{align*}
On $\cEsafe$, we have
\begin{align*}
\Delhatsafe^{i,\iotaeps}(z) \le \Delsafe^i(z) + 3 \cd \left ( \min_j |\Delhatsafe^{j,\iotaeps}(z)| + \max_{j} \pos(-\Delhatsafe^{j,\iotaeps}(z)) + \pos(\Delhat^{\iotaeps}(z)) \right ) + 3(\cd + \ce)\epsilon
\end{align*}
so it follows that if $z \in \cYend$, then
\begin{align*}
-\cg \epsilon \le \Delsafe^i(x).
\end{align*}

To see that $\zst \in \cYend$, note that by definition of $\cYend$, using a calculation analogous to \eqref{eq:suff_cond_Yell}, a sufficient condition for $z \in \cYend$ is
\begin{align*}
& ( 3\cd +  3\ce + 6 \cd c_3 + 12 \cd c_4 + c_4 - \cg) \epsilon +  (3 \cd + 3 \cd c_3 + 6 \cd c_4 + c_4) \pos(\Deltil^{\iotaeps}(z)) \\
& \qquad +  (3 \cd + 6 \cd c_4 + c_4 ) \min_j | \Delsafe^j(z)| +  (3 \cd + 3 \cd c_3 + 6 \cd c_4 + c_4) \max_j \pos(-\Delsafe^j(z)) \\
&\qquad \le \Delsafe^i(z), \quad \forall i \in [m].
\end{align*}
By definition of $\zst$ and since, by \Cref{lem:safe_bai_gaps1}, all $z \in \cY_{\iotaeps}$ are safe, we have $\Delta^{\epsilon_{\iotaeps}}(\zst) \le 0$. Furthermore, by definition we also have $\Delsafe^j(\zst) \ge 0$ for all $j$, so $\pos(-\Delsafe^j(\zst)) = 0$. Thus, assuming that
\begin{align}\label{eq:constants_condition2}
3\cd +  3\ce + 6 \cd c_3 + 12 \cd c_4 + c_4 - \cg \le 0
\end{align}
a sufficient condition to guarantee $\zst \in \cYend$ is that
\begin{align*}
(8 \cd + 16 \cd c_4 + c_4 ) \min_j | \Delsafe^j(\zst)|  \le \Delsafe^i(\zst), \quad \forall i \in [m].
\end{align*}
However, as long as
\begin{align}\label{eq:constants_condition3}
3 \cd + 6 \cd c_4 + c_4 \le 1,
\end{align}
this is true, since by definition $\Delsafe^i(\zst) \ge 0$. 
\end{proof}

\subsection{Algorithm Correctness and Sample Complexity}
\begin{lemma}[Correctness]\label{lem:correctness}
On $\cEsafe \cap \cEragest$, we will have that
\begin{align*}
\zhat^\top \thetast \ge (\zst)^\top \thetast - \frac{c_3(1+\cg)}{1 - c_3} \epsilon, \quad \Delsafe^i(\zhat) \ge - \cg \epsilon, \forall i \in [m]. 
\end{align*}
\end{lemma}
\begin{proof}
We choose $\zhat$ to be any $z \in \cYend$ such that $\Delhat^{\eend}(z) = 0$. By \Cref{lem:Yend_properties}, we have that $\Delsafe^i(\zhat) \ge - \cg \epsilon$ for all $i \in [m]$. If $\Deltil^{\eend}(\zhat) \le 0$, we are done, since by \Cref{lem:Yend_properties}, $\zst \in \cYend$, so $\zhat^\top \thetast \ge (\zst)^\top \thetast$. Assume that $\Deltil^{\eend}(\zhat) > 0$. 
By \Cref{lem:safe_bai_gaps2}, we have that 
\begin{align*}
\Deltil^{\eend}(\zhat) \le c_3 \epsilon + c_3 \pos(\Deltil^{\eend}(\zhat)) + c_3 \max_j  \pos(-\Delsafe^j(\zhat)).
\end{align*}
By \Cref{lem:Yend_properties}, since $\zhat \in \cYend$, $\pos(-\Delsafe^j(\zhat)) \le  \cg \epsilon$ for all $j$, so we can bound
\begin{align*}
\Deltil^{\eend}(\zhat) \le c_3 (1 + \cg) \epsilon + c_3 \pos(\Deltil^{\eend}(\zhat)) = c_3 (1 + \cg) \epsilon + c_3 \Deltil^{\eend}(\zhat).
\end{align*}
We can rearrange this as
\begin{align*}
\Deltil^{\eend}(\zhat) \le \frac{c_3(1+\cg)}{1 - c_3} \epsilon
\end{align*}
which proves the result, since, by \Cref{lem:Yend_properties}, $\Delend(\zhat) = \max_{y \in \cYend} y^\top \thetast - \zhat^\top \thetast \ge (\zst)^\top \thetast - \zhat^\top \thetast$.
\end{proof}

\begin{lemma}\label{lem:complexity_safety}
On $\cEragest \cap \cEsafe$, the total complexity of \Cref{line:safety_gap_est} is bounded by
\begin{align*}
C \cdot \sum_{\ell=1}^{\iotaeps}\inf_{\lambda \in \simplex_\cX} \max_{z \in \cZ} \frac{ \| z \|_{A(\lambda)^{-1}}^2 \cdot \log(\ub)}{\left ( \min_j |\Delsafe^j(z)| + \max_j \pos(-\Delsafe^j(z))  + \pos(\Delta^{\epsilon_{\ell-1}}(z)) + \epsilon_\ell \right )^2} + 4 \iotaeps \log (\tfrac{4 m |\cZ| \iotaeps^2}{\delta})
\end{align*}
for an absolute constant $C$. 
\end{lemma}
\begin{proof}
Applying the same argument as in \Cref{claim:Yell_condition} but in the opposite direction, we have
\begin{align*}
&  \min_j |\Delhatsafe^{j,\ell-1}(z)| + \max_{j} \pos(-\Delhatsafe^{j,\ell-1}(z)) + \pos ( \Delhat^{\ell-1}(z) ) \\
& \ge -2(c_3 + 2 c_4) \epsilon_\ell + (1 - c_3 - 2 c_4) \pos(\Deltil^\ell(z)) + (1 - 2 c_4 ) \min_j | \Delsafe^j(z)| + (1 - c_3 - 2 c_4) \max_j \pos(-\Delsafe^j(z)).
\end{align*}
We assume that $c_3,c_4$, and $\cabs$ are chosen such that
\begin{align}\label{eq:constants_condition1}
1 - 2 c_3 - 4 c_4 \ge \cabs, 
\end{align}
which allows us to bound:
\begin{align*}
& \inf_{\lambda \in \simplex_\cX} \max_{z \in \cZ}  - \cd \left ( \min_j |\Delhatsafe^{j,\ell-1}(z)| + \max_{j} \pos(-\Delhatsafe^{j,\ell-1}(z)) + \pos(\Delhat^{\ell-1}(z)) + \epsilon_\ell \right ) + \sqrt{\frac{\| z \|_{A(\lambda)^{-1}}^2 \cdot \log ( \ub)}{\tau}}\\
& \le \inf_{\lambda \in \simplex_\cX} \max_{z \in \cZ} - \cd \cabs \left ( \min_j |\Delsafe^j(z)| + \max_j \pos(-\Delsafe^j(z))  + \pos(\Deltil^{\ell-1}(z)) + \epsilon_\ell \right ) + \sqrt{\frac{\| z \|_{A(\lambda)^{-1}}^2 \cdot \log ( \ub)}{\tau}} .
\end{align*}
It follows that if, for all $z \in \cZ$,
\begin{align*}
\tau \ge \frac{ \| z \|_{A(\lambda)^{-1}}^2}{ \left ( \cd \cabs \min_j |\Delsafe^j(z)| + \cd \cabs \max_j \pos(-\Delsafe^j(z))  + \cd \cabs \pos(\Deltil^{\ell-1}(z)) + (\cd \cabs + \ce) \epsilon_\ell \right )^2} \cdot \log \ub
\end{align*}
we will have that this is less than $\ce \epsilon_\ell$. Since we can take the best-case $\lambda \in \simplex_\cX$, and since $\tau_\ell$ is always within a factor of 2 of the optimal, it follows that
\begin{align*}
\tau_\ell & \le \inf_{\lambda \in \simplex_\cX} \max_{z \in \cZ} \frac{2 \| z \|_{A(\lambda)^{-1}}^2 \cdot \log \ub}{\left ( \cd \cabs \min_j |\Delsafe^j(z)| + \cd \cabs \max_j \pos(-\Delsafe^j(z))  + \cd \cabs \pos(\Deltil^{\ell-1}(z)) + (\cd \cabs + \ce) \epsilon_\ell \right )^2} \\
& \qquad + 4 \log \ub
\end{align*}	
The result then follows by summing over epochs and lower bounding $\Deltil^{\ell-1}(z)$ by $\Delta^{\epsilon_{\ell-1}}(z)$ using \Cref{lem:safebai_gap_lb}, and assuming that
\begin{align}\label{eq:constants_condition7}
\cd \cabs + \ce \ge \cabs.
\end{align}
\end{proof}

\begin{proof}[Proof of \Cref{thm:main_complexity}]
By \Cref{lem:cEsafe_high_prob} we have that $\cEsafe$ holds with probability at least $1-\delta$. By \Cref{lem:cEragest_high_prob}, we have that $\cEragest^\ell$ holds with probability at least $1-\delta/(4 \ell^2)$. It follows then that $\cEsafe \cup (\cup_\ell \cEragest^\ell)$ holds with probability at least
\begin{align*}
1 - \delta - \sum_\ell \frac{\delta}{4 \ell^2} \ge 1 - 2 \delta.
\end{align*}
Assume henceforth that $\cEsafe \cup (\cup_\ell \cEragest^\ell)$ holds.
\Cref{eq:main_thm_correctness} follows by \Cref{lem:correctness}. The total number of samples collected on \Cref{line:safety_gap_est}  can be bounded by \Cref{lem:complexity_safety}. It remains to bound the total number of samples used by $\ragest$. 

By \Cref{lem:ragest_complexity}, at epoch $\ell$ $\ragest$ will collect at most
\begin{align*}
C \lceil \log \frac{2}{\cf \epsilon_\ell} \rceil \cdot \inf_{\lambda \in \simplex_\cX} \max_{z \in \cZ} \frac{\| z - \yst^\ell \|_{A(\lambda)^{-1}}^2 \cdot \log(\tfrac{8 |\cZ|^2  \log^4(1/\epsilon)}{\delta} )}{( \max_j \pos(-\Delhatsafe^{j,\ell-1}(z)) + \pos ( \Deltil^{\ell}(z)) + \epsilon_\ell)^2} + 8 \lceil \log \tfrac{2}{\cf\epsilon} \rceil \log (\tfrac{4 |\cZ|^2 \lceil \log \tfrac{2}{\cf\epsilon} \rceil^2}{\delta})
\end{align*}
samples, where $\yst^\ell = \argmax_{y \in \cY_\ell} y^\top \thetast$. Assume that $\max_j \pos(-\Delsafe^j(z)) > 0$, then we can upper bound $\min_j | \Delsafe^j(z)| \le \max_j \pos(-\Delsafe^j(z))$, and by \Cref{lem:safe_bai_gaps2} we can lower bound
\begin{align*}
\max_j \pos(-\Delhatsafe^{j,\ell-1}(z)) & \ge (1-2c_4) \max_j \pos(-\Delsafe^{j}(z)) - c_4 \pos(\Deltil^{\ell-1}(z)) - c_4 \epsilon_{\ell-1}.
\end{align*}
Assume instead that $\max_j \pos(-\Delsafe^j(z)) = 0$. Then again by \Cref{lem:safe_bai_gaps2}:
\begin{align*}
\max_j \pos(-\Delhatsafe^{j,\ell-1}(z))  & \ge 0 = \max_j \pos(-\Delsafe^j(z)) \ge (1-2c_4) \max_j \pos(-\Delsafe^{j}(z)) - c_4 \pos(\Deltil^{\ell-1}(z)) - c_4 \epsilon_{\ell-1}.
\end{align*}
By \Cref{lem:safe_bai_gaps1}, it follows that
\begin{align*}
&  \max_j \pos(-\Delhatsafe^{j,\ell-1}(z)) + \pos ( \Deltil^{\ell}(z)) + \epsilon_\ell \\
& \ge (1-2c_4) \max_j \pos(-\Delsafe^{j}(z)) + (1-c_4) \pos(\Deltil^\ell(z)) + (1 - 2 c_4) \epsilon_\ell.
\end{align*}
By definition and \Cref{lem:safe_bai_gaps1} and \Cref{lem:Yend_properties} for all $\ell$ including $\ell = \eend$, we can bound $\pos ( -\Delsafe^{j}(\yst^\ell)) \le \cg \epsilon$. Furthermore, by definition $\pos(\Deltil^\ell(\yst^\ell)) = 0$. Putting all of this together, we have:
\begin{align*}
& \inf_{\lambda \in \simplex_\cX} \max_{z \in \cZ} \frac{\| z - \yst^\ell \|_{A(\lambda)^{-1}}^2 \cdot \log(\tfrac{8 |\cZ|^2 \log^4(1/\epsilon)}{\delta})}{( \max_j \pos( - \Delhatsafe^{j,\ell-1}(z)) + \pos ( \Deltil^{\ell}(z)) + \epsilon_\ell)^2} \\
& \le \inf_{\lambda \in \simplex_\cX} \max_{z \in \cZ} \frac{\| z - \yst^\ell \|_{A(\lambda)^{-1}}^2 \cdot \log(\tfrac{8 |\cZ|^2  \log^4(1/\epsilon)}{\delta})}{( (1-2c_4) \max_j \pos( - \Delsafe^{j}(z)) + (1-c_4)\pos ( \Deltil^{\ell}(z)) + (1-2c_4)\epsilon_\ell)^2} \\
& \le \inf_{\lambda \in \simplex_\cX} \max_{z \in \cZ} \frac{2\| z - \zst \|_{A(\lambda)^{-1}}^2 \cdot \log(\tfrac{8 |\cZ|^2  \log^4(1/\epsilon)}{\delta})}{( (1-2c_4) \max_j \pos( - \Delsafe^{j}(z)) + (1-c_4)\pos ( \Deltil^{\ell}(z)) + (1-2c_4)\epsilon_\ell)^2} \\
& \qquad + \inf_{\lambda \in \simplex_\cX}  \frac{2 \| \zst - \yst^\ell \|_{A(\lambda)^{-1}}^2 \cdot \log(\tfrac{8 |\cZ|^2  \log^4(1/\epsilon)}{\delta})}{( (1-2c_4) \max_j \pos( - \Delsafe^{j}(\yst^\ell)) + (1-c_4)\pos ( \Deltil^{\ell}(\yst^\ell)) + (1-2c_4-\cg)\epsilon_\ell)^2} \\
& \le \inf_{\lambda \in \simplex_\cX} \max_{z \in \cZ} \frac{4\| z - \zst \|_{A(\lambda)^{-1}}^2 \cdot \log(\tfrac{8 |\cZ|^2  \log^4(1/\epsilon)}{\delta})}{( (1-2c_4) \max_j \pos( - \Delsafe^{j}(z)) + (1-c_4)\pos ( \Deltil^{\ell}(z)) + (1-2c_4-\cg)\epsilon_\ell)^2}
\end{align*}
As long as
\begin{align}\label{eq:constants_condition5}
1 - 2 c_4 - \cg \ge \cabs, 
\end{align}
summing over the epochs and lower bounding $\Deltil^{\ell}(z)$ by $\Delta^{\epsilon_\ell}(z)$ via \Cref{lem:safebai_gap_lb} gives the result. Finally, the settings of the constants follows from \Cref{lem:constants}.
\end{proof}

\subsection{Proofs of Corollaries to \Cref{thm:main_complexity_nice}}\label{sec:complexity_corollaries}

\begin{proof}[Proof of \Cref{cor:bai}]
If $m = 1$, $\mu_{*,1} = 0$, and $\gamma = 1$, then we have $\Delsafe(z) = 1$ for each $z$, and $\Delta^{\epstil}(z) = \Delta(z)$ for $\epsilon \le 1$. The result follows directly from this and some algebra.
\end{proof}

\begin{proof}[Proof of \Cref{cor:worst_case_upper}]
We can trivially upper bound the complexity given in \Cref{thm:main_complexity_nice} by
\begin{align*}
& C \cdot \inf_{\lambda \in \simplex_\cX} \max_{z \in \cZ} \frac{ \| z \|_{A(\lambda)^{-1}}^2 \cdot \log(\frac{m | \cZ |}{\delta})}{\epsilon^2} + C \cdot \inf_{\lambda \in \simplex_\cX} \max_{z \in \cZ} \frac{\| z - \zst \|_{A(\lambda)^{-1}}^2 \cdot \log(\tfrac{ |\cZ|}{\delta})}{\epsilon^2} + C_0 \\
& \le C \cdot \inf_{\lambda \in \simplex_\cX} \max_{z \in \cZ} \frac{ \| z \|_{A(\lambda)^{-1}}^2 \cdot \log(\frac{m | \cZ |}{\delta})}{\epsilon^2} + C_0.
\end{align*}
In the case when $\cX = \cZ$, we can bound $ \inf_{\lambda \in \simplex_\cX} \max_{z \in \cZ} \| z \|_{A(\lambda)^{-1}}^2 \le d$ by Kiefer-Wolfowitz \citep{lattimore2020bandit}, which proves the result. 
\end{proof}

\newcommand{\betatil}{\widetilde{\beta}}

\section{Computationally Efficient Optimization}

Throughout, we will let $\cR(z;\xi_1,\ldots,\xi_n)$ denote some generic weighted risk estimate of the form
\begin{align*}
\cR(z;\xi_1,\ldots,\xi_n) = \sum_{t=1}^T f_t(\xi_1,\ldots,\xi_n) \I \{ z(u_t) \neq v_t \}
\end{align*}
for some weights $ f_t(\xi_1,\ldots,\xi_n)$ and observations $(u_t,v_t)$. The exact setting of $\cR$ will change from line to line---we simply use it as a stand-in for an objective that a cost-sensitive-classification oracle can efficiently minimize. We will also use $f(\xi_1,\ldots,\xi_n)$ to refer to some generic function (the particular form of which is not important).

\begin{lem}\label{lem:class_dist_equiv}
Consider some $z,\ztil \in \simplex_\cH$. Denote
\begin{align*}
\rho_\lambda(h,h') = \Exp_{U \sim \nu} \left [\frac{\I \{ h(U) \neq h'(U) \}}{\lambda(U)/\nu(U)} \right ] = \| h - h' \|_{A(\lambda)^{-1}}^2
\end{align*}
and overload notation so that $z = \sum_{h \in \cH} z_h h$ denotes the feature vector for the mixed classifier $z$. Then,
\begin{align*}
\sum_{h,h' \in \cH} z_h \ztil_{h'} \rho_\lambda(h,h') = \Exp_{U \sim \nu} \left [ \frac{(z(U) - \ztil(U))^2}{\lambda(U)/\nu(U)} \right ] = \| z - \ztil \|_{A(\lambda)^{-1}}^2.
\end{align*}
\end{lem}
\begin{proof}
Note that
\begin{align*}
\rho_\lambda(h,h') = \Exp_{U \sim \nu} \left [\frac{\I \{ h(U) \neq h'(U) \}}{\lambda(U)/\nu(U)} \right ] = \Exp_{U \sim \nu} \left [\frac{| h(U) - h'(U) |}{\lambda(U)/\nu(U)} \right ] = \Exp_{U \sim \nu} \left [\frac{( h(U) - h'(U) )^2}{\lambda(U)/\nu(U)} \right ]
\end{align*}
where the final equality holds because $|h(U) - h'(U)|$ is always either 0 or 1. Thus,
\begin{align}
\sum_{h,h' \in \cH} z_h \ztil_{h'} \rho_\lambda(h,h') & = \sum_{h,h' \in \cH} z_h \ztil_{h'}  \Exp_{U \sim \nu} \left [\frac{( h(U) - h'(U) )^2}{\lambda(U)/\nu(U)} \right ] \nonumber \\
& =  \Exp_{U \sim \nu} \left [\frac{\sum_{h,h' \in \cH} z_h \ztil_{h'} ( h(U) - h'(U) )^2}{\lambda(U)/\nu(U)} \right ] \nonumber \\
& =  \Exp_{U \sim \nu} \left [\frac{\sum_{h,h' \in \cH} z_h \ztil_{h'}(h(U) + h'(U) - 2 h(U) h'(U))}{\lambda(U)/\nu(U)} \right ] \label{eq:mixed_class_equiv}.
\end{align}
However, 
\begin{align*}
\sum_{h,h' \in \cH} z_h \ztil_{h'} h(U) = \sum_{h \in \cH} z_h h(U) = z(U), \quad \sum_{h,h' \in \cH} z_h \ztil_{h'} h'(U) = \ztil(U)
\end{align*}
and
\begin{align*}
\sum_{h,h' \in \cH} z_h \ztil_{h'} h(U) h'(U) & = ( \sum_{h \in \cH} z_h h(U)) ( \sum_{h' \in \cH} \ztil_{h'} h'(U)) = z(U) \ztil(U).
\end{align*}
Thus,
\begin{align*}
\eqref{eq:mixed_class_equiv} = \Exp_{U \sim \nu} \left [\frac{(z(U) - \ztil(U))^2}{\lambda(U)/\nu(U)} \right ]
\end{align*}
which proves the first equality. To prove the second, recall that $[h]_u = \nu(u) h(u)$, so $[z]_u = \sum_{h \in \cH} z_h [h]_u =  \nu(u) z(u)$. It follows that,
\begin{align*}
\| z - \ztil \|_{A(\lambda)^{-1}}^2 & = \sum_{u} \frac{\nu(u)^2}{\lambda(u)} (z(u) - \ztil(u))^2 =  \Exp_{U \sim \nu} \left [\frac{(z(U) - \ztil(U))^2}{\lambda(U)/\nu(U)} \right ]
\end{align*}
which proves the second equality. 
\end{proof}

\subsection{Computational Efficiency of $\ragest$}
$\ragest$ requires solving the optimization
\begin{align}\label{eq:ragest_opt}
\inf_{\lambda \in \simplex_\cX} \max_{z \in \simplex_\cH} \min_{\alpha \in \cA} - \ca ( \pos(- \Delhatsafe(z)) + \pos( \Delhat^{\ell-1}(z)) + \epsilon_\ell) + \alpha \| z - \yhat_{\ell-1} \|_{A(\lambda)^{-1}}^2 + \frac{\log(2|\cZ|^2 |\cA| \ell^2/\delta)}{\alpha \tau}   .
\end{align}
Here we take $\tau$ to be fixed, and recall that 
\begin{align*}
	\yhat_\ell & \leftarrow \argmin_{y \in \cY} \min_{\alpha \in \cA} \Rtilgam_\ell(y) - \Rtilgam_\ell(\yhat_{\ell-1}) + 2\alpha \| y - \yhat_{\ell-1} \|_{A(\lambda_\ell)^{-1}}^2 + \frac{2 \log(2 |\cZ|^2 |\cA| \ell^2/\delta)}{\alpha \tau_\ell} \\
	\Delhat^\ell(y) & \leftarrow \min_{\alpha \in \cA} \Rtilgam_\ell(y) - \Rtilgam_\ell(\yhat_{\ell}) + \alpha \| y - \yhat_{\ell} \|_{A(\lambda_\ell)^{-1}}^2 + \frac{\log(2|\cZ|^2 |\cA|\ell^2/ \delta)}{\alpha \tau_\ell} .
\end{align*}
Furthermore, $\cY$ will be a set of the form
\begin{align*}
\bigcup_{k = 1}^{\ell'} \cY_k = \bigcup_{k = 1}^{\ell'} \left \{ z \in \cZ \ : \ c ( \epsilon_k + \pos(\Delhat^{k-1}(z)) +  \max_{j \in [n]} \pos(-\Delhatsafe^{j,k-1}(z)) +  \min_{j \in [n]} | \Delhatsafe^{j,k-1}(h)|) \le \Delhatsafe^{i,k}(h),  \forall i \in [n] \right \} 
\end{align*}
Recall also that
\begin{align*}
 \| h - h' \|_{A(\lambda)^{-1}}^2 = \Exp_{U \sim \nu} \left [ \frac{\I \{ h(U) \neq h'(U) \}}{(9\lambda(U)/10 + 1/10d)/\nu(U)}  \right ] = \sum_{U \in \cX} \frac{\nu(U)^2}{9\lambda(U)/10 + 1/10d} \I \{ h(U) \neq h'(U) \} 
\end{align*}
and
\begin{align*}
\Rtilgam_\ell(h) = \frac{1}{\tau_\ell} \sum_{t=1}^{\tau_\ell} \frac{1}{w_t + \alpha} \I \{ h(u_t) \neq v_t \} .
\end{align*}
For $z \in \simplex_\cH$, we denote $\Rtilgam_\ell(z) = \sum_{h \in \cH} z_h \Rtilgam_\ell(h)$ and $\cR(z;\alpha) = \sum_{h \in \cH} z_h \cR(h;\alpha)$. Finally, we assume that $\Delhatsafe(z) = \min_{\alpha \in \cA} \cR(z; \alpha) + f(\alpha)$.


\subsubsection{Solving for $\yhat_\ell$}
Using \Cref{lem:class_dist_equiv}, we can write the optimization for $\yhat_\ell$ as
\begin{align*}
& \min_{k \in [\ell']} \min_{y \in \cY_k} \min_{\alpha \in \cA} \frac{1}{\tau_\ell} \sum_{t=1}^{\tau_\ell} \frac{1}{w_t + \alpha} \sum_{h \in \cH} y_h \I \{ h(u_t) \neq v_t \}  + \alpha  \sum_{h, h' \in \cH} y_h \yhat_{\ell-1,h'}  \sum_{U \in \cX} \frac{\nu(U)^2}{9\lambda_\ell(U)/10 + 1/10d} \I \{ h(U) \neq h'(U) \} \\
& \qquad - \Rtilgam_\ell(\yhat_{\ell-1})  + \frac{\log(2 |\cZ|^2 |\cA| \ell^2/\delta)}{\alpha \tau_\ell} 
\end{align*}
We can rewrite 
\begin{align*}
\sum_{h, h' \in \cH} y_h \yhat_{\ell-1,h'}  \sum_{U \in \cX} \frac{\nu(U)^2}{9\lambda_\ell(U)/10 + 1/10d} \I \{ h(U) \neq h'(U) \} = \sum_{h \in \cH} y_h \sum_{i=1}^{\| \yhat_{\ell-1} \|_0 | \cX|} w_i \I \{ h(u_i) \neq v_i \}
\end{align*}
for some weights $w_i$. It follows that if $\| \yhat_{\ell-1} \|_0$ is polynomial in problem parameters then the optimization for $\yhat_\ell$ can be written as
\begin{align*}
\min_{k \in [\ell']} \min_{y \in \cY_k} \min_{\alpha \in \cA} \cR(y; \alpha) + f(\alpha)
\end{align*}
for $\cR(y;\alpha)$ a CSC loss over only polynomially many points (as well as linear in $y$ and convex in $\alpha$), and $f(\alpha)$ convex in $\alpha$. Note also that, for any $y$, we can upper bound $\cR(y;\alpha) \le \cO( \frac{1}{\alpha} + d \alpha)$. 
Here $\cY_k$ a set of the form
\begin{align*}
\left \{ z \in \simplex_\cH \ : \  \sum_{h \in \cH} z_h c \Big ( \epsilon_k + \pos(\Delhat^{k-1}(h)) +  \max_{j \in [n]} \pos(-\Delhatsafe^{j,k-1}(h)) +  \min_{j \in [n]} | \Delhatsafe^{j,k-1}(h)| \Big ) \le \sum_{h \in \cH} z_h \Delhatsafe^{i,k}(h),  \forall i \in [n] \right \} 
\end{align*}
$\yhat_\ell$ will be the element in $\cY_k$ minimizing the, for the $k$ achieving the minimum.
The dual of this problem has the form 
\begin{align*}
\min_{k \in [\ell']} & \min_{z \in \simplex_\cH} \min_{\alpha \in \cA} \max_{\mu_i \ge 0, i \in [n]}  \cR(z; \alpha) + f(\alpha) \\
& + \sum_{i=1}^n \mu_i \left ( \sum_{h \in \cH} z_h c \Big ( \epsilon_k + \pos(\Delhat^{k-1}(h)) +  \max_{j \in [n]} \pos(-\Delhatsafe^{j,k-1}(h)) +  \min_{j \in [n]} | \Delhatsafe^{j,k-1}(h)| \Big ) - \sum_{h \in \cH} z_h \Delhatsafe^{i,k}(h) \right ).
\end{align*}
Note that we can swap the min over $\alpha$ and $z$ without issue. Furthermore, for a fixed $\mu$, the objective is linear in $z$, and for a fixed $z$, the objective is linear in $\mu$. By the minimax theorem, we can then swap the min and max to obtain the equivalent optimization:
\begin{align*}
\min_{k \in [\ell']} & \min_{\alpha \in \cA} \max_{\mu_i \ge 0, i \in [n]} \min_{z \in \simplex_\cH}    \cR(z; \alpha) + f(\alpha) \\
& + \sum_{i=1}^n \mu_i \left ( \sum_{h \in \cH} z_h c \Big ( \epsilon_k + \pos(\Delhat^{k-1}(h)) +  \max_{j \in [n]} \pos(-\Delhatsafe^{j,k-1}(h)) +  \min_{j \in [n]} | \Delhatsafe^{j,k-1}(h)| \Big ) - \sum_{h \in \cH} z_h \Delhatsafe^{i,k}(h) \right ).
\end{align*}
We can simply enumerate over $k$ and $\alpha$, as there are a finite number of each of these constraints. For a fixed $k$ and $\alpha$, to solve the inner maxmin problem, we can apply the approach proposed in \cite{agarwal2018reductions}. In particular, we alternate between running the exponential gradient algorithm for the $\mu$ player, and computing the best-response for the $z$ player. The update to the $\mu$ player is trivial, as the problem is simply linear in $\mu$ (in practice, as in \cite{agarwal2018reductions}, we will also upper bound the domain of $\mu_i$ by some value $B$, to ensure this is finite). 

Computing the best-response for the $z$ player (with $\mu$ fixed) is slightly trickier. Ignoring all other parameters, which are all currently fixed, the minimization over $z$ can be written as
\begin{align*}
& \min_{z \in \simplex_\cH}  \sum_{h \in \cH} z_h \sum_t a_t \I \{ h(u_t) \neq o_t \}  \\
& + \sum_{i=1}^n \mu_i \left ( \sum_{h \in \cH} z_h c \Big ( \epsilon_k + \pos(\Delhat^{k-1}(h)) +  \max_{j \in [n]} \pos(-\Delhatsafe^{j,k-1}(h)) +  \min_{j \in [n]} | \Delhatsafe^{j,k-1}(h)| \Big ) - \sum_{h \in \cH} z_h \Delhatsafe^{i,k}(h) \right ). \\
& = \min_{z \in \simplex_\cH} \sum_{h \in \cH} z_h \bigg ( \sum_t a_t \I \{ h(u_t) \neq o_t \} \\
& \qquad +  \sum_{i \in [n]} c_i \Big ( \epsilon_k + \pos(\Delhat^{k-1}(h)) +  \max_{j \in [n]} \pos(-\Delhatsafe^{j,k-1}(h)) +  \min_{j \in [n]} | \Delhatsafe^{j,k-1}(h)| -  \Delhatsafe^{i,k}(h) \Big ) \bigg ) .
\end{align*}
Now note that $\max_{j \in [n]} \pos(-\Delhatsafe^{j,k-1}(z)) = \sup_{\lamtil \in \simplex_n} \sum_{j \in [n]}  \lamtil_j \pos(-\Delhatsafe^{j,k-1}(z))$, and similarly for $\min_{j \in [n]} | \Delhatsafe^{j,k-1}(z)|$. 
Using this, we can rewrite the above optimization as
\begin{align*}
& \min_{z \in \simplex_\cH}  \max_{\lamtil^{1h} \in \simplex_n, h \in \cH}  \min_{\lamtil^{2h} \in \simplex_n, h \in \cH} \sum_{h \in \cH} z_h \bigg ( \sum_t a_t \I \{ h(u_t) \neq o_t \} \\
& \qquad +  \sum_{i \in [n]} c_i \Big ( \epsilon_k + \pos( \Delhat^{k-1}(h)) +  \sum_{j \in [n]} \lamtil_j^{1h} \pos(-\Delhatsafe^{j,k-1}(h)) +  \sum_{j \in [n]} \lamtil_j^{2h} | \Delhatsafe^{j,k-1}(h) | -  \Delhatsafe^{i,k}(h) \Big ) \bigg ). 
\end{align*}
We also have:
\begin{align*}
\pos(-\Delhatsafe^{j,k-1}(z)) = \max_{\beta \in [0,1]} -\beta \Delhatsafe^{j,k-1}(z) , \quad | \Delhatsafe^{j,k-1}(z)| = \max_{\beta \in [-1,1]} \beta \Delhatsafe^{j,k-1}(z) .
\end{align*}
So we can further simplify the above to:
\begin{align*}
&  \min_{z \in \simplex_\cH}  \max_{\lamtil^{1h} \in \simplex_n, h \in \cH}  \min_{\lamtil^{2h} \in \simplex_n, h \in \cH} \max_{\beta_1^h, \beta_{2}^{hj} \in [0,1], \beta_{3}^{hj} \in [-1,1], h \in \cH} \sum_{h \in \cH} z_h \bigg ( \sum_t a_t \I \{ h(u_t) \neq o_t \} \\
& \qquad +  \sum_{i \in [n]} c_i \Big ( \epsilon_k + \beta_{1}^{h} \Delhat^{k-1}(h) -  \sum_{j \in [n]} \lamtil_j^{1h} \beta_{2}^{hj} \Delhatsafe^{j,k-1}(h)  +  \sum_{j \in [n]} \lamtil_j^{2h} \beta_{3}^{hj} \Delhatsafe^{j,k-1}(h) -  \Delhatsafe^{i,k}(h) \Big ) \bigg ). 
\end{align*}
Note that the objective is linear in $\beta$ and $\lamtil^{2}$, and both have continuous, compact, convex constraint sets, so we can swap the min and max to get that the above is equivalent to 
\begin{align*}
&  \min_{z \in \simplex_\cH}  \max_{\lamtil^{1h} \in \simplex_n, h \in \cH}  \max_{\beta_1^h, \beta_{2}^{hj} \in [0,1], \beta_{3}^{hj} \in [-1,1], h \in \cH} \min_{\lamtil^{2h} \in \simplex_n, h \in \cH}  \sum_{h \in \cH} z_h \bigg ( \sum_t a_t \I \{ h(u_t) \neq o_t \} \\
& \qquad +  \sum_{i \in [n]} c_i \Big ( \epsilon_k + \beta_{1}^{h} \Delhat^{k-1}(h) -  \sum_{j \in [n]} \lamtil_j^{1h} \beta_{2}^{hj} \Delhatsafe^{j,k-1}(h)  +  \sum_{j \in [n]} \lamtil_j^{2h} \beta_{3}^{hj} \Delhatsafe^{j,k-1}(h) -  \Delhatsafe^{i,k}(h) \Big ) \bigg ). 
\end{align*}
We can write this in the form
\begin{align}\label{eq:fw_minimax_opt}
&  \min_{z \in \simplex_\cH}  \max_{\lamtil^{1}, \beta} g(z; \lamtil^1, \beta)
\end{align}
for 
\begin{align*}
& g(z; \lamtil^1, \beta) := \min_{\lamtil^{2h} \in \simplex_n, h \in \cH}  \sum_{h \in \cH} z_h \bigg ( \sum_t a_t \I \{ h(u_t) \neq o_t \} \\
& \qquad +  \sum_{i \in [n]} c_i \Big ( \epsilon_k + \beta_{1}^{h} \Delhat^{k-1}(h) -  \sum_{j \in [n]} \lamtil_j^{1h} \beta_{2}^{hj} \Delhatsafe^{j,k-1}(h)  +  \sum_{j \in [n]} \lamtil_j^{2h} \beta_{3}^{hj} \Delhatsafe^{j,k-1}(h) -  \Delhatsafe^{i,k}(h) \Big ) \bigg ). 
\end{align*}

To solve this, we will apply a version of Frank-Wolfe that handles adversarial losses to the outer player (see Section 4.2 of \cite{hazan2012projection}), and will play best response for the inner player. 

From the perspective of the outer player, at iteration $t$ of the algorithm given in \cite{hazan2012projection}, they must optimize the function
\begin{align*}
f_t(z) = g(z; \lamtil^1_t, \beta_t) = \sum_{h \in \cH} z_h c_h(\lamtil_t^1,\beta_t)
\end{align*}
for some $c_h(\lamtil_t^1,\beta_t)$. Note that this is $L = \max_h |c_h(\lamtil_t^1,\beta_t)|$ Lipschitz in the $\ell_1$-norm, and that we can bound this $L$ for all $t$ by something like $\cO(\frac{1}{\alpha} + d \alpha + n)$. The algorithm introduced in Section 4.2 of \cite{hazan2012projection} computes the standard FW update
\begin{align*}
\ztil_t = \argmin_{z \in \simplex_{\cH}} \nabla F_t(z_t)^\top z, \quad z_{t+1} = (1-t^{-1/4}) z_t + t^{-1/4} \ztil_t
\end{align*}
for
\begin{align*}
F_t(z) = \frac{1}{t} \sum_{\tau=1}^t \nabla f_\tau(z_\tau)^\top z + \sigma_t \| z - z_1 \|_2^2
\end{align*}
for $\sigma_t = (L/D) t^{-1/4}$ for $D = \max_{z_1,z_2 \in \simplex_{\cH}} \| z_1 - z_2 \|_1$ (note that in that work, the function seems to be Lipschitz in the $\ell_2$ norm while here we use $\ell_1$---this does not seem to change their result at all). It is shown in \cite{hazan2012projection} that running this procedure we obtain the bound, for any $z \in \simplex_{\cH}$,
\begin{align*}
\sum_{t=1}^T (f_t(z_t) - f_t(z)) \le 57 LD T^{3/4} .
\end{align*}
It follows that if we are able to compute $\ztil_t$ efficiently, and if the max player plays best response (and the best response can be computed efficiently), using analysis similar to that in \cite{agarwal2018reductions}, we can show that an approximate solution to \eqref{eq:fw_minimax_opt} will be found in a polynomial number of iterations.

\paragraph{Computing the Best Response for $\lamtil^1,\beta$.}
For the inner player, they must solve
\begin{align*}
\max_{\lamtil^{1}, \beta} g(z_t; \lamtil^1, \beta).
\end{align*}
Assume that $\| \ztil_t\|_0 \le m$ for each $t$, and that $\| z_1 \|_0 = 1$. Then $z_t$ will be $(mt+1)$-sparse, so the sum in $g(z_t; \lamtil^1, \beta)$ will contain at most $(mt+1)$ values. Note that the optimization over $\beta^h$ and $\lamtil^{1h}$ is completely independent, so to compute the best-response, we need to solve the following problem at most $(mt+1)$ times:
\begin{align*}
&  \max_{\lamtil^{1h} \in \simplex_n}  \max_{\beta_1^h, \beta_{2}^{hj} \in [0,1], \beta_{3}^{hj} \in [-1,1]} \min_{\lamtil^{2h} \in \simplex_n}     \sum_{i \in [n]} c_i \Big ( \beta_{1}^{h} \Delhat^{k-1}(h) -  \sum_{j \in [n]} \lamtil_j^{1h} \beta_{2}^{hj} \Delhatsafe^{j,k-1}(h)  +  \sum_{j \in [n]} \lamtil_j^{2h} \beta_{3}^{hj} \Delhatsafe^{j,k-1}(h)  \Big ) .
\end{align*}
The optimization over the first two terms is trivial and can be solved by enumerating. The third term now is a maxmin problem, however, this can also be solved trivially as it is equivalent to $\min_{j \in [n]} | \Delhatsafe^{j,k-1}(h)|$. Note that each of these gap terms is themself the solution to an optimization over $\alpha \in \cA$, but that can be solved easily for each (since there are at most polynomial of them), so they can be regarded as constants.

Thus, we conclude that the best response for $\lamtil^1,\beta$ can be computed efficiently, assuming that $m$ is polynomial in problem parameters. Note that the values of $\beta^h$ and $\lamtil^{1h}$ do not matter for $h \not\in \mathrm{support}(z_t)$ do not matter to compute the best response, so we can set them to the same value for all $h \not\in \mathrm{support}(z_t)$.

\paragraph{Computing $\ztil_t$.}
It remains to show that we can efficiently find a near-optimal $\ztil_t$ such that $\| \ztil_t \|_0 \le m$. The optimization for $\ztil_t$ will have the form
\begin{align*}
\ztil_t = \argmin_{z \in \simplex_\cH} \sum_{\tau=1}^t \nabla f_\tau(z_\tau)^\top z + 2 \sigma_t (z_t - z_1)^\top z
\end{align*}
for 
\begin{align*}
 [\nabla f_\tau(z_\tau)]_h  & =  c_h(\lamtil^1_{\tau},\beta_\tau) \\
& = \min_{\lamtil^{2h} \in \simplex_n} \sum_j a_j \I \{ h(u_j) \neq o_j \}   +  \sum_{i \in [n]} c_i \Big ( \epsilon_k + \beta_{1\tau}^{h} \Delhat^{k-1}(h) -  \sum_{j \in [n]} \lamtil_{j\tau}^{1h} \beta_{2\tau}^{hj} \Delhatsafe^{j,k-1}(h)  \\
& \qquad +  \sum_{j \in [n]} \lamtil_{j}^{2h} \beta_{3\tau}^{hj} \Delhatsafe^{j,k-1}(h) -  \Delhatsafe^{i,k}(h) \Big ) .
\end{align*}
Let $C_t \subseteq \cH$ denote the classifiers supported on $z_t$ and assume that $z_1$ is only supported on a single classifier $h_0$. Note from our discussion on computing the best-response for the $\lamtil^1$ and $\beta$ player, we have that $\beta^h$ and $\lamtil^{1h}$ are identical for all $h \not\in C_t$. We can therefore rewrite the above objective as (dropping the $\tau$ subscript and denoting, e.g. $\beta_{1}^h = \sum_{\tau=1}^t \beta_{1\tau}^h$):
\begin{align*}
& \min_{\lamtil^{2h} \in \simplex_n, h \in \cH} \sum_{h \in \cH \backslash C_t} z_h \bigg ( \sum_j a_j \I \{ h(u_j) \neq o_j \}   +  \sum_{i \in [n]} c_i \Big ( \epsilon_k + \beta_{1} \Delhat^{k-1}(h) -  \sum_{j \in [n]} \lamtil_{j}^{1} \beta_{2}^{j} \Delhatsafe^{j,k-1}(h)  \\
& \qquad +  \sum_{j \in [n]} \lamtil_{j}^{2h} \beta_{3}^{j} \Delhatsafe^{j,k-1}(h) -  \Delhatsafe^{i,k}(h) \Big ) \bigg ) \\
& \qquad+ \sum_{h \in C_t} \bigg ( \sum_j a_j \I \{ h(u_j) \neq o_j \}   +  \sum_{i \in [n]} c_i \Big ( \epsilon_k + \beta_{1\tau}^{h} \Delhat^{k-1}(h) -  \sum_{j \in [n]} \lamtil_{j\tau}^{1h} \beta_{2\tau}^{hj} \Delhatsafe^{j,k-1}(h)  \\
& \qquad +  \sum_{j \in [n]} \lamtil_{j}^{2h} \beta_{3\tau}^{hj} \Delhatsafe^{j,k-1}(h) -  \Delhatsafe^{i,k}(h) \Big ) + 2 \sigma_t z_t \bigg ) - 2 \sigma_t z_{h_0} .
\end{align*}

We will focus first on the sum over $\cH \backslash C_t$. Note that $\Delhat^{k-1}(h)$ and $\Delhatsafe^{j,k}(h)$ are both of the form
\begin{align*}
\min_{\alpha \in \cA} \sum_t \frac{1}{w_t + \alpha} \I \{ h(u_t) \neq o_t \} + \alpha \sum_t \wtil_t \I \{ h(u_t) \neq o_t \} + \frac{c}{\alpha} .
\end{align*}
Given this, we can rewrite the minimization over the first term as (where the $\alphatil$ correspond to the gaps that have negative coefficients, which is where the max comes from):
\begin{align*}
\min_{z \in \simplex_\cH}  \min_{\lamtil^{2h} \in \simplex_n, h \in \cH} \min_{\alpha^h \in \cA^k, h \in \cH} \max_{\alphatil^h \in \cA^k, h \in \cH} \sum_{h \in \cH \backslash C_t} z_h \bigg ( \cR(h;\alpha^h,\alphatil^h, \lamtil^{2h}) + f(\alpha^h, \lamtil^{2h}) + g(\alphatil^h, \lamtil^{2h}) \bigg )
\end{align*}
for $\cR$ convex in $\alpha$, and concave in $\alphatil$, $f$ convex in $\alpha$, and $g$ concave in $\alphatil$, and all functions are linear in $\lamtil^2$. Normally $\cA$ is a discrete set, but if we let $\cAtil$ be a continuous relaxation of it, we can rewrite the above as
\begin{align*}
\min_{z \in \simplex_\cH} \max_{\alphatil^h \in \cA^k, h \in \cH} \min_{\lamtil^{2h} \in \simplex_n, h \in \cH} \min_{\alpha^h \in \cA^k, h \in \cH}  \sum_{h \in \cH \backslash C_t} z_h \bigg ( \cR(h;\alpha^h,\alphatil^h, \lamtil^{2h}) + f(\alpha^h, \lamtil^{2h}) + g(\alphatil^h, \lamtil^{2h}) \bigg ) .
\end{align*}
To solve this we can again apply the FW algorithm of \cite{hazan2012projection} with the max player playing best-response. As before, as long as $\z_t$ (where $\z_t$ denotes the update for this inner optimization) is sparse, we can efficiently compute the best-response for the $\alphatil$ player, since we only need to compute it for $h \in \z_t$. The FW-style update will then have the form
\begin{align*}
& \min_{z \in \simplex_\cH} \min_{\lamtil^{2h} \in \simplex_n, h \in \cH} \min_{\alpha^h \in \cA^k, h \in \cH}  \sum_{h \in \cH \backslash C_t} z_h \bigg ( \cR(h;\alpha^h,\alphatil_t^h, \lamtil^{2h}) + f(\alpha^h, \lamtil^{2h}) + g(\alphatil_t^h, \lamtil^{2h}) \bigg )  \\
& =  \min_{\lamtil^{2h} \in \simplex_n, h \in \cH} \min_{\alpha^h \in \cA^k, h \in \cH}  \min_{h \in \cH \backslash C_t}  \cR(h;\alpha^h,\alphatil_t^h, \lamtil^{2h}) + f(\alpha^h, \lamtil^{2h}) + g(\alphatil_t^h, \lamtil^{2h})  
\end{align*}
where the equality follows since we can always swap min, and since there will always be an optimal solution supported on a single $h$. We can solve the inner min using a CSC oracle that is able to optimize over a set $\cH \backslash C_t$, and by enumerating $\lamtil^{2}$ and $\alpha$ (since we can always find an optimal solution supported on a single $h$, we can set $\lamtil^{2h},\alpha^h$ identical for all $h$ and will arrive at the same minimum). 

This will converge in polynomially many steps, and will produce some $\z_{t'}$ which is $m$-sparse (for $m$ polynomial in parameters). It follows that $\z_{t'}$ is the near-optimal value for $\ztil_t$ supported on $\cH \backslash C_t$. To pick a final value for $\ztil_t$, we can simply enumerate over the (polynomially many) $h \in C_t$, compute their loss values, and then pick the minimum out of those and the value achieved by $\z_{t'}$. This procedure will always return some $\ztil_t$ supported on at most polynomially many $h$, so $m$ can be chosen suitably to make the best-response of the max player efficient. 

Putting all of this together, we can efficiently solve for $\yhat_\ell$.

\subsubsection{Solving for $\lambda_\ell$}
We turn now to solving the optimization \eqref{eq:ragest_opt}. Using arguments similar to what we have already shown, we have that
\begin{align*}
\eqref{eq:ragest_opt} & = \inf_{\lambda \in \simplex_\cX} \max_{z \in \cZ} \min_{\alpha \in \cA, \alpha_2,\ldots,\alpha_p \in \cA} \max_{\beta_1,\ldots,\beta_m \in \cB} \cR(z;\alpha,\alpha_2,\ldots,\alpha_p,\beta_1,\ldots,\beta_m) \\
& \qquad \qquad + 2\alpha \sum_{U \in \cX} \frac{\nu(U)^2}{9\lambda(U)/10 + 1/10d} \I \{ z(U) \neq \zhat_{\ell-1}(U) \} + f(\alpha,\alpha_2,\ldots,\alpha_p,\beta_1,\ldots,\beta_m).
\end{align*}
As before, we can simply enumerate over all possible choices of $\alpha$ and $\beta$. For a fixed setting of $\alpha$ and $\beta$, to solving the $\inf$ over $\lambda$, we can apply Mirror Descent. In this case we choose the mirror map to be the negative entropy, which is strongly convex with respect to the $\ell_1$ norm. 

Given this, to solve this in a computationally efficient manner, all we need is that the objective is convex (which it is) and Lipschitz with respect to the $\ell_1$ norm. Let $g(\lambda)$ denote the objective of the above optimization. By the Mean Value Theorem,
\begin{align*}
| g(\lambda) - g(\lamtil) | = \nabla g( (1-c) \lambda + c \lamtil)^\top (\lambda - \lamtil)
\end{align*}
for some $c \in [0,1]$. So, for any $\lambda,\lamtil \in \simplex_\cX$, we can bound
\begin{align*}
| g(\lambda) - g(\lamtil) | \le \left ( \sup_{\lambda' \in \simplex_\cX} \| \nabla g(\lambda') \|_\infty \right ) \cdot \| \lambda - \lamtil \|_1.
\end{align*}
We have,
\begin{align*}
& \frac{d}{dt} \sum_{U \in \cX} \frac{\nu(U)^2}{9\lambda(U)/10 + 1/10d + 9t \lambda_0(U)/10} \I \{ z(U) \neq \zhat_{\ell-1}(U) \} |_{t=0} \\
&\qquad  = \sum_{U \in \cX} \frac{-\lambda_0(U) \nu(U)^2}{(9\lambda(U)/10 + 1/10d)^2}  \I \{ z(U) \neq \zhat_{\ell-1}(U) \}.
\end{align*}
It follows that 
\begin{align*}
\sup_{\lambda' \in \simplex_\cX} \| \nabla g(\lambda') \|_\infty \le 100 d^2
\end{align*}
so we can apply Mirror Descent to optimize the above with computational complexity scaling only polynomially in problem parameters.

\subsection{Computational Efficiency of \safebai}
The primary computational cost of \safebai is incurred by calling $\ragest$, and solving the optimization on Line \ref{line:safety_gap_est} of Algorithm~\ref{alg:beside}. We have already shown that $\ragest$ can be run in a computationally efficient manner. The optimization on Line \ref{line:safety_gap_est} has a form very similar to the optimization we solve in $\ragest$, so the same argument and solution approach (applying Mirror Descent) allows us to compute the optimal distribution, $\lambda_\ell$, here as well.

\section{Experimental details and additional results}\label{sec:exp_details}
\subsection{Experimental details}
All code was written in Python and run on a Intel Xeon 6226R CPU with 64 cores.

Algorithm~\ref{alg:beside_elim} is the precise implementation of BESIDE using elimination. It largely resemble to Algorithm~\ref{alg:beside}, with the difference that it explicitly eliminates arms.

\begin{algorithm}[h]
\caption{Best Safe Arm Identification with Elimination}
\begin{algorithmic}[1]
\State \textbf{input:} tolerance $\epsilon$, confidence $\delta$
\State $\iotaeps \leftarrow \lceil \log(\frac{1}{\epsilon}) \rceil$, $\Z^0_\text{active} \leftarrow \Z$, $\Z^0_\text{safe} \leftarrow \emptyset$
\For{$\ell = 1,2,\ldots,\iotaeps$}
	\State $\epsilon_\ell \leftarrow  2^{-\ell}$
	\State Compute allocation \XYsafe on $\Z^{\ell-1}_\text{active}$ and sample from it $\tau_\ell = \mc{O}(\text{\XYsafe}(\Z^{\ell-1}_\text{active})/\epsilon_\ell^2)$ times
	\State $\muhat^{\ell} \leftarrow \rips(\{ (x_t, s_{t,i} ) \}_{t=1}^{\tau_\ell},\cZ,\frac{\delta}{2\ell^2})$
	\State Set $\Delhatsafe^{\ell}(z) \leftarrow  \gamma - z^\top \muhat^{\ell}$ for all $z\in \Z^{\ell-1}_\text{active}$ and
	$$
	\widetilde{Z}^{\ell}_\text{active} = \{z\in\widetilde{Z}^{\ell-1}_\text{active} \; : \; \Delhatsafe^{\ell}(z) \in [-\epsilon_\ell, 2\epsilon_\ell]\} \quad \widetilde{Z}^{\ell}_\text{safe} = \{z\in\widetilde{Z}^{\ell-1}_\text{active} \; : \; \Delhatsafe^{\ell}(z) \geq 2\epsilon_\ell]\}
	$$
	\State $\Z^{\ell}_\text{active}, \Z^{\ell}_\text{safe} \leftarrow \textsc{Rage-elim}^\epsilon\xspace \Big (\widetilde{Z}^{\ell}_\text{active} \cup \widetilde{Z}^{\ell}_\text{safe} \cup \Z^{\ell-1}_\text{safe},\widetilde{Z}^{\ell}_\text{safe} \cup \Z^{\ell-1}_\text{safe}, \epsilon_\ell \Big )$
\EndFor
\State $\Z_\text{final}, \emptyset \leftarrow \textsc{Rage-elim}^\epsilon\xspace \Big (\Z^{\ell}_\text{active} \cup \Z^{\ell}_\text{safe}, \Z^{\ell}_\text{active} \cup \Z^{\ell}_\text{safe}, \epsilon_\ell \Big )$ 
\State \textbf{return} Any arm in $\Z_\text{final}$.
\end{algorithmic}
\label{alg:beside_elim}
\end{algorithm}

\begin{algorithm}[h]
\caption{$\textsc{Rage-elim}^\epsilon\xspace$}
\begin{algorithmic}[1]
\State \textbf{input:} active set $\cZ$, optimal set $\cY$, tolerance $\epsilon$
\State $\iotaeps \leftarrow \lceil \log(\frac{1}{\epsilon}) \rceil$, $\Z^0 \leftarrow \Z$, $\Y^0 \leftarrow \Y$
\For{$\ell = 1,2,\ldots,\iotaeps$}
	\State $\epsilon_\ell \leftarrow 2^{-\ell}$
	\State Compute allocation \XYdiff on $(\Z^{\ell-1}\cup \Y^{\ell-1}, \Y^{\ell-1})$ and sample from it $\tau_\ell = \mc{O}(\Z^{\ell-1}\cup \Y^{\ell-1}, \Y^{\ell-1})/\epsilon_\ell^2)$ times
	\State $\thetahat^{\ell} \leftarrow \rips(\{ (x_t, s_{t,i} ) \}_{t=1}^{\tau_\ell},\cZ,\frac{\delta}{2\ell^2})$
	\State Set $\Delhat^{\ell}(z) \leftarrow \max_{y\in\Y^{\ell-1}} y^\top \thetahat^{\ell} - z^\top \thetahat^{\ell}$ for all $z\in \Z\cup\Y$ and
	$$
	\Z^{\ell} = \{z\in\Z^{\ell-1} \; : \; \Delhat^{\ell}(z) \leq \epsilon_\ell\} \quad \Y^{\ell} = \{y\in\Y^{\ell-1} \; : \; \Delhat^{\ell}(y) \leq \epsilon_\ell]\}
	$$
\EndFor
\State \textbf{return} $\Z^{\ell}, \Y^{\ell}$
\end{algorithmic}
\label{alg:ragest_elim}
\end{algorithm}

\subsection{Additional results}

We evaluate Algorithm~\ref{alg:active_cc} and the passive baseline on two other datasets. Recall that the passive baseline selects points uniformly at randoms from the pool of examples $\X$ and then retrains the model using the same Constrained Empirical Risk Minimization oracle (\texttt{CERM}).

\paragraph{Half circle dataset.}
We consider a two-dimensional half circle dataset, visualized on Figure~\ref{fig:half_cicle}. We report in Figures~\ref{fig:halfcircle_precision} and \ref{fig:halfcircle_recall} the precision and (respectively) the recall obtained when varying the number of labels given to each method. The confidence intervals are obtained over $25$ repetitions. We observe that Algorithm~\ref{alg:active_cc} allows us to provide a classifier satisfying a given recall or precision in far fewer queries. This is in line with the results of \cite{jain2020new} on One Dimensional Thresholds, where the sample complexity of the active strategy is $O(log(n))$ while the sample complexity of the passive strategy is at least of order $n$.

\begin{figure}[tbh]
\centering
\begin{minipage}{.3\textwidth}
  \centering
  \includegraphics[width=1\linewidth]{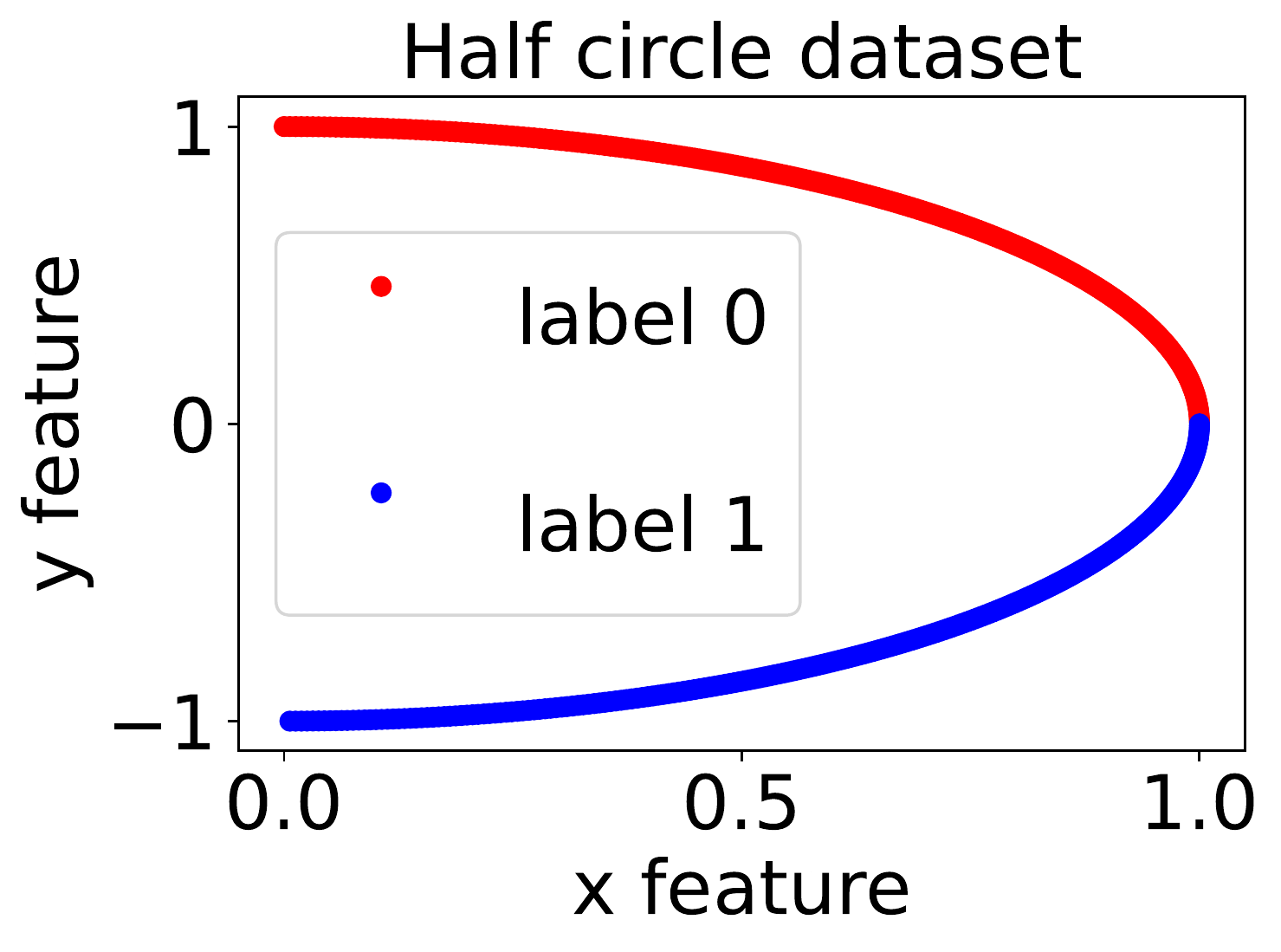}
  \vspace*{-4mm}
  \caption{Half circle dataset.}
  \label{fig:half_cicle}
\end{minipage}%
\hspace{5.00mm}
\begin{minipage}{.3\textwidth}
  \centering
  \includegraphics[width=1\linewidth]{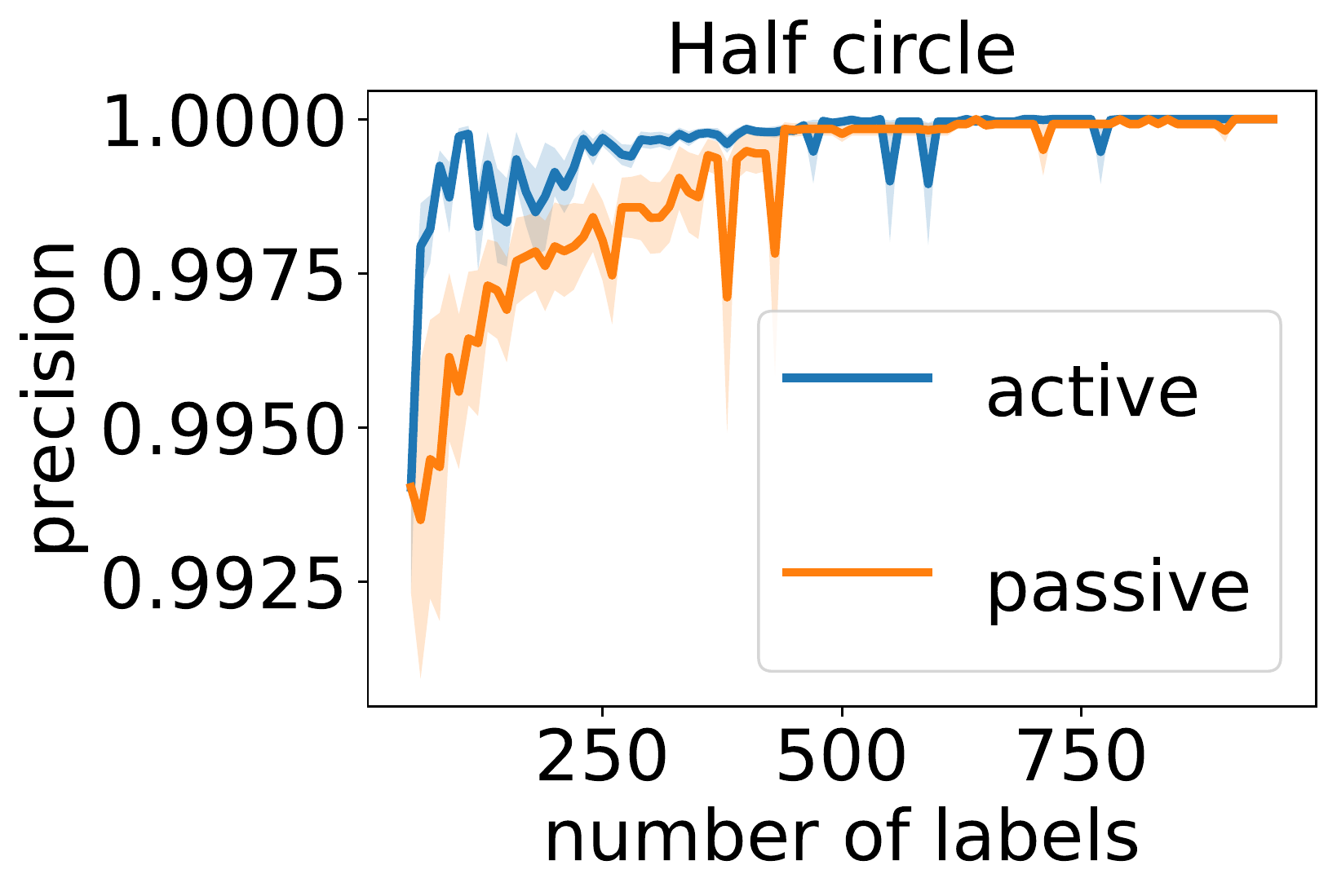}
  \vspace*{-4mm}
  \caption{Precision}
  \label{fig:halfcircle_precision}
\end{minipage}
\hspace{5.00mm}
\begin{minipage}{.3\textwidth}
  \centering
  \includegraphics[width=1\linewidth]{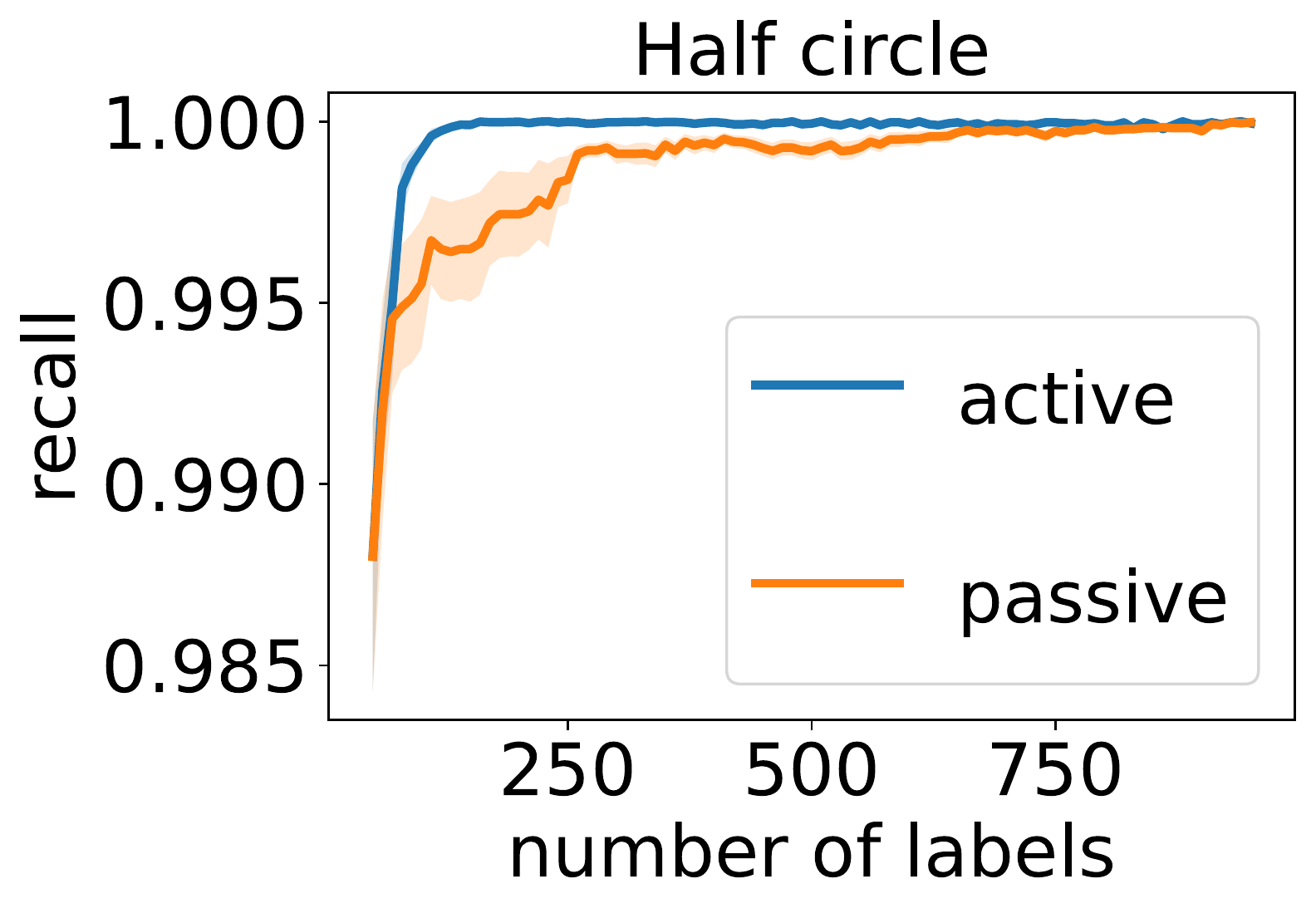}
  \vspace*{-4mm}
  \caption{Recall}
  \label{fig:halfcircle_recall}
\end{minipage}
\end{figure}

\end{document}